\documentclass[11pt]{article} 
\usepackage{iclr2027_conference_cleaned,times}

\usepackage{hyperref}
\usepackage{url}

\usepackage{ISG_style_base}
\usepackage{macros}
\usepackage[utf8]{inputenc}
\usepackage[T1]{fontenc}
\usepackage{hyperref}
\usepackage{url}
\usepackage{booktabs}
\usepackage{array}
\usepackage{siunitx}
\usepackage{amsfonts}
\usepackage{nicefrac}
\usepackage{microtype}
\usepackage{xcolor}
\usepackage{graphicx}
\usepackage{subcaption}
\usepackage{float}
\usepackage{enumitem}

\usepackage{svg}

\usepackage{tabularx}
\usepackage{caption}
\usepackage{multirow}
\usepackage{makecell}
\usepackage{graphicx}
\usepackage{hyperref}

\usepackage{amssymb}
\usepackage{mathtools}
\usepackage{amsthm}

\newtheorem{definition}{Definition}

\usepackage{tikz}
\usetikzlibrary{arrows.meta,positioning,fit,backgrounds}
\usepackage{pifont} 
\usepackage{booktabs}
\usepackage{threeparttable}
\usepackage{multirow}
\usepackage{comment}

\usepackage{titletoc}

\usepackage{sidecap} 
\definecolor{lightgray}{gray}{0.90} 
\definecolor{balancedred}{HTML}{E04848}

\newcommand{\hatz}{\smash{\hat{z}}}
\newcommand{\hatZ}{\smash{\hat{Z}}}
\newcommand{\hatZhatI}{\smash{\hat{Z}_{\hat{\mcI}}}}
\newcommand{\supk}{^{(k)}}

\providecommand{\rank}{}\renewcommand{\rank}{{\rm rank}}

\usepackage{thmtools}
\usepackage{thm-restate}

\hypersetup{%
    colorlinks,
    linkcolor={blue!75!black},
    citecolor={green!50!black},
    urlcolor={violet},
    linktocpage,
    plainpages=false}

\usepackage[capitalize, nameinlink, sort&compress]{cleveref}
\crefname{equation}{eqn.}{eqns.}
\newcommand{\email}[1]{\href{mailto:#1}{#1}}

\usepackage{crossreftools}
 \title{Block Disentanglement in CRL: Bridging Identifiability and Visual State Estimation}

\author{%
  Emre Acartürk\thanks{Corresponding author \email{acarte@rpi.edu}. Work partially done while at Google DeepMind.} \\
  Rensselaer Polytechnic Institute \\
  Troy, NY%
  , USA \\
  \And
  Pranamya Kulkarni\thanks{Work done while at Google DeepMind.} \\
  The University of Texas at Austin \\
  Austin, TX%
  , USA \\
  \And
  Puranjay Datta\footnotemark[2] \\
  Swiss AI, EPFL \\
  Lausanne, Switzerland \\
  \AND
  Karthikeyan Shanmugam \\
  Google DeepMind \\
  Bengaluru%
  , India \\
  \And
  Burak Varıcı \\
  University of Delaware \\
  Newark, DE%
  , USA \\
  \And
  Ali Tajer \\
  Rensselaer Polytechnic Institute \\
  Troy, NY%
  , USA
}

\begin{document}

\maketitle

\begin{abstract}
    Causal representation learning (CRL) is the process of recovering causally-related latent variables from high-dimensional observations. As a label-free inference method, CRL is particularly attractive for applications where data labels are unavailable or impractical to obtain. While there has been significant progress in understanding the identifiability guarantees of CRL, such guarantees often hold under highly stylized assumptions, which temper the direct application to real-world problems. This paper has a two-fold objective for \emph{interventional} CRL. First, it establishes identifiability guarantees for substantially weaker interventional assumptions, resulting in \emph{block disentanglement} of the causal variables, where the block structure depends on the realistically available intervention mechanisms. Secondly, the block disentanglement framework is used for embodied visual state estimation, in which the objective is to recover the latent physical variables of a robotic system directly from visual data (images and videos) \emph{without labeled data}. These two components are critically complementary. The block disentanglement theory delineates identifiability guarantees under weakened assumptions, and the application demonstrates that the resulting objective remains effective in a controlled embodied setting despite further assumption violations, providing a theory-to-practice bridge needed to translate the promise of \textbf{label-free} CRL into practical problems.
\end{abstract}

\section{Introduction}\label{sec:intro}

Physical systems are generally described by a set of state variables needed to characterize the systems' condition, in conjunction with the physical laws that delineate how these variables causally influence one another. These variables and causal mechanisms determine how the system evolves and responds to different forms of external stimulation or interventions. In many applications, however, these variables and mechanisms are not observed directly. Instead, the system is monitored or perceived through high-dimensional sensory systems (e.g., images and videos). This creates a fundamental need to recover hidden variables and causal mechanisms from high-dimensional observable data. Such recovery is critical for meaningful reasoning about the system (e.g., how the system responds to hypothetical stimuli) and for transferring models across environments with distributional shifts. 

The recovered latent models and mechanisms must be provably correct, generating system representations that specify the true underlying system state variables. Supervised learning offers a direct solution to this, but it is label-intensive and requires large-scale per-sample labels, which are often unavailable due to expense and complexity of curating annotated and labeled data. To address this challenge, \emph{identifiable representation learning} circumvents dependence on labeled data while providing provably correct recovery without per-sample supervision.

Causal representation learning (CRL) has emerged as a form of identifiable representation learning apt for physical systems. Specifically, it is formed at the confluence of three primary research directions: (i)~disentangling generative factors embedded in high-dimensional data, (ii)~causal inference, and (iii)~representation learning. Because a system's variables are not necessarily statistically independent (they are coupled through complex interactions), CRL aims to jointly disentangle them and learn the causal mechanisms that connect them \citep{scholkopf2021toward}. Such causal mechanisms represent stable conditional laws relating causes to effects. Extensive studies in recent years have established the identifiability limits of CRL, establishing when latent causal factors are recoverable; for representative studies see \citep{varici2025score,ahuja2023interventional,yao2025unifying,ng2025causal}. In parallel, there has been an active interest in translating these CRL guarantees into applications in genomics \citep{tejada2025causal}, biomedical \citep{sun2025causal}, and dynamical systems \citep{yao2024marrying}.

A particularly important application with growing interest in which CRL offers a natural solution is the analysis of video observations of physical systems (e.g., robots), where the objective is recovering the latent state variables from image sequences without requiring frame-by-frame labels \citep{lippe2022icitris,lippe2023biscuit,chen2025causalverse}. A video provides a high-dimensional record of how the physical system evolves and responds to external inputs, while the true variables (e.g., a robot's positions, orientations, and velocity) are hidden and embedded in the video record. In such applications, while videos are readily and inexpensively available, annotating each frame is prohibitively laborious and costly. Therefore, this is a setting in which CRL can offer significant potential by identifying and recovering latent variables and their causal dynamics directly from unlabeled video.

\paragraph{Theory-to-practice gap.}  Achieving these objectives in videos, however, necessitates operationalizing CRL in regimes beyond what the current CRL theory covers. Specifically, while identifiability results in CRL often provide strong guarantees, they also rely on \emph{collectively} strong assumptions such as access to counterfactual observations; atomic (single-node) intervention targets; stochastic hard or $do$ interventions; independent samples; exhaustive interventional diversity; and an exactly known latent dimensionality. While each of these assumptions can be justifiably valid in certain settings, in real-world settings, such as systems evolving continuously over time, they might not all be satisfied \emph{simultaneously}. Challenges that violate this set of assumptions include sequential access to information and partial or imperfect interventions with unknown targets. As a result, the gap between identifiable CRL theory and label-free state recovery from video is fundamentally a gap in the joint validity of the assumptions required by existing guarantees. For more discussion on the existing CRL guarantees and their attendant assumptions, we refer to \Cref{app:related-work}.

To bridge the gap between current CRL theory and practical visual state estimation, we take two complementary steps, one theoretical and one empirical. Both are critically needed, as they address different but complementary obstacles. 
\begin{itemize}[topsep=0in, itemsep=0in, leftmargin=0.2in]
    \item \textbf{Theoretical bridge:} The theoretical step weakens the intervention assumptions required for identifiability. Specifically, current CRL theory focuses on disentangling \emph{individual} variables, which in turn often necessitates intervention structures that isolate one variable at a time. To address this, we consider \emph{block disentanglement}, and establish conditions under which an intervened \emph{block} of variables remains identifiable even when individual variables cannot be disentangled separately.
    \item \textbf{Empirical bridge:} On the empirical side, we further weaken the assumptions and investigate whether the resulting block disentanglement principle remains effective when applied to embodied visual data, where temporal dependence, imperfect intervention effects, and uncertainty in the latent representation arise naturally. For this purpose, robotics is a natural target and has long motivated CRL \citep{scholkopf2021toward}: an effective robotic world model must capture the causal factors governing the robot's state and its interaction with the environment. In our experiments, we show that even under extensive assumption violations, the proposed method retains significant information about the ``active block'' of variables in a simulated robot pose estimation task.
\end{itemize}

Collectively, our results demonstrate that interventional CRL remains highly effective beyond stringent regimes defined by stylized assumptions. We achieve this by relaxing the identifiability target from individual causal variables to an intervention-defined block under weaker conditions; the resulting block disentanglement principle remains effective even when experimental data further depart from the theoretical requirements. Accordingly, our results point to a step-by-step transition in which progressively weaker theoretical \emph{and} empirical conditions still yield useful representation learning objectives. \looseness=-1

\section{Preliminaries}\label{sec:background}

\paragraph{Latent causal model.} Consider a latent causal space consisting of $n$ causal random variables $Z \triangleq  [Z_1,\dots,Z_n]^{\top}$. An \emph{unknown} transformation $f: \R^n \to \R^d$ maps $Z$ to high-dimensional observed random variables $X \triangleq [X_1,\dots,X_d]^{\top}$ according to $X = f(Z)$, where $d \geq n$. As is standard in CRL literature, we assume that $f$ is injective, and a smooth diffeomorphism onto its image. Denote the probability distributions of $Z$ and $X$ by $p$ and $ p_X$, respectively. Following causal Bayesian network formalism~\citep{pearl2009causality}, the distribution $p$ of $Z$ factorizes according to a directed acyclic graph (DAG) $\mcG$ on $n$ nodes, where node $i$ represents $Z_i$:
\begin{equation}
    p(z) = \prod_{i=1}^{n} p_i(z_i \mid z_{\Pa(i)}) \ ,
\end{equation}
where $\Pa(i)$ denotes the parents of node $i$, and conditional distribution $p_i(z_i \mid z_{\Pa(i)})$ encodes its \emph{causal mechanism}. We use $\Paplus(i)$ to denote $\Pa(i) \cup \{i\}$. We assume $p$ has full support on $\R^n$ and that all causal mechanisms $\{p_i(z_i \mid z_{\Pa(i)}) : i \in[n]\}$ are continuously differentiable with respect to all latent variables.

\paragraph{Interventions.} 
Interventions are a defining ingredient of causal modeling and play a central role in CRL. We begin by describing the most restrictive intervention model and then progressively relax it. The strongest form is an atomic $do$-intervention on a latent variable $Z_i$, which deterministically sets that variable to a prescribed value, $Z_i=z_i$, while removing its dependence on its usual causal mechanism. In this paper, intervention assumptions are weakened in two directions:
\begin{itemize}[topsep=0in, itemsep=0in, leftmargin=0.2in]
    \item {\bf Unknown interventions:}
    We assume access to $K$ intervention models, each targeting an \emph{unknown} subset ${\cal I}_k\subseteq[n]$, $k\in[K]$, of the latent variables.
    \item \textbf{Stochastic interventions:} 
    We assume access only to stochastic interventions, which modify probability distributions rather than deterministically fixing variable values. 
    Specifically, an interventional environment that intervenes on $Z_i$ changes its causal mechanism $p_i(z_i \mid z_{\Pa(i)})$ to another distinct (conditional) distribution, while leaving the mechanisms of all other variables intact. 
    We consider both \emph{stochastic hard and stochastic soft interventions}: a soft intervention on $Z_i$ replaces the observational mechanism with a different $q_i(z_i \mid z_{\Pa(i)})$, whereas a hard intervention removes parental dependence and replaces the observational mechanism of $Z_i$ with $q_i(z_i)$. Thus, applying intervention $k$ modifies the causal mechanisms of nodes in ${\cal I}_k$ and creates an \emph{interventional environment} that induces the following distribution over latent variables:
    \begin{alignat}{6}
        \label{eq:qk_hard} & \mbox{hard intervention:} \quad & q\supk(z) &= \prod_{j \notin \mathcal{I}_k} p_j(z_j \mid z_{\Pa(j)}) \cdot \prod_{i \in \mathcal{I}_k} q_i\supk(z_i) \ , \\
        \label{eq:qk_soft} & \mbox{soft intervention:} \quad & q\supk(z) &= \prod_{j \notin \mathcal{I}_k} p_j(z_j \mid z_{\Pa(j)}) \cdot \prod_{i \in \mathcal{I}_k} q_i\supk(z_i \mid z_{\Pa(i)})  \ .
    \end{alignat}
\end{itemize}

\paragraph{CRL block disentanglement.}
The main objective in CRL is to use observational data $X$ to recover the latent variables $Z$ and the latent causal graph $\mcG$. Identifiability in CRL pertains to establishing conditions under which a unique equivalence class of models explains the observed data.
Existing CRL identifiability results mainly focus on \emph{elementwise} disentanglement of $Z$, that is, recovering the variables $\{Z_i:i\in[n]\}$ individually~\citep{squires2023linear,buchholz2023learning,zhang2023identifiability,vonkugelgen2023twohard,varici2025score,chen2024identifying}. However, these results generally require strong assumptions. In contrast, a more relaxed identifiability target, \emph{block disentanglement}, has received relatively less attention \citep{ahuja2023interventional,liang2023causal,yao2025unifying}.
In our theoretical analysis, we aim to prove block disentanglement under weaker conditions on data availability.
For completeness, we formally define both types of disentanglement as follows.
\begin{definition}[Disentanglement classes]\label{defn:block-id}
    Let $\hatZ$ be an estimate of the latent variables $Z$. 
    \begin{enumerate}[label=(\roman*), topsep=0in, itemsep=0in, leftmargin=0.2in]
        \item We say that $\hatZ$ achieves \textbf{elementwise disentanglement} if there exist a permutation $\pi$ of $[n]$ and diffeomorphisms $\phi_i : \R \to \R$ such that $\hatZ_{\pi_i} = \phi_i(Z_i)$ for all $i \in [n]$.
        \item We say that $\hatZ$ achieves \textbf{block disentanglement} for $\mcI \subseteq [n]$ if there exist a subset $\hat{\mcI}\subseteq [n]$ with size $|\mcI|$ and a diffeomorphism $\phi_\mcI : \R^{|\mcI|} \to \R^{|\mcI|}$ such that $\hatZhatI = \phi_\mcI(Z_\mcI)$.
    \end{enumerate}
\end{definition}
Hence, elementwise disentanglement states that all latent variables are recovered up to a permutation and elementwise transformations, whereas block disentanglement states that the estimate $\hatZhatI$ of block $\mcI$ is an invertible function of the true block $Z_\mcI$ without leakage from other variables or noise.

\paragraph{Weakened assumptions.} In this paper, we consider a substantially more relaxed setting than those typically studied in the existing literature.
To compare assumptions across different studies, we first organize them into six broad categories, which provides a common reference point.
\begin{itemize}[leftmargin=0.2in,topsep=0em,itemsep=1pt]
    \item \emph{(C1) Atomic targets.} Each intervention alters the mechanism of exactly one latent coordinate.
    \item \emph{(C2) $do$ or hard interventions.} These interventions break the dependence of the intervened variables on their parents.
    \item \emph{(C3) Pairing.} Some studies adopt more than one intervention environment per node. In such cases, the environments associated with $Z_i$ come in matched pairs $(q^i, \bar q^i)$.
    \item \emph{(C4) Interventional discrepancy.} When a node undergoes a pair of interventions, the mechanisms $q_{i}(z_i)$ and $\bar{q}_{i}(z_i)$ are sufficiently distinct: $\nabla_{z_i} \log \big(q_i(z_i)/\bar{q}_i(z_i)\big)$ is nonzero almost everywhere.
    \item \emph{(C5) Sample independence.} Samples within each environment are i.i.d.
    \item \emph{(C6) Latent dimension.} The latent dimension $n$ is assumed to be known.
\end{itemize}

Existing results obtain identifiability under different tradeoffs between intervention structure and assumptions on the latent model or observation map. In the general nonlinear setting closest to ours, \citep{vonkugelgen2023twohard,varici2025score} establish elementwise disentanglement under (C1)--(C6): interventions are atomic and hard, with paired interventions satisfying the node-wise discrepancy condition. Other works in the literature operate under additional structural assumptions, such as linear transformations, linear causal models, or known causal graphs \citep{squires2023linear,buchholz2023learning,zhang2023identifiability,ahuja2023interventional,liang2023causal,yao2025unifying}.

In this paper, we relax assumptions (C1)--(C6) in two stages. The first stage forms the \emph{theoretical bridge}: we replace assumptions (C1)--(C4) with weaker ones and establish identifiability results in the form of block-disentanglement guarantees. Crucially, these replacement assumptions impose collectively less stringent conditions on the power, coverage, and diversity of the interventions. The second stage forms the \emph{empirical bridge}: we move beyond the theory by considering settings that additionally violate (C5) and (C6). Together, these two stages allow us to provide empirical evidence for disentanglement from visual data under substantially weaker conditions than those considered in the existing literature.

\section{Theoretical Bridge: Weakening Intervention Diversity}\label{sec:theory}

Next, we develop the theoretical framework for establishing \emph{block disentanglement} under the weakened assumptions by removing (C1)--(C4). We begin by introducing the definitions needed to formalize the problem, methodology, and analysis used throughout the remainder of the paper. Following the score-based interventional CRL framework of \citep{varici2025score}, we denote the score functions associated with PDFs $p$ and $p_X$ by $s(z) \triangleq \nabla_z \log p(z)$ and $s_X(x) \triangleq \nabla_x \log p_X(x)$, respectively. Similarly, the score of interventional latent PDFs $q\supk$ and its observed counterpart are denoted by $s\supk$ and $s_X\supk$, respectively. Accordingly, we define the score difference function under intervention ${\cal I}_k$ with respect to the base observational environment as
\begin{equation}\label{eq:score-diff-z}
    d\supk(z)\triangleq s\supk(z) - s(z) = \sum_{i \in \mcI_k} \nabla_z \log {\textstyle\dfrac{q_i\supk(z_i \mid z_{\Pa(i)})}{p_i(z_i \mid z_{\Pa(i)})}} \ .
\end{equation}
Similarly, define the observed score difference function $d_X\supk(x)\triangleq s_X\supk(x)-s_X(x)$. Since the ratios in \eqref{eq:score-diff-z} depend only on $z_{\Paplus(i)}$, for the gradient we have
\begin{equation}\label{eq:dzk-is-sparse}
    [d\supk(z)]_j = 0\ , \qquad \forall j \not\in \Paplus(\mcI_k) \ .
\end{equation}
Furthermore, the score differences in the latent space and the observed space are related by the Jacobian of $f$ according to~\citep[Lemma 8]{varici2025score}:
\begin{equation}\label{eq:d-transformation}
    d_X\supk(x) = \left[J_f(z)^\dag\right]^\top \cdot d\supk(z) \ . 
\end{equation}
Next, we define the \textbf{active block} as the union of the intervention targets $\mcI \triangleq \bigcup_{k \in [K]} \mcI_k$, with size $m \triangleq |\mcI|$, and denote the complementary \textbf{nuisance} coordinates by $\mcI^c \triangleq [n] \setminus \mcI$. We emphasize that the sets $\{{\cal I}_k:k\in[K]\}$, and subsequently $\cal I$ and $\mcI^c$ are \emph{unknown}. Collecting the latent score differences for each environment $d\supk(z)$, we define the latent score difference matrix as
\begin{equation}\label{eq:DZ}
    \bD(z) \;\triangleq\; \big[\, d^{(1)}(z) \,,\,\dots\,,\, d^{(K)}(z) \,\big] \;\in\; \R^{n \times K} \ .
\end{equation}
Similarly, we denote the score difference matrix in observation domain by $\bD_X(x)$. Next, we define the key conditions to formalize sufficient conditions that ensure block disentanglement.
\begin{definition}[Ancestral closure of the active block]\label{defn:mcI-anc-closed}
The active block $\mcI$ is said to be \emph{ancestrally closed} in $\mcG$ if $\Pa(i) \subseteq \mcI$, $\forall i \in \mcI$. Equivalently, $\Paplus(\mcI) = \mcI$.
\end{definition}

\begin{definition}[Collective rank diversity]\label{defn:dz-full-rank}
The collection of interventions is said to satisfy \emph{collective rank diversity} over the active block $\mcI$ if the associated score differences jointly span the active subspace, i.e., for almost every $z$, we have $\rank\big(\bD(z)\big) = |\mcI|$.
\end{definition}
\Cref{defn:mcI-anc-closed,defn:dz-full-rank} replace the canonical interventional assumptions (C1)--(C4).
Rather than constraining each intervention individually to be atomic, hard, paired, or sufficiently distinct, these definitions impose only aggregate structural conditions on the collection of interventional environments. \Cref{defn:mcI-anc-closed} ensures that intervention-induced score differences remain confined to the active block $\mcI$, while \Cref{defn:dz-full-rank} requires that these score differences collectively span that block.

Specifically, \Cref{defn:dz-full-rank} is the block-level counterpart of the interventional discrepancy condition (C4). This condition is satisfied by $K \geq |\mcI|$ multi-node interventional environments as long as their score difference directions span the active subspace almost everywhere. Furthermore, this condition is strictly weaker than the canonical full-coverage intervention structure in (C1)--(C4). In particular, under (C1)--(C4), taking score differences between the paired hard-intervention environments produces one nonzero direction per intervened coordinate, and hence satisfies \Cref{defn:dz-full-rank}. The converse does not hold: collective rank diversity permits unpaired, multi-node, and soft interventions, provided their score difference directions jointly span the active block. Hence, the two definitions replace intervention-by-intervention restrictions with collective conditions on support and span, substantially relaxing the intervention structure while retaining a nontrivial identifiability guarantee. Leveraging these definitions, we provide sufficient conditions for the block-identifiability result formalized next.

\begin{restatable}[Block identifiability under collective rank diversity, (cf.~\Cref{app:proof} for proof)]{theorem}{thmBlockID} \label{thm:block_id}
    Let $(h, g)$ be a smooth encoder--decoder pair that perfectly reconstructs $X$, i.e., $g(h(X)) = X$. If the estimated matrix 
    \begin{equation}
        \bD_{\hatZ}(\hatz) = \left[J_h(x)^\dagger\right]^\top \bD_X(x) 
    \end{equation}
    is supported on a row set $\hat{\mcI}\subseteq[n]$ with
    $|\hat{\mcI}|=|{\mcI}|$, then under the ancestral-closure and collective-rank-diversity conditions of
    \Cref{defn:mcI-anc-closed,defn:dz-full-rank}, together with assumptions (C5)--(C6), the estimate
    $\hatZ=h(X)$ block-disentangles the true active block $Z_{\mcI}$.
\end{restatable}

\paragraph{Comparison to existing block disentanglement results.} Existing block-identifiability results rely on substantially different structural information. \citet{vonkugelgen2021self} established the concept under counterfactual CRL setting. In interventional CRL, \citep{ahuja2023interventional} obtains a block-affine identification result by exploiting intervention-induced support independence, and \citep{liang2023causal} assumes that the causal graph and intervention targets are known and require multiple perfect interventions on the same target block satisfying a block-interventional discrepancy condition. Most closely related setting is
\kern8em   
\citep{yao2025unifying} which provides a general invariance-based framework for block identifiability: latent blocks that \emph{remain invariant} can be identified. In contrast, our block disentanglement target $\mcI$ is precisely the intervention-affected (i.e., \emph{varying}) block.

\section{Empirical Bridge: Block Disentanglement to Embodied Video}\label{sec:regimes}

The theoretical bridge in \Cref{sec:theory} already relaxes assumptions (C1)--(C4), while retaining (C5)--(C6), and establishes block-disentanglement guarantees under this weaker setting. The next question is whether the resulting CRL principles remain effective when we move further toward practical visual state estimation and additionally relax (C5)--(C6). This constitutes our \emph{empirical bridge}.

To evaluate this progression, we consider two complementary empirical settings. The first, which we call the \emph{assumption-matched benchmark}, satisfies the full set of assumptions (C1)--(C6). It provides a reference point for the performance attainable when the canonical CRL conditions hold, and allows us to assess the overall effectiveness of the methodology in its most favorable identifiable setting. The second, which we call the \emph{embodied setting}, additionally dispenses with (C5)--(C6), on top of the relaxation of (C1)--(C4) already established by the block-disentanglement theory. Comparing these two settings therefore serves two complementary purposes: benchmarking performance under the canonical assumptions, and assessing how far the combined theoretical and empirical bridges can carry CRL toward visual state estimation under substantially weaker conditions.

\subsection{Setup and Data Generation}

We consider estimating a robot's pose and joint configuration from visual observations as an important instance of visual state estimation. This capability is particularly relevant when a robot's internal state is unavailable, unreliable, or cannot be directly communicated to other agents. Accordingly, visual recovery of articulated robot state has received growing attention in computer vision and robotics. We refer to \Cref{app:related-work} for a thorough review of the related literature.

For our study, the latent physical state of interest is the robot arm's joint configuration. The rendered images depend on these joint angles, as well as other intrinsic and extrinsic factors such as link geometry, camera pose, lighting, and background. To isolate the state estimation problem, we vary only the joint angles while holding the remaining generative factors fixed. We generate the visual observations using the Panda-Gym platform~\citep{gallouedec2021panda}, which provides controlled actuation of the robotic arm and corresponding rendered images. To preserve observability of the full joint configuration and support the injective rendering assumption on $f$, we record each pose from two distinct camera viewpoints.

As briefly mentioned earlier, in our analysis we consider the following empirical regimes. 

\begin{itemize}[leftmargin=*, topsep=0in, itemsep=0in]
    \item \textbf{Assumption-matched benchmark.} This setting satisfies (C1)--(C6). Interventions are atomic (C1), replacing one joint's sampling distribution with a \emph{hard} intervention while leaving the others untouched (C2); environments come in matched pairs $(q^i, \bar q^i)$ per joint as (C3) requires, and are designed to satisfy (C4); joint angles are sampled i.i.d.\ and rendered independently, so there is no temporal structure (C5); and the true number of varying joints is known (C6). This regime validates whether disentanglement and variable estimation are numerically stable on rendered $128\times128$ images, and whether the canonical objective recovers the intended coordinates at this visual scale.
    \item \textbf{Embodied setting.} This setting moves beyond the conditions of \Cref{thm:block_id} and additionally violates (C5)--(C6). Specifically, it replaces i.i.d.\ rendering with continuous videos. A proportional controller drives the arm from a sampled start to a sampled goal configuration, and it retains every frame along the way. Episodes terminate and reset, limiting dependence across trajectories but not within one. Each environment $k$ corresponds to a distinct goal-sampling distribution that \emph{prioritizes} displacement of a designated target joint while letting the remaining joints move as the dynamics dictate. Importantly, no observational dataset is distinguished. These departures from the theoretical setting arise through three distinct mechanisms: temporal dependence in the observed data, imperfect and coupled intervention effects, and mismatch between the learned bottleneck dimension and the true latent dimension. We summarize each violation below and indicate the corresponding theoretical assumption that it relaxes.
    \begin{itemize}[leftmargin=0.2in,topsep=0em,itemsep=1pt]
        \item \emph{V1: Temporal dependence} (violates C5). Consecutive frames differ by one controller step. Therefore, samples within a video are strongly correlated.
        We in fact leverage the controlled nature of the videos to compute score differences without an observational environment by treating the two halves of each video as observational and interventional data.
        \item \emph{V2: Off-target effects} (violates C1--C4 and \Cref{defn:mcI-anc-closed,defn:dz-full-rank}). The control task itself induces small but nonzero motion in off-target joints, so $\bD$ carries small nonzero entries on rows in $\mcI^c$ instead of exact zeros. The active/nuisance partition becomes approximate. Furthermore, environments are unpaired, and interventions are not hard.
        \item \emph{V3: Overcomplete bottleneck} (violates C6). We train autoencoders with $d_b \in \{8, 12, 16, 32\}$ dimensions against $6$ true latent factors. Increasing $d_b$ increases the autoencoder's reconstruction capacity. However, without a disentanglement signal, nothing forces the active factor information to concentrate. This violation makes the active/nuisance contrast a valid metric even when $\mathcal I=[n]$.
    \end{itemize}
\end{itemize}

\subsection{Methodology}\label{sec:methodology}

In this section, we describe how we implement a CRL algorithm that is motivated by the theoretical identifiability results. In practice, we train an encoder--decoder pair $(h, g)$ where $h : \R^d \to \R^{d_b}$ encodes the visual data into $d_b$ entries. Our approach achieves disentanglement by estimating latent score functions. This can be implemented efficiently by leveraging a fundamental connection between latent score estimates and the score estimates of the observational data in~\eqref{eq:d-transformation}: by estimating the score differences $d_X\supk$ once on the visual domain $X$, one can compute the estimated score difference $d_{\hat Z}\supk$ via $d_{\hat Z}^{(k)}(\hatZ) = \left[J_g(\hatZ)\right]\smash{^\top} \cdot d_X\supk(x)$.

\paragraph{Assumption-matched benchmark.} In the assumption-matched benchmark introduced above, where assumptions (C1)--(C6) are all satisfied, \citet[Theorem~22]{varici2025score} shows that elementwise disentanglement can be achieved by learning an encoder--decoder pair that minimizes a score-based regularization objective together with a reconstruction loss.
Specifically, in this setting, we consider score differences between the \emph{interventional pair} of environments $(q^i,\bar{q}^i)$ affecting each node $i$, which we denote by $d^i$.
The objective penalizes deviations from the expected structure of score differences across interventional environments, while the reconstruction term ensures that the learned representation preserves the observed data. Under the stated assumptions, minimizing this objective guarantees \emph{elementwise disentanglement} of the latent variables:
{\small
\begin{equation}\label{eq:loss_function}
    \mathcal{L}(h,g) = \underbrace{ \mathbb{E} \big[\|g(h(x)) - x\|^2\big]}_{\text{Reconstruction Loss}} \;+\; \lambda \sum_{i=1}^{n} \underbrace{\big\| \mathbb{E} \big[ \big|J_g(\hatz)^\top \cdot d_X^i(x)\big|\big] - e_i \big\|^2}_{\text{Sparsity Loss}} \ ,
\end{equation}
}\!
where $e_i$ denotes the $i$-th standard unit vector. This result follows from the observation that the score difference \emph{between} the paired interventions $q^i$ and $\bar q^i$ is \emph{one-sparse} on the true latent domain.

\paragraph {Embodied setting.} In the embodied setting, where (C1)--(C6) no longer jointly hold, neither the pairing nor the atomicity is available. Therefore, the score difference does not fit to a single coordinate $i$, unlike \eqref{eq:loss_function}. Following \Cref{thm:block_id}, we instead confine the score differences to a fixed block. For each environment $k \in [K]$ we estimate $d_X\supk$ directly from trajectories, use change-of-variables formula \eqref{eq:d-transformation} to obtain $d_{\hat Z}\supk$, and penalize parts falling outside the estimated active block $\hat{\mcI}$:
{\small
\begin{equation}\label{eq:loss_block}
    \!\!\!\! \mcL_{\text{blk}}(h,g) = \underbrace{\E\big[\|g(h(x)) - x\|^2\big]}_{\text{Reconstruction}}
    \;+\; \lambda \underbrace{\Big\| \Big[\, \E\big[\big|J_g(\hatz)^\top \cdot \bD_X(x)\big|\big] \,\Big]_{\hat{\mcI}^c, :} \Big\|_F^2}_{\text{Block sparsity}} \;+\; \lambda_s \underbrace{(\E\big[\|h(x)\|\big] - 1)^2}_{\text{Scale}} \ .
\end{equation}
}\!
Importantly, we do not need to know the latent dimension $n$. We assume that the active-block size $m=|\mcI|$ is known and designate $m$ coordinates of the $d_b$-dimensional representation as $\hat{\mcI}$. The remaining $d_b-m$ coordinates are nuisance and can encode whatever information is needed for reconstruction. \looseness=-1

\section{Empirical Results and Analysis}\label{sec:Results}

This section empirically demonstrates that, despite extensive relaxation of interventional assumptions, we see only graceful degradation compared to CRL approaches designed under those assumptions. All experimental details, additional experiments and ablations are given in \Cref{app:regime_a,app:regime_b}.

\paragraph{Experiment pipeline.} 
The implementation consists of three stages, summarized in \Cref{fig:data-pipeline}. The input consists of high-dimensional visual observations $X$ collected under different interventional environments, and the objective is to obtain a low-dimensional estimate $\hatZ$ of the physical state.

\textbf{Stage 1 (AE1)} performs dimensionality reduction. Directly estimating score functions and score differences in the image domain is computationally expensive. We therefore use a convolutional autoencoder $(E_1,D_1)$ to map each grayscale image $X$ to a lower-dimensional $8\times 8\times 1$ feature representation $Y = E_1(X)$ by minimizing the reconstruction loss $\smash{\E_X\![\|X-D_1(E_1(X))\|_2^2]}$. After training, we freeze $E_1$ and use it as a feature extractor. This compression substantially reduces the dimension in which we must estimate score differences, at the potential cost of discarding latent information. We therefore separately evaluate how much state information is retained by $Y$.

\textbf{Stage 2 (LDR)} estimates intervention-induced score differences in the feature space. We exploit the fact that a binary classifier trained with cross-entropy loss estimates the log-density ratio between the two classes~\citep{gutmann2012noise}. For each interventional environment $k$, we train a log-density-ratio estimator ${f_{\mathrm{LDR}}\supk}$ to distinguish samples between chosen distributions. The gradient of its logit, ${\nabla_y f_{\mathrm{LDR}}\supk(y)}$, then provides an estimate of the corresponding feature-space score difference.\looseness=-1

\textbf{Stage 3 (AE2)} learns the disentangled representation. A second autoencoder $(E_2,D_2)$ is trained on $Y$ using the objectives developed in \Cref{sec:methodology}, producing $\smash{\hatZ \triangleq E_2(Y) \in \R^{d_b}}$. The training objective depends on the empirical regime: in the assumption-matched benchmark, we use the canonical score sparsity objective \eqref{eq:loss_function} that promotes elementwise disentanglement, whereas in the embodied setting, we use the block sparsity objective \eqref{eq:loss_block} as motivated by our block-disentanglement theory.

\paragraph{Empirical objectives.}
The experiments are designed to determine how far the CRL principles developed in the preceding sections carry from the assumption-matched setting toward embodied visual data. In particular, we address the following questions:
\begin{itemize}[leftmargin=0.2in,topsep=-1pt,itemsep=-2pt]
    \item \emph{Can CRL recover physical state from visual observations without state labels?}
    In the assumption-matched benchmark, we test whether the canonical score-based CRL objective can recover individual joint variables from high-dimensional images when its identifiability assumptions are satisfied, and compare its performance with supervised pose-estimation methods.
    \item \emph{Does the block-disentanglement principle remain effective beyond the theory?} When we additionally violate the assumptions underlying the block identifiability theorem, we test whether our approach still concentrates information about the relevant physical variables in the target block?
    \item \emph{Where in the pipeline is state information retained or lost?} We compare how much information about the true state is retained after visual compression, in the complete learned latent representation, and in the recovered active block. This lets us distinguish limitations from feature extraction from those from representation learning or block selection.
    \item \emph{How sensitive is the resulting representation to departures from the idealized setting?} We study the effect of factors such as the bottleneck dimension and other experimental variations to assess whether the observed disentanglement behavior persists as the representation and data-generating process become less tightly matched to the theoretical assumptions.
\end{itemize}


\begin{figure}[t]
    \centering
    \includegraphics[width=\textwidth]{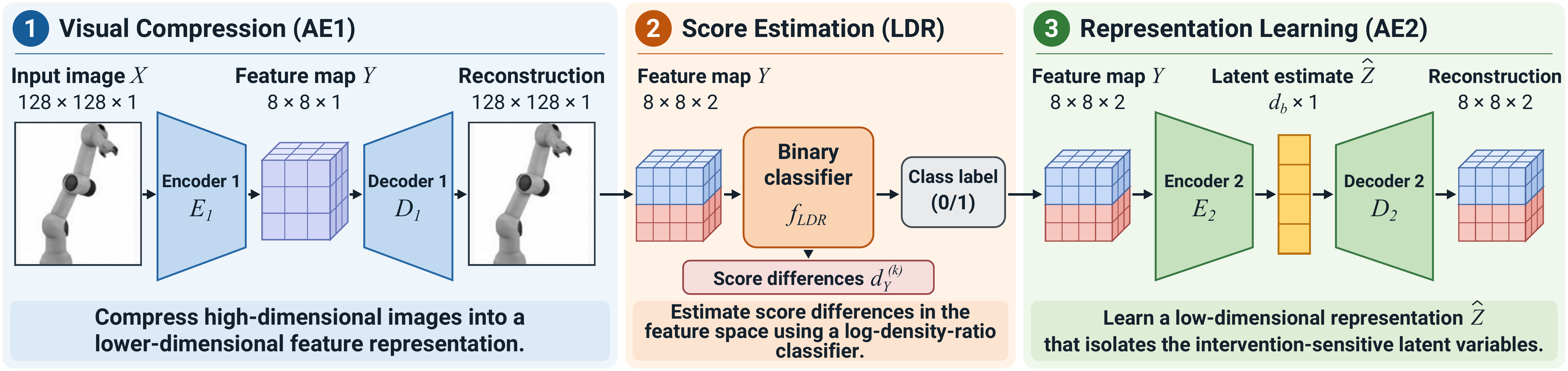}
    \caption{\small \textbf{Training pipeline.} A shared Autoencoder 1 (AE1) compresses each $128 \times 128 \times 1$ image $X$ into an $8 \times 8 \times 1$ feature vector $Y$. Features from both cameras are concatenated along the channel axis, forming an $8 \times 8 \times 2$ input to the LDR and AE2. AE2 then outputs the final $d_b \times 1$ latent estimate $\hatZ$.}
    \label{fig:data-pipeline}
\end{figure}

\paragraph{Evaluation metrics.} In order to answer the questions above, we specify appropriate metrics for both assumption-matched and embodied settings. In the assumption-matched setting, we consider
Mean Squared Error (MSE) as our main metric. To account for the elementwise diffeomorphism ambiguity of the recovered variables, we calibrate each estimated component using an \emph{affine} transformation fitted on a separate set of 1000 samples and report MSE on the held-out evaluation set.

To evaluate the embodied setting, we need a block-level measure of whether the learned representation preserves the information in the original latent variables. By \Cref{thm:block_id}, the representation within the active block is identifiable only up to an arbitrary diffeomorphism $\hatZhatI = \phi_\mcI(Z_\mcI)$.
Thus, we assess whether the learned active block retains sufficient information to recover $Z_\mcI$, which we denote \emph{active predictability}, and measure it using explained variance $R^2(Z_\mcI \mid \hatZhatI)$ under nonlinear kernel ridge regression. High active predictability indicates that the active block retains substantial information about the corresponding ground-truth latent. To localize where information is lost, we compare active predictability with two corresponding quantities at earlier stages of the data pipeline $Z \rightarrow X \rightarrow Y \rightarrow \hatZ
\rightarrow\hatZhatI$. Specifically, we report 
{\normalsize
\begin{equation}\label{eq:r2-metrics}
    \underbrace{R^2(Z_\mcI \mid \hatZhatI)}_{\text{active predictability}},
    \qquad
    \underbrace{R^2(Z_\mcI \mid \hatZ)}_{\text{all-latent predictability}},
    \qquad
    \underbrace{R^2(Z_\mcI \mid Y)}_{\text{feature predictability}}.
\end{equation}
}\!
The first term measures recovery from the active block, the second the full learned latent, and the third how much information about $Z_\mcI$ is contained in the feature map after the initial dimensionality reduction stage. These three quantities diagnose information loss at successive stages of the pipeline: feature extraction, representation learning, and construction of the active block.

\paragraph{Assumption-matched setting results.}
{\footnotesize
\begin{table}[b]
\centering
\caption{Comparison of MSE ($\text{rad}^2$) for RoboPEPP and our method in assumption-matched regime.
}
\label{tab:mse_results_baseline}
\footnotesize
\sisetup{table-format=1.3}
\begin{tabular}{@{} ll c *{6}{S} @{}}
\toprule
\textbf{Model} & \textbf{Training Labels} & {\textbf{Joint 1}} & {\textbf{Joint 2}} & {\textbf{Joint 3}} & {\textbf{Joint 4}} & {\textbf{Joint 5}} & {\textbf{Joint 6}} \\
\midrule
\multirow{1}{*}{Ours}     & none during training & 0.083 & 0.015 & 0.217 & 0.035 & 0.198 & 0.080 \\
\midrule
\multirow{4}{*}{RoboPEPP} & $1\%$   of dataset   & 0.136 & 0.039 & 0.237 & 0.053 & 0.253 & 0.200 \\
                          & $5\%$   of dataset   & 0.075 & 0.017 & 0.134 & 0.024 & 0.153 & 0.081 \\
                          & $10\%$  of dataset   & 0.030 & 0.010 & 0.072 & 0.022 & 0.091 & 0.063 \\
                          & $100\%$ of dataset   & 0.003 & 0.001 & 0.007 & 0.003 & 0.010 & 0.011 \\
\bottomrule
\end{tabular}
\end{table}
}
The setting for this experiment satisfies (C1)--(C6), and we use canonical CRL techniques. We compare the performance of our approach against a recent supervised baseline, RoboPEPP~\citep{goswami2025robopepp}, in \Cref{tab:mse_results_baseline}. For extensive ablations and additional results, see \Cref{app:regime_a}. RoboPEPP is among the few methods that can estimate pose when joint angles are unknown at test time (along with HPE~\citep{ban2024real} and RoboPOSE~\citep{labbe2021single}), and reports stronger performance than these alternatives. RoboPEPP uses a two-stage training procedure. First, a self-supervised I-JEPA backbone~\citep{assran2023self} is trained on the entire training dataset. In the second stage, a ``JointNet'' is trained to predict the six joint angles from the concatenated I-JEPA embeddings of the two camera views, using an $L_2$ loss and ground-truth joint-angle labels. \emph{In contrast, our CRL approach requires no labeled state data for training.}

Importantly, our approach achieves performance comparable to RoboPEPP when RoboPEPP is trained using labels for 5\% of the dataset (approximately 6.5K labeled poses, corresponding to 13K labeled images from two camera orientations). The key point is therefore not only the resulting estimation accuracy, but that we achieve comparable performance while learning the underlying state representation \textbf{without} training on labeled data.

{\footnotesize
\begin{table}[t]
    \vspace{-0.1in}
    \footnotesize
    \centering
    \caption{Embodied setting results across score regularization weight $\lambda$ and bottleneck dimension $d_b$.}
    \label{tab:embodied-main-results}
    \begin{tabular}{ccccc}
        \toprule
        $\lambda$ & $d_b$ & Active $R^2(Z_\mcI \mid \hatZhatI)$ & All-latent $R^2(Z_\mcI \mid \hatZ)$ & Feature $R^2(Z_\mcI \mid Y)$ \\
        \midrule
        \textbf{0}   & 16 & 0.33 & 0.72 & \multirow{12}*{\makecell{0.77 \\ (AE1 frozen)}} \\
        50  & 16 & 0.58 & 0.63 & \\
        100 & 16 & 0.63 & 0.66 & \\
        \cmidrule{1-4}
        \textbf{0}   & 32 & 0.23 & 0.77 & \\
        50  & 32 & 0.61 & 0.67 & \\
        100 & 32 & 0.55 & 0.66 & \\
        \cmidrule{1-4}
        50  & 8  & 0.58 & 0.60 & \\
        50  & 12 & 0.63 & 0.66 & \\
        50  & 16 & 0.58 & 0.63 & \\
        50  & 32 & 0.61 & 0.67 & \\
        \bottomrule
    \end{tabular}
    \vspace{-.1in}
\end{table}
}

\paragraph{Embodied setting results.}  We now move beyond the theoretically-guaranteed regime, with assumption violations V1--V3 all active simultaneously. We apply our block-disentanglement approach and report the results in \Cref{tab:embodied-main-results}. 
Specifically, we answer two main questions:
\begin{itemize}[leftmargin=0.2in, topsep=0in, itemsep=0in]
    \item \textbf{Does the approach remain effective under violations V1--V3?} With score regularization, active predictability approaches all-latent predictability, indicating that the regularizer retains much of the state information already present in $\hatZ$ within the active block $\hatZhatI$. In comparison, a reconstruction-only objective ($\lambda=0$) results in significantly lower active predictability: for $d_b=32$, $\lambda=0$ yields 0.23 whereas $\lambda=50$ yields 0.61, retaining most of the all-latent predictability of 0.67.
    \item \textbf{How sensitive is the representation to the assumption violations?} Violations V1--V2 arise from the video modality and are therefore not cleanly ablatable without changing the embodied setting itself. Violation V3, in contrast, can be varied directly through the bottleneck dimension. Increasing $d_b$ substantially changes the capacity and information content of the full representation, but has little effect on active predictability: across $d_b\in\{8, 12, 16, 32\}$, the active block recovery remains at essentially the same level. Thus, the active/nuisance contrast persists under substantial overparameterization, while the full representation continues to learn additional features that purely help reconstruction.
\end{itemize}

In summary, our main empirical observation is that the structural principles motivating CRL remain effective well beyond their validity regime. In the embodied setting, despite extensive violations of the theoretical conditions that underlie it, the block regularizer retains substantially more recoverable state information within the designated active coordinates than the unregularized baseline.

Finally, we provide additional results in~\Cref{app:regime_b}, including different episode generation setups and detailed training curves for~\Cref{tab:embodied-main-results}.

\section{Concluding Remarks} \label{sec:conclusion} 

We have proposed a two-step path from CRL identifiability theory toward label-free visual state estimation. In this first step, we weaken canonical intervention assumptions and establish block-disentanglement guarantees, showing that an active group of causal variables can remain identifiable even when individual variables cannot be recovered separately. In the second step, we show that the resulting principle remains effective for retaining relevant information in embodied visual settings that further violate the theoretical assumptions. Empirically, score-based CRL recovers robot joint states from images without state labels in the assumption-matched setting, and the block-based objective continues to be informative about the relevant physical variables under temporal dependence, imperfect interventions, and overcomplete representations. These two components are complementary. The theory characterizes what remains identifiable under weaker assumptions, while the experiments assess whether those principles remain effective as CRL moves toward practical visual data. To further narrow the theory-practice gap and realize CRL's important potential, a key next step is to extend the identifiability theory to settings with temporally dependent observations, approximate intervention effects, and unknown or overcomplete latent dimensionality.

\section*{AI Use Statement}

In this work, we used generative AI tools to
implement methods, and
provide feedback on research methodology or experiments.
We have not used generative AI tools to
generate synthetic data sets,
help develop theoretical models or conceptual frameworks,
formulate mathematical claims,
provide critical ingredients for proving mathematical claims,
assist in the writing of proofs,
propose or refine hypotheses,
support qualitative and thematic data analysis, and
interpret results; and
clean and reformat dataset, and
assist with translation,
are not applicable to this work. Additionally, we used generative AI tools to
create or edit software code,
draft parts of a research paper,
summarize or analyze existing literature,
brainstorming, and
edit a research paper to improve readability. We have reviewed all AI-assisted work.
LLM-generated code was verified and tested for correctness. LLMs were used to assist with high-level structural planning and localized refinements. All substantive claims, arguments, technical content, and the core draft were produced manually by the authors.
We take responsibility for the final content of this work, including text, claims or artifacts produced with the aid of generative AI.

\section*{Reproducibility Statement}

The relevant loss functions are fully specified in \Cref{sec:methodology} and all implementation details are included in \Cref{app:regime_a,app:regime_b,app:architecture}.

\bibliography{References/main,References/legacy_references,References/burak_addition}
\bibliographystyle{iclr2027_conference}

\newpage
\appendix
\crefalias{section}{appendix}
\crefalias{subsection}{appendix}
\crefalias{subsubsection}{appendix}



\startcontents[app]

\section*{Table of Contents}
\printcontents[app]{}{1}[2]{}

\section{Related Work and Limitations}\label{app:related-work}

\paragraph{Identifiable causal representation learning.} Without additional structure or distributional diversity, recovering latent variables from observations under an unknown nonlinear mixing function is generally impossible \citep{hyvarinen1999nonlinear,locatello2019challenging}. Causal representation learning (CRL) addresses this ambiguity by exploiting causal structure and changes across environments. A substantial literature establishes identifiability from interventions under different assumptions on the observation map, latent causal model, and intervention mechanisms \citep{squires2023linear,buchholz2023learning,zhang2023identifiability,vonkugelgen2023twohard,varici2024linear,varici2025score,acarturk2024sample,ng2025causal,lee2026beyond}. Other approaches include counterfactual CRL~\citep{vonkugelgen2021self}, unified frameworks to analyze interventional CRL along with nonlinear ICA and perturbation models~\citep{yi2026concept}, and temporal~\citep{lippe2022icitris} and multi-view CRL literature~\citep{yao2025unifying}. These results range from linear latent representations or observation models to general nonlinear transformations, and consider settings including hard, soft, known, and unknown interventions.

Interventional diversity is particularly relevant to our work. Existing nonlinear CRL guarantees commonly obtain elementwise disentanglement by exploiting sufficiently localized and diverse changes across interventional environments \citep{vonkugelgen2023twohard,varici2025score} or the nonlinear ICA connection under Gaussian additive noise models \citep{ng2025causal}. Our result instead asks what remains identifiable when such intervention structure is unavailable, replacing intervention-by-intervention requirements with collective conditions on the support and span of score differences. The precise comparison between these assumptions and our conditions is given in \Cref{sec:theory}.

\paragraph{Block disentanglement.} When individual latent coordinates cannot be identified, a weaker objective is to recover groups of variables up to an invertible transformation within each group. \citet{vonkugelgen2021self} identify the invariant content block from paired observations (also called counterfactual setting) whose augmentations preserve content while sufficiently varying the remaining style variables, allowing causal dependence between content and style.

Under polynomial transformation, \citet{ahuja2023interventional} exploit intervention-induced support independence to obtain block-affine guarantees. For general nonlinear transformation, \citet{liang2023causal} identify intervention-target blocks when the causal graph and target sets are known and sufficiently diverse perfect interventions satisfy a block score-rank condition. \cite{ahuja2024multi} use ancestrally closed soft intervention target sets (with proper distributional assumptions) to disentangle the non-intervened variables from the intervened ones. \citet{yao2025unifying} establish a broader \emph{invariant-block} identification result by combining information preservation with an invariance condition. Our \Cref{thm:block_id} establishes recovery of the intervention-affected block through collective score diversity, even when its individual coordinates cannot be separated. The discussion following the theorem compares the required intervention structure and diversity conditions with these earlier guarantees.

\paragraph{Visual robot pose estimation.} Visual robot pose estimation is most commonly formulated as a supervised prediction problem. DREAM \citep{lee2020camera} estimates robot pose through image-space keypoints, while RoboPose \citep{labbe2021single} iteratively refines pose and joint-angle estimates. More recent methods include HPE \citep{ban2024real} and RoboPEPP \citep{goswami2025robopepp}, the latter combining self-supervised visual pretraining with supervised joint-angle regression. These methods demonstrate strong pose-estimation performance, but ultimately rely on pose or joint-state supervision for the prediction task. In contrast, our objective is to recover the underlying physical state representation without per-frame state labels, using variation across intervention environments as the learning signal.

Classical alternatives include geometry- and marker-based systems, such as fiducial-marker approaches \citep{garrido2014automatic,garcia2023fiducial}. These methods provide a different tradeoff, relying on known geometric structure, calibration, or instrumentation rather than learned latent representations. Our experiments are not intended to replace these specialized pose-estimation pipelines; rather, robot pose provides a controlled embodied setting in which to study whether identifiable-CRL principles can be carried from high-dimensional observations toward state estimation under progressively weaker assumptions.

\paragraph{Limitations.} We list the main limitations of our paper.

\begin{itemize}[leftmargin=0.2in, topsep=0in, itemsep=0in]
    \item Our empirical analysis is limited to simulation studies in a controlled environment. This is a deliberate choice to maximize controllability of the experimental setting. However, in order to demonstrate practicality, CRL field needs to start using real-world data as well. This paper shows that such a transition does not have to be done abruptly.
    \item Proving robust estimation for both score estimation and disentanglement stages against using non-i.i.d.\ data would address violation V1 (temporal dependence) directly, and are likely a prerequisite for any theoretical treatment of video modality from the interventional CRL perspective.
    \item An approximate version of \Cref{thm:block_id} that bounds the degradation of block recovery in terms of the magnitude of off-block score deviations would be required to treat violation V2 (off-target effects) from a theoretically grounded perspective.
    \item There exists CRL approaches that can not only tolerate $d_b > n$ but can also identify $n$. While this would help with addressing violation V3 (overcomplete bottleneck), our empirical analysis implies that our method is currently robust against using an overcomplete learned representation.
    \item As shown in \Cref{app:regime_a}, violations of the injectivity of the transformation $f$ are very impactful: a consistent occlusion of a joint from an unfavorable camera angle impedes recovery of that joint, and requires a multi-camera setup to resolve.
\end{itemize}

\section{Proofs}\label{app:proof}

In this section, we provide the proof of \Cref{thm:block_id}, stated next for ease of reference.
\thmBlockID*{}
\begin{proof}
Since $(h,g)$ achieves perfect reconstruction and $f$ is an injective diffeomorphism, the composition
\begin{equation}\label{eq:proof_tau}
    \tau \;\triangleq\; h \circ f : \R^{n} \to \R^{n}
\end{equation}
is a diffeomorphism with $\hatZ = \tau(Z)$. The proof consists of showing that $\tau$ cannot mix the block $\mcI$ with its complement.

Due to the score difference function sparsity property in~\eqref{eq:dzk-is-sparse}, the score difference $d\supk$ is zero outside the coordinates $\Paplus(\mcI_k)$. \Cref{defn:mcI-anc-closed} posits $\Paplus(i) \subseteq \mcI$, which implies $[d\supk(z)]_j = 0$ for any $j \not\in \Paplus(\mcI_k) \subseteq \mcI$. Therefore, for any $j \in \mcI^c$ and $k$, we have $[d\supk(z)]_j = 0$. Accordingly, matrix $\bD(z)$ can be written in block format,
\begin{equation}\label{eq:proof_support}
    \bD(z) = \begin{bmatrix}
        [\bD(z)]_{\mcI, :} \\ 0
    \end{bmatrix} \ ,
\end{equation}
where the ``diagonal'' part $[\bD(z)]_{\mcI, :}$ has size $m \times K$. Due to the collective rank diversity condition \Cref{defn:dz-full-rank}, this diagonal term has full row rank---i.e., $\rank([\bD(z)]_{\mcI, :}) = m$---at almost every $z$. We note that \eqref{eq:d-transformation} implies the matrix identity
\begin{equation}\label{eq:proof_matrix}
    \bD_{\hatZ}(\hatZ) = \left[J_\tau(z)\right]^{-\top} \cdot \bD(z) \ .
\end{equation}
Therefore,
\begin{equation}
    \bD_{\hatZ}(\hatZ) = \begin{bmatrix}
        [[J_\tau(z)]^{-\top}]_{:, \mcI} & [[J_\tau(z)]^{-\top}]_{:, \mcI^c}
    \end{bmatrix}
    \cdot \begin{bmatrix}
        [\bD(z)]_{\mcI, :} \\ 0
    \end{bmatrix}
    = [[J_\tau(z)]^{-\top}]_{:, \mcI} \cdot [\bD(z)]_{\mcI, :} \ .
\end{equation}
The theorem statement requires the rows of $\bD_{\hatZ}(\hatZ)$ indexed by $\hat{\mcI}^c$ to be zero. That is,
\begin{equation}
    \bzero_{(n-m) \times K} = [\bD_{\hatZ}(\hatZ)]_{\hat{\mcI}^c, :} = [[J_\tau(z)^{-\top}]_{\hat{\mcI}^c, \mcI} \cdot [\bD(z)]_{\mcI, :} \ .
\end{equation}

Applying transpose,
\begin{equation}
    \bzero_{K \times (n-m)}
    =
    [\bD(z)^\top]_{:, \mcI}
    [J_\tau(z)^{-1}]_{\mcI,\hat{\mcI}^c}.
\end{equation}
Since $[\bD(z)]_{\mcI,:}$ has full row rank,
$[\bD(z)^\top]_{:,\mcI}$ has full column rank. Hence
\begin{equation}
    [J_\tau(z)^{-1}]_{\mcI,\hat{\mcI}^c}
    = \bzero_{m\times(n-m)}.
\end{equation}
This holds almost everywhere by the collective rank diversity condition, and therefore everywhere by continuity of $J_\tau^{-1}$.

Partitioning the rows of $J_\tau^{-1}$ according to $(\mcI,\mcI^c)$ and its columns according to $(\hat{\mcI},\hat{\mcI}^c)$, we thus have
\begin{equation}
    J_\tau^{-1}
    =
    \begin{bmatrix}
        [J_\tau^{-1}]_{\mcI,\hat{\mcI}}
        & \bzero \\
        [J_\tau^{-1}]_{\mcI^c,\hat{\mcI}}
        & [J_\tau^{-1}]_{\mcI^c,\hat{\mcI}^c}
    \end{bmatrix}.
\end{equation}
Since $J_\tau^{-1}$ is invertible, both diagonal blocks are invertible. Consequently, by block matrix inversion,
\begin{equation}
    J_\tau
    =
    \begin{bmatrix}
        [J_\tau]_{\hat{\mcI},\mcI}
        & \bzero \\
        [J_\tau]_{\hat{\mcI}^c,\mcI}
        & [J_\tau]_{\hat{\mcI}^c,\mcI^c}
    \end{bmatrix},
\end{equation}
where $[J_\tau]_{\hat{\mcI},\mcI}$ is invertible. Therefore,
\begin{equation}
    \frac{\partial \hatz_j}{\partial z_k}=0
    \qquad
    \forall j\in\hat{\mcI},\quad
    \forall k\in\mcI^c.
\end{equation}
Hence $\hatz_{\hat{\mcI}}$ depends only on $Z_{\mcI}$.

Likewise, the upper-right zero block of $J_\tau^{-1}$ implies that $Z_{\mcI}$ depends only on $\hatz_{\hat{\mcI}}$. Thus, there exists a map $\tau_{\mcI}$ such that
\begin{equation}
    \hatz_{\hat{\mcI}}=\tau_{\mcI}(Z_{\mcI}),
\end{equation}
whose inverse is given by the corresponding restriction of $\tau^{-1}$. Moreover, $[J_\tau]_{\hat{\mcI},\mcI}$ is invertible. Therefore $\tau_{\mcI}$ is a diffeomorphism, proving block disentanglement.
\end{proof}

\section{Embodied Regime}\label{app:regime_b}

This appendix provides additional implementation and evaluation details for the embodied setting of \Cref{sec:regimes}. Unless otherwise stated, all reported experiments use $K=n=6$ environments, with one \emph{nominally} targeted joint per environment. Hence, the union of the intervention targets is $\mcI=[n]=\{1,\ldots,6\}$. We designate the corresponding first six learned coordinates as $\hat{\mcI}=\{1,\ldots,6\}$; when $d_b>6$, the remaining coordinates $\hat{\mcI}^{c}=\{7,\ldots,d_b\}$ form the nuisance portion of the learned representation.

Ground-truth joint angles are not used to train AE1, the LDR estimators, or AE2. They are used only for evaluation through the $R^2$ metrics defined in \Cref{eq:r2-metrics}. All results and training curves reported in this section correspond to a single fixed random seed and to the final training iterate.

\subsection{Data Generation}\label{app:embodied-data}

Data collection replaces i.i.d.\ rendering with continuous trajectories. We construct six environments, one for each robot joint. Each environment is generated from 4000 closed-loop trajectories and differs through the goal-sampling distribution used for its nominally targeted joint.

Each episode proceeds as follows.
\begin{enumerate}[leftmargin=0.2in, topsep=0in, itemsep=0in]
    \item Sample a start configuration and a goal configuration. In environment $k$, the goal
    distribution prioritizes displacement of joint $k$ (sampled uniformly at random with minimum displacement of $0.5$ radians), while the remaining joint targets are changed in a small amount (shifted by a uniformly random value less than $0.05$ radians).
    \item Drive the arm from its start configuration toward the sampled goal using a proportional controller. Every simulator step is rendered and retained, so consecutive observations differ by one controller step.
    \item Terminate the trajectory once the goal is reached within a prescribed tolerance or after a fixed step budget (max. 40 frames), and reset the environment before generating the next trajectory.
\end{enumerate}

We use the same two-camera rendering configuration: one at 45 and other 135 degrees with respect to the arm's natural position, with observations from the two views concatenated after AE1 feature extraction.

Start states are sampled uniformly at random from the valid angles. The simulator itself rejects malformed configurations (e.g., where the robot arm intersects the table), so no additional processing during generation were needed.

The target distribution at the default configuration (the default ``embodied regime'' setup) is chosen such that in environment $k$, the $k$-th joint angle is set to be significantly different (difference $\geq 0.5$ radians) from the initial state. The off-target joints are instead targeted with relatively smaller perturbations ($<0.05$ radians).

The proportional controller gain is set to $1$; however, the trajectories were stable and smooth across the configuration space. Maximum trajectory length was set at $40$ frames.

The resulting data exhibit strong temporal dependence within each trajectory. Resets limit dependence across episodes but do not remove the correlation between consecutive frames within one episode. Moreover, targeting one joint does not yield an atomic intervention in the resulting trajectory: closed-loop motion generally induces simultaneous motion of other joints. Thus, the environment label indicates the joint whose displacement is emphasized rather than an intervention that changes only a single latent coordinate.

In the finalized experiments we use all six environments jointly, so $\mcI=[n]$. Consequently, off-target motion makes the individual environments non-atomic and couples their intervention effects.

\subsection{Score Difference Estimation}\label{app:embodied-ldr}

Unlike the assumption-matched setting, the embodied regime does not contain a separately generated observational dataset. Instead, we exploit the directional structure of each controlled trajectory to construct the two distributions required by the log-density-ratio estimator.

For every trajectory in environment $k$, we divide the sequence into two non-overlapping temporal halves. Samples from the first half form one classifier class and samples from the second half form the other. We then train one binary LDR classifier $f_{\mathrm{LDR}}^{(k)}$ for each environment. As described in the main pipeline, its logit gradient provides the corresponding feature-space score difference estimate.

In the embodied setting, this temporal-half contrast is an \emph{empirical proxy} motivated by the theoretical score-difference construction, rather than a direct estimate of the interventional score difference in \Cref{thm:block_id}.

We deliberately use non-overlapping temporal halves rather than frame-by-frame contrasts between nearby portions of the same trajectory. At finite sample size, heavily overlapping temporal supports would cause the classifier to observe nearly indistinguishable samples from the two classes over much of the trajectory, thus substantially weakening the density-ratio signal. The half-trajectory construction instead provides a coarser but more separable contrast while retaining the continuous-video nature of the data.

\subsection{Training and Evaluation Details}\label{app:embodied-training}

We follow the three-stage training pipeline described in \Cref{sec:Results}. AE1 is first trained to compress each $128\times128\times1$ camera image to an $8\times8\times1$ feature representation. The AE1 parameters are then frozen. Features from the two camera views are concatenated to form an $8\times8\times2$ input to the LDR estimators and AE2.

For the embodied experiments, all trainable stages use Adam optimizer, use learning rate of $10^{-4}$, are trained for $10$ epochs and using batch size $256$. AE1 was trained on 16 H100s whereas LDR and AE2 models were trained on a single H100.

In AE1 and AE2 runs, we process the dataset frame-by-frame as if the data was i.i.d.; consequently, we observe $2.8\times 10^4$ optimization steps in the reported runs. In LDR, trajectories themselves were batched, and for each trajectory one frame from each half is selected uniformly at random to train the model in that epoch. Therefore, the classes were by definition balanced during training.

For AE2, the scale-regularization weight $\lambda_s$ is fixed at $\lambda_s=1$. We train bottleneck dimensions
\begin{equation}
    d_b\in\{8,12,16,32\},
\end{equation}
with the first six coordinates designated as $\hat{\mcI}$ and the remaining $d_b-6$ coordinates are treated as nuisance coordinates. The AE2 objective is the block-regularized loss in \Cref{eq:loss_block}. We vary the score-regularization weight $\lambda$ in the sensitivity experiments below; the bottleneck-dimension sweep uses $\lambda=50$.

\paragraph{Evaluation protocol.} For the embodied setting, we report the three predictability metrics introduced in \Cref{eq:r2-metrics}: feature predictability $R^2(Z_{\mcI}\mid Y)$, all-latent predictability $R^2(Z_{\mcI}\mid\hat Z)$, and active predictability $R^2(Z_{\mcI}\mid\hat Z_{\hat{\mcI}})$.

These quantities are estimated using multivariate nonlinear kernel ridge regression with an RBF kernel on held-out labeled states. The kernel bandwidth is selected using the median-distance heuristic, i.e., from the median pairwise distance among the regression inputs. These labels are used only for evaluation and do not enter representation learning.
We use two separate sets of 1000 samples each to fit KRR and compute the reported $R^2$ values.

\subsection{Sensitivity and Additional Results}\label{app:regime-b-additional-results}

We next examine the two experimental factors for which finalized sweeps are available: the learned bottleneck dimension $d_b$ and the score-regularization weight $\lambda$. These experiments are intended to complement the aggregate results of \Cref{tab:embodied-main-results} by showing the training dynamics underlying the reported final values.

\subsubsection{Bottleneck Dimension}

\begin{figure}[htbp]
    \centering
    \includegraphics[width=0.6\textwidth]{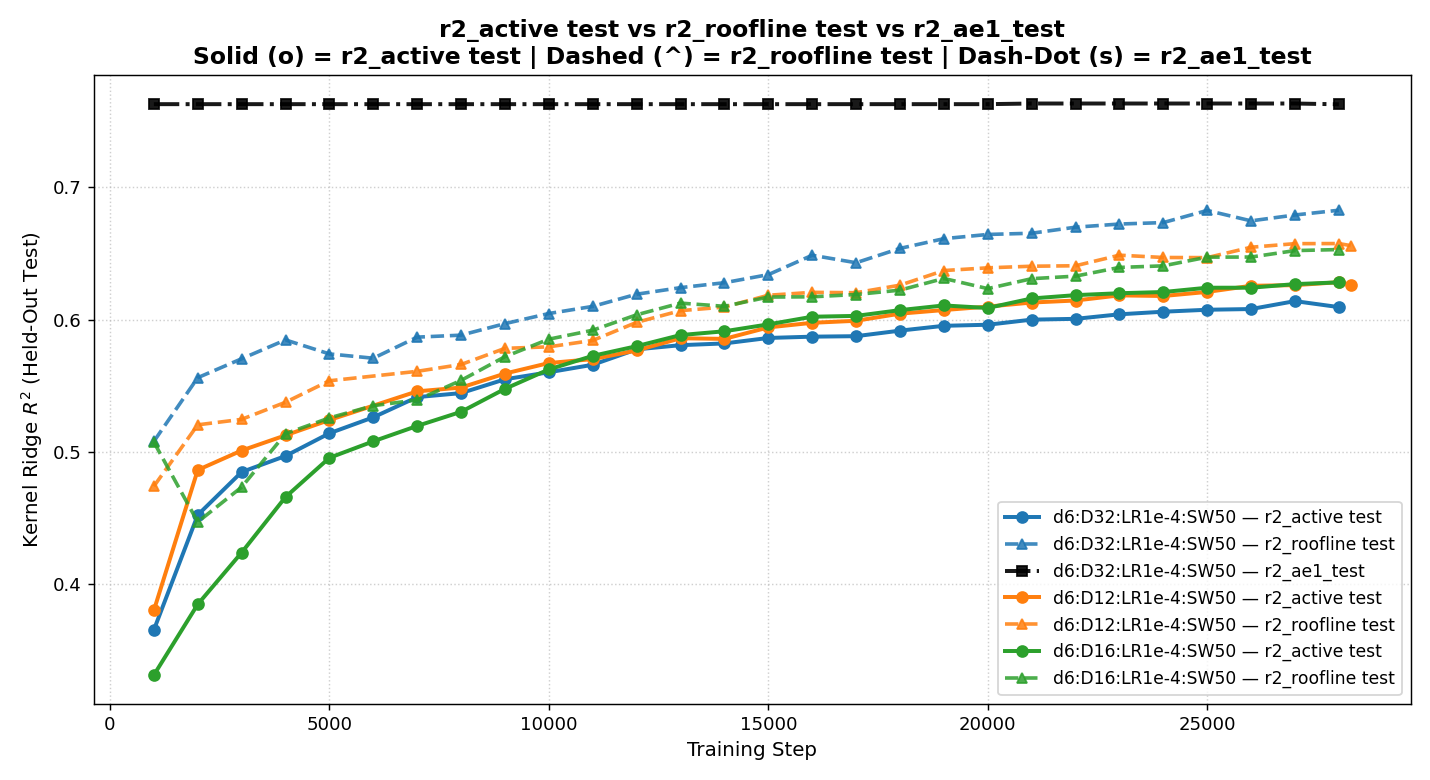}
    \caption{Training dynamics across bottleneck dimensions $d_b$ ($\lambda=50$).}
    \label{fig:r2-vs-db-lambda=50}
\end{figure}

\Cref{fig:r2-vs-db-lambda=50} shows the evolution of the three predictability quantities as the bottleneck dimension is varied. Increasing $d_b$ substantially increases the capacity available to the complete representation $\hat Z$, but active predictability remains comparatively stable. In particular, the information about the true state retained in the designated coordinates $\hat Z_{\hat{\mcI}}$ remains close to that available from the complete learned representation over the tested bottleneck dimensions.

This behavior is consistent with the main result in \Cref{tab:embodied-main-results}: the block regularizer concentrates much of the recoverable state information into the designated six-dimensional subspace even when the learned representation contains substantially more coordinates than the true latent state.

\subsubsection{Score-Regularization Weight}

\begin{figure}[htbp]
    \centering
    \includegraphics[width=0.6\textwidth]{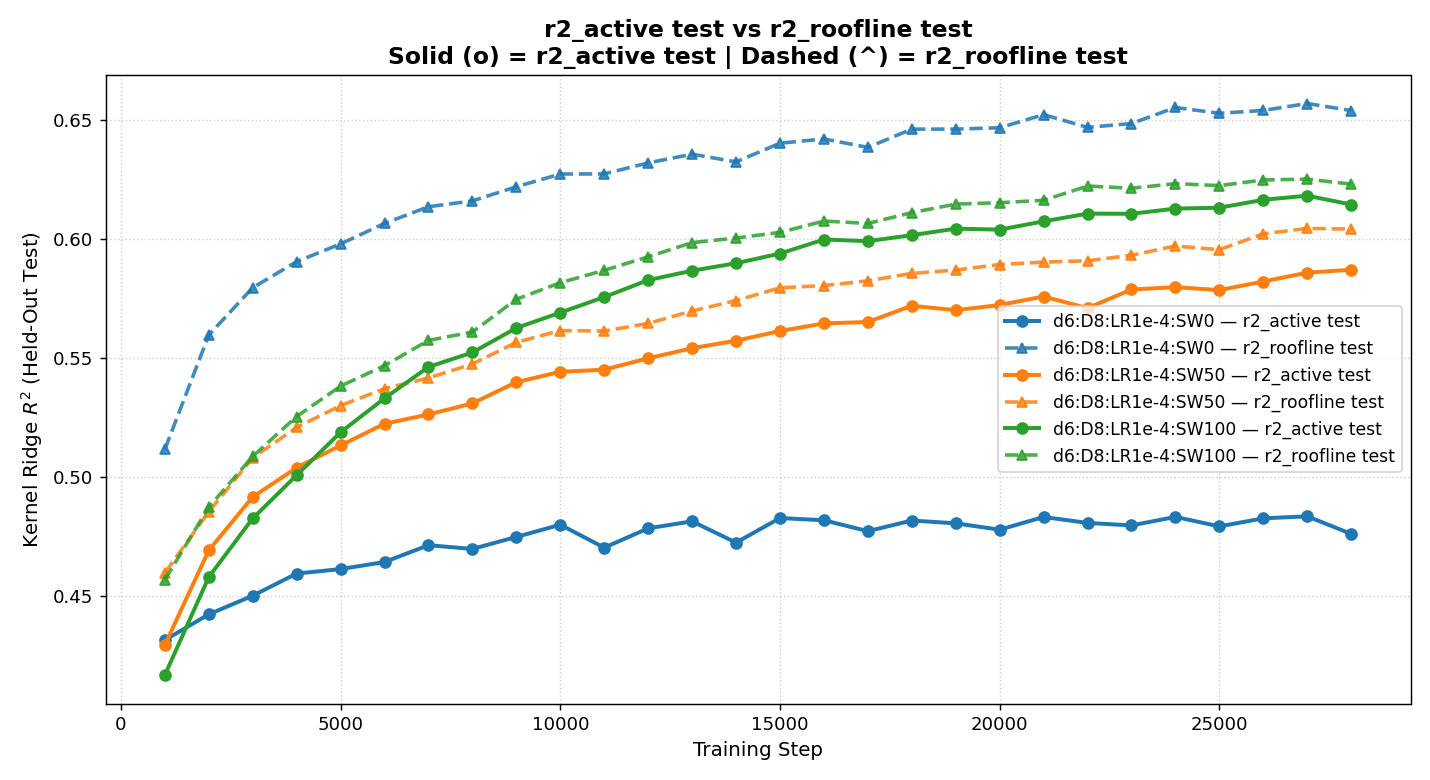}
    \caption{Training dynamics across score-regularization weights $\lambda$ ($d_b=8$).}
    \label{fig:r2-vs-lambda-db=8}
\end{figure}

\begin{figure}[htbp]
    \centering
    \includegraphics[width=0.6\textwidth]{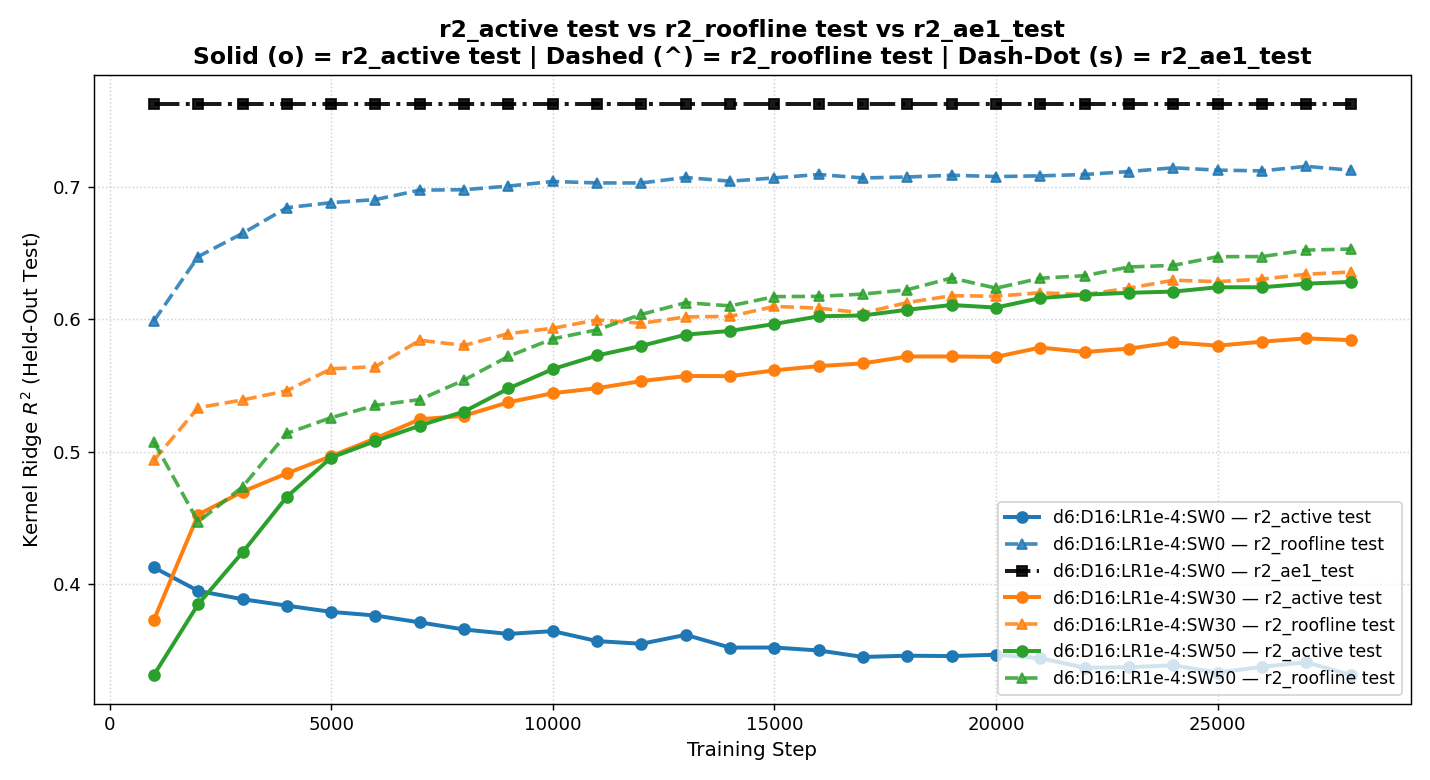}
    \caption{Training dynamics across score-regularization weights $\lambda$ ($d_b=16$).}
    \label{fig:r2-vs-lambda-db=16}
\end{figure}

\begin{figure}[htbp]
    \centering
    \includegraphics[width=0.6\textwidth]{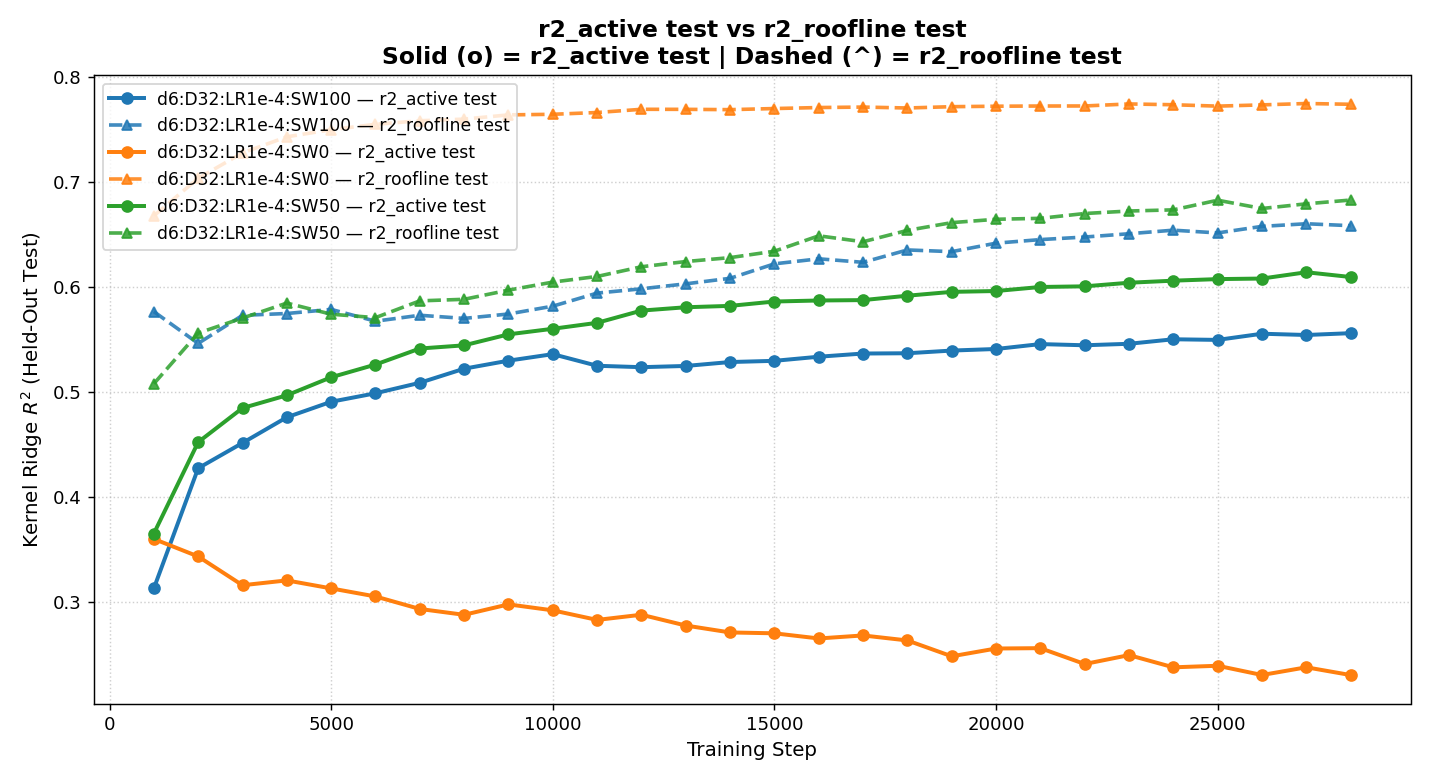}
    \caption{Training dynamics across score-regularization weights $\lambda$ ($d_b=32$).}
    \label{fig:r2-vs-lambda-db=32}
\end{figure}

\Cref{fig:r2-vs-lambda-db=8,fig:r2-vs-lambda-db=16,fig:r2-vs-lambda-db=32} show the training dynamics as the score-regularization weight is varied for $d_b\in\{8,16,32\}$. The reconstruction-only baseline, $\lambda=0$, preserves substantial state information in the complete representation but does not consistently concentrate that information into $\hat Z_{\hat{\mcI}}$. Introducing the block regularizer substantially closes this gap.

The dependence on $\lambda$ is not monotone: increasing the regularization weight beyond a useful range does not necessarily improve active predictability. The relevant comparison is therefore between the absence and presence of a substantial block-disentanglement signal, rather than between progressively larger values of $\lambda$.

\subsubsection{Multi-Dominant-Joint Trajectories}

Finally, we consider ``multi-dominant-joint trajectories''. Specifically, in environment $k$, we choose \emph{two} joints, $k$ and $k + 2 \pmod{6}$, as targets. In this regime, we also increase the off-target effects from $0.05$ to $0.1$.

\begin{figure}[htbp]
    \centering
    \includegraphics[width=0.6\textwidth]{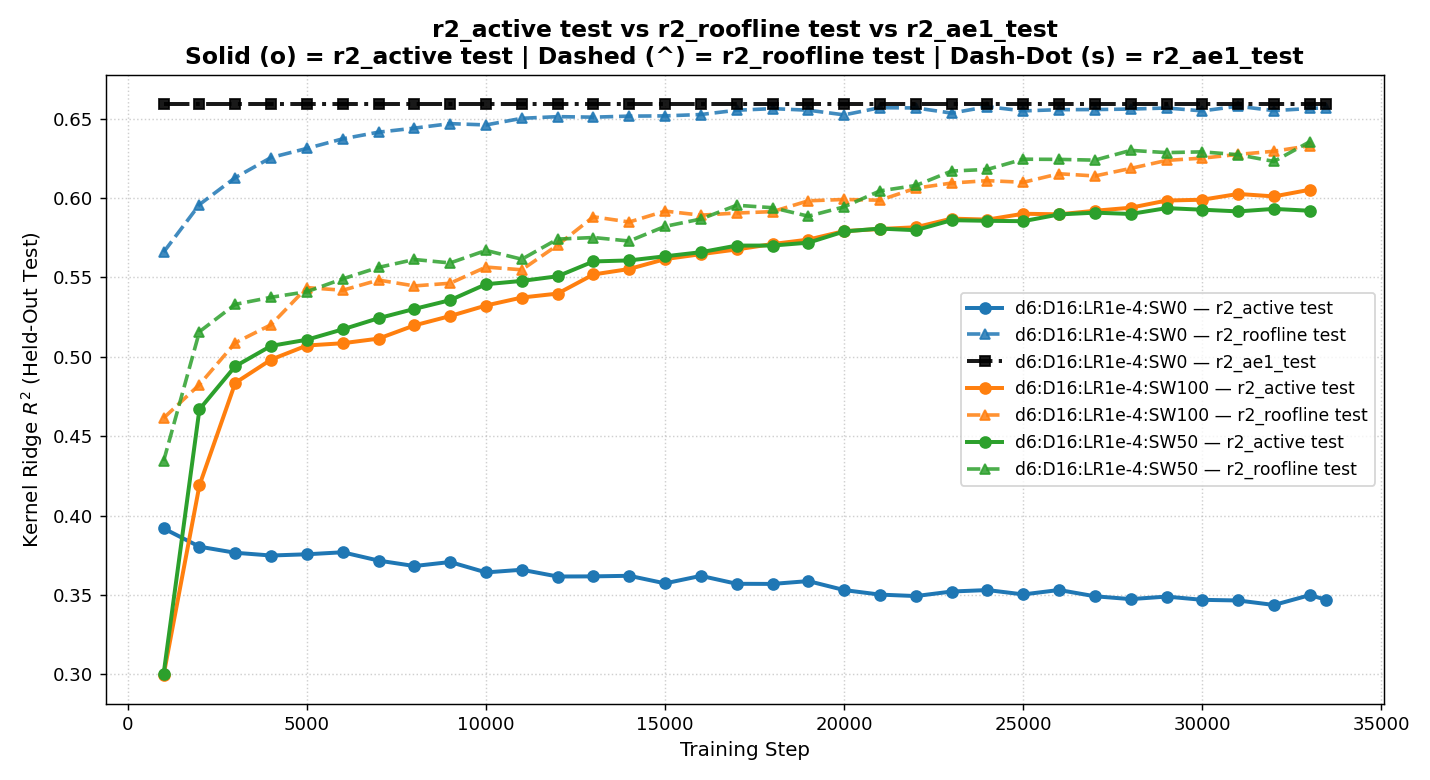}
    \caption{$R^2$ under multi-dominant-joint data ($d_b=16$)}
    \label{fig:r2-vs-lambda-multijoint-db=16}
\end{figure}

In \Cref{fig:r2-vs-lambda-multijoint-db=16}, we provide the equivalent of \Cref{fig:r2-vs-lambda-db=16} under the multi-dominant-joint data. This experiment increases the extent to which each environment departs from an atomic intervention and therefore provides an additional test of whether the block-based objective remains useful when intervention effects are distributed across multiple coordinates.

The results show that the multi-target setting is almost indistinguishable from the single-dominant-target regime. This is consistent with the theoretical motivation, which does not require atomic interventions. Furthermore, it demonstrates that our implementation is not dependent on the single-node-like score signatures in data.

Overall, the finalized sweeps support the same qualitative conclusion as the main embodied results: the score-based block regularizer concentrates substantially more recoverable state information into the designated coordinates than reconstruction alone, while the result remains relatively insensitive to substantial overparameterization of the learned representation.

\section{Assumption-Matched Regime} \label{app:regime_a}

In this section, we additionally report a standard metric in elementwise CRL: Mean Correlation Coefficient (MCC) \citep{khemakhem2020ice}. Furthermore, MSE requires the estimated latent variable estimates to be compared against the ground-truth values. In order to do so, we first calibrate the estimates using a held-out set of 1000 labelled samples using elementwise affine regression.

\subsection{Data Generation}

\paragraph{Observational and Interventional Distributions.} Our inference process is unsupervised, meaning we do not require pose-labeled images. The data needed for inference are generated via interventions, which are applied by manipulating the joints individually and capturing the resulting images. Following \Cref{sec:methodology}, for each joint variable $Z_i$, we have one observational and two interventional distributions, $p$, $q_{i}$, and $\bar{q}_{i}$. Our learning algorithm is oblivious to any metadata or statistics associated with these distributions and learns from a random collection of poses generated before and after interventions. We denote the images drawn from $q_i$ and $\bar{q}_i$ by `0' and `1' respectively. 
Details of these distributions are described next, and exact parameterizations are given in \Cref{app:int_distributions}.

\paragraph{Data Generation.} We generate our dataset using a Franka Emika Panda arm in a Panda-Gym simulator~\citep{gallouedec2021panda}. The arm has six primary joint angles, which are marked in~\Cref{fig:robot_joints_4}. Our data generation process consists of two stages. First, we focus on a setup with a single camera, generating interventions for the joints whose movements are confined to the camera's plane, i.e., the 2D  projection plane, perpendicular to the camera's viewing direction. In the second stage, to accommodate the out-of-plane motions from other joints, which cannot be captured from a single viewpoint, we expand the dataset using images generated by two camera angles. The images in the dataset are converted to grayscale, resulting in a shape of  $128 \times 128 \times 1$ for each image.

\begin{figure}[htbp]
    \centering
    \includegraphics[width=0.3\textwidth]{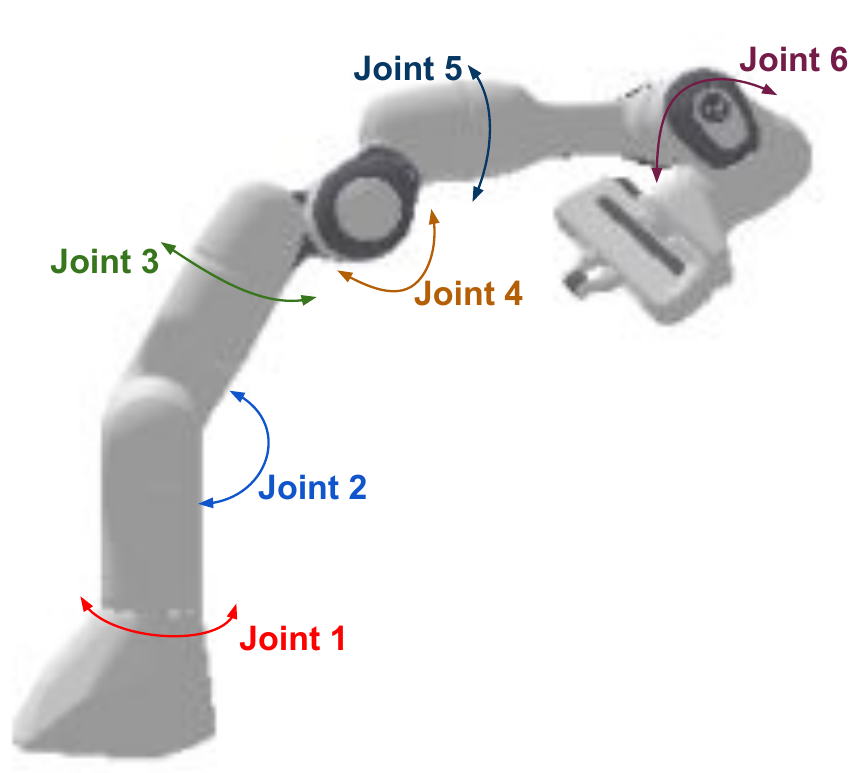}
    \caption{The six primary joint angles of the Franka Emika Panda arm targeted for intervention, together with their respective axes of rotation. Note that the final (7th) joint is kept fixed.}
    \label{fig:robot_joints_4}
\end{figure}

\noindent\emph{Data on In-Plane Joints using a Single Camera.} \label{Single_camera_text}
We first focus on joints 2, 4, and 6, as their movement primarily results in motion within the camera's plane, and a single camera view is sufficient to cover the plane. Each data point in this dataset consists of a set of six interventional images (two for each joint) anchored by a single observational state. This observational state is generated by sampling a configuration for all six joints from a base truncated normal distribution. From this anchor, we create two distinct ``hard interventions'' for each of the target joints (2, 4, and 6). To perform an intervention on a specific joint, its angle is resampled from a distribution with a mean shifted far from the observational mean, while all other joint angles are held constant. This results in one observational image and six interventional images per data point.

\noindent\emph{Data on All Joints using Two Cameras.} \;
To extend our analysis to all six degrees of freedom, we also need to consider joints 1, 3, and 5, whose actuation causes significant out-of-plane motion. To ensure full observability, we captured every pose from two distinct camera angles ($45^{\circ}$ and $135^{\circ}$ yaw). We augment the dataset such that each data point now begins with an observational pose captured from both cameras (2 images). Subsequently, we perform two hard interventions on each of the six joints ($j \in \{1, \dots, 6\}$). Each of these 12 interventional pose distributions is captured by both camera angles. Thus, a single complete data point in this extended dataset consists of 26 images: 2 observational images plus 24 interventional images (6 joints $\times$ 2 interventions/joint $\times$ 2 cameras/pose).

\subsubsection{Independent Interventional Distributions}\label{app:int_distributions}

For the single-camera experiments, a single set of sampling distributions, detailed in \Cref{tab:singleTrunc}, is used for both training and evaluation. As such, the concepts of in-distribution (ID) and out-of-distribution (OOD) do not apply in this case. For the more comprehensive two-camera experiments, we distinguish between ID and OOD data to evaluate model generalization. The ID dataset, which is used for training the Independent and Occlusion models, is generated using the parameters in \Cref{two_camera_distributions}. To test for robustness to distributional shifts, we created a separate OOD test set with modified observational parameters, as shown in \Cref{tab:ood_distributions}. The experiment with causal model over the joints is also evaluated on this OOD test set. We use a truncated normal distribution, denoted by $\mathcal{TN}_{[a,b]}(\mu, \sigma^2)$, which represents a normal distribution with mean $\mu$ and variance $\sigma^2$ truncated to the interval $[a, b]$.s

\begin{table}[ht!]
  \centering
  \small
  \caption{Sampling distributions for observational and interventional settings for single camera setup.}
  \label{tab:singleTrunc}
  \renewcommand{\arraystretch}{1.15}
  \begin{tabular}{@{}llc@{}}
    \toprule
    \textbf{Joint} & \textbf{Scenario} & \textbf{Distribution} \\
    \midrule
    \multirow{3}{*}{2} & Observational  & $\mathcal{TN}_{[-1.5,\,1.5]}(0,\,1)$ \\[2pt]
                       & Intervention~1 & $\mathcal{TN}_{[-1.5,\,1.5]}(-0.75,\,0.5)$ \\
                       & Intervention~2 & $\mathcal{TN}_{[-1.5,\,1.5]}(\;\;0.75,\,0.5)$ \\[2pt]
    \midrule
    \multirow{3}{*}{4} & Observational  & $\mathcal{TN}_{[-1.5,\,1.5]}(0,\,1)$ \\[2pt]
                       & Intervention~1 & $\mathcal{TN}_{[-1.5,\,1.5]}(-0.75,\,0.5)$ \\
                       & Intervention~2 & $\mathcal{TN}_{[-1.5,\,1.5]}(\;\;0.75,\,0.5)$ \\[2pt]
    \midrule
    \multirow{3}{*}{6} & Observational  & $\mathcal{TN}_{[0,\,3]}(1.5,\,1)$ \\[2pt]
                       & Intervention~1 & $\mathcal{TN}_{[0,\,3]}(2.25,\,0.5)$ \\
                       & Intervention~2 & $\mathcal{TN}_{[0,\,3]}(\;\;0.75,\,0.5)$ \\
    \bottomrule
  \end{tabular}
  \label{single_camera_distributions}
\end{table}

\begin{table}[ht!]
  \small
  \centering
  \caption{Sampling distributions for the in-distribution (ID) dataset. These truncated normal distributions define the observational and interventional data used for the two-camera Independent and Occlusion experiments.}
  \label{tab:twoCamera_trunc}
  \renewcommand{\arraystretch}{1.15}
  \begin{tabular}{@{}llc@{}}
    \toprule
    \textbf{Joint} & \textbf{Scenario} & \textbf{Distribution} \\
    \midrule
    \multirow{3}{*}{1} & Observational  & $\mathcal{TN}_{[0,\,3]}(1.2,\,0.4)$ \\[2pt]
                       & Intervention~1 & $\mathcal{TN}_{[0,\,3]}(2.0,\,0.4)$ \\
                       & Intervention~2 & $\mathcal{TN}_{[0,\,3]}(\;\;0.6,\,0.4)$ \\[2pt]
    \midrule
    \multirow{3}{*}{2} & Observational  & $\mathcal{TN}_{[-1.5,\,1.5]}(0,\,0.4)$ \\[2pt]
                       & Intervention~1 & $\mathcal{TN}_{[-1.5,\,1.5]}(0.7,\,0.4)$ \\
                       & Intervention~2 & $\mathcal{TN}_{[-1.5,\,1.5]}(\;\;-0.7,\,0.4)$ \\[2pt]
    \midrule
    \multirow{3}{*}{3} & Observational  & $\mathcal{TN}_{[-1.5,\,1.5]}(0,\,0.4)$ \\[2pt]
                       & Intervention~1 & $\mathcal{TN}_{[-1.5,\,1.5]}(0.7,\,0.4)$ \\
                       & Intervention~2 & $\mathcal{TN}_{[-1.5,\,1.5]}(\;\;-0.7,\,0.4)$ \\[2pt]
    \midrule
    \multirow{3}{*}{4} & Observational  & $\mathcal{TN}_{[-1.5,\,1.5]}(0,\,0.4)$ \\[2pt]
                       & Intervention~1 & $\mathcal{TN}_{[-1.5,\,1.5]}(0.9,\,0.4)$ \\
                       & Intervention~2 & $\mathcal{TN}_{[-1.5,\,1.5]}(\;\;-0.9,\,0.4)$ \\[2pt]
    \midrule
    \multirow{3}{*}{5} & Observational  & $\mathcal{TN}_{[-1.5,\,1.5]}(0,\,0.4)$ \\[2pt]
                       & Intervention~1 & $\mathcal{TN}_{[-1.5,\,1.5]}(0.9,\,0.4)$ \\
                       & Intervention~2 & $\mathcal{TN}_{[-1.5,\,1.5]}(\;\;-0.9,\,0.4)$ \\[2pt]
    \midrule
    \multirow{3}{*}{6} & Observational  & $\mathcal{TN}_{[0,\,3]}(1.5,\,0.4)$ \\[2pt]
                       & Intervention~1 & $\mathcal{TN}_{[0,\,3]}(2.4,\,0.4)$ \\
                       & Intervention~2 & $\mathcal{TN}_{[0,\,3]}(\;\;0.7,\,0.4)$ \\
    \bottomrule
    \label{two_camera_distributions}
  \end{tabular}
\end{table}

\begin{table}[htbp]
\centering
\caption{Observational sampling distributions for the out-of-distribution (OOD) dataset. These parameters define the OOD test sets for the Two-Camera Independent, Causal, and Occlusion experiments. The interventional distributions for the OOD dataset remain identical to those of the in-distribution dataset, as defined in \Cref{two_camera_distributions}.}
\label{tab:ood_distributions}
\renewcommand{\arraystretch}{1.2} 

\begin{tabular}{@{}cc@{}}
\toprule
\textbf{Joint} & \textbf{OOD Observational Distribution} \\
\midrule
1 & $\mathcal{TN}_{[0,\,3]}(1.2,\,0.4)$ \\
2 & $\mathcal{TN}_{[-1.5,\,1.5]}(0.0,\,0.4)$ \\
3 & $\mathcal{TN}_{[-1.5,\,1.5]}(0.0,\,0.4)$ \\
4 & $\mathcal{TN}_{[-1.5,\,1.5]}(0.8,\,0.4)$ \\
5 & $\mathcal{TN}_{[-1.5,\,1.5]}(0.0,\,0.4)$ \\
6 & $\mathcal{TN}_{[0,\,3]}(0.5,\,0.4)$ \\
\bottomrule
\end{tabular}
\end{table}

\subsubsection{Causal Dataset Generation} \label{appendix causal dataset generation}

\begin{figure}[ht]
  \centering  \includegraphics[width=0.4\textwidth]{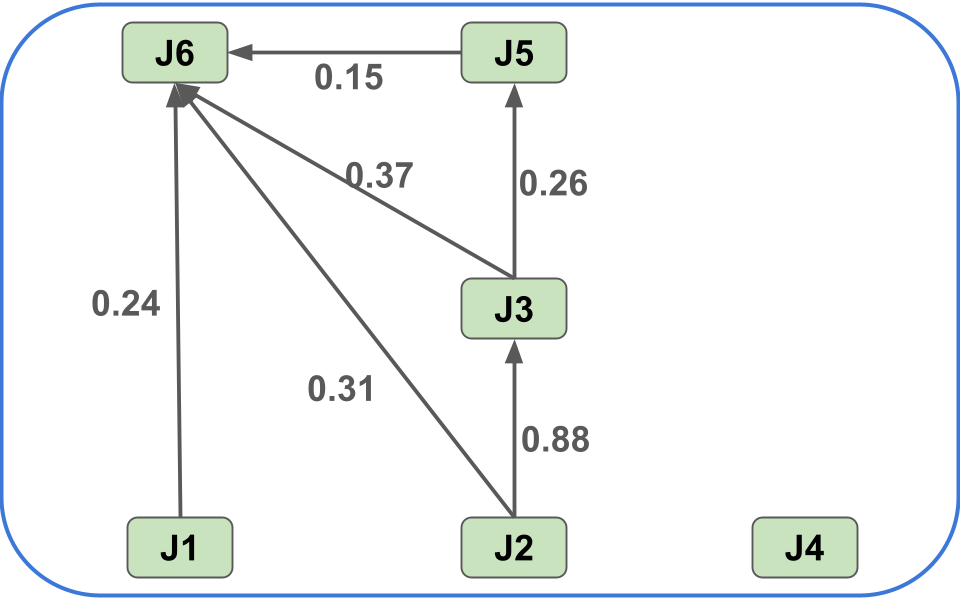}
  \caption{The causal model of the robot joints used to generate the dataset. In this graph, J\textit{i} represents the angle of joint \textit{i}.}
  \label{fig:causal model}
\end{figure}
\begin{table}[ht]
  \small
  \centering
  \caption{Sampling distributions for observational and interventional settings for causal dataset corresponding to the causal graph \Cref{fig:causal model}}
  \label{tab:noise_models}
  \renewcommand{\arraystretch}{1.15}
  \begin{tabular}{@{}llc@{}}
    \toprule
    \textbf{Joint} & \textbf{Scenario} & \textbf{Distribution} \\
    \midrule
    \multirow{3}{*}{1} & Observational  & $\mathcal{TN}_{[0,\,3]}(1.2,\,0.4)$ \\[2pt]
                       & Intervention~1 & $\mathcal{TN}_{[0,\,3]}(2.0,\,0.4)$ \\
                       & Intervention~2 & $\mathcal{TN}_{[0,\,3]}(0.6,\,0.4)$ \\[2pt]
    \midrule
    \multirow{3}{*}{2} & Observational  & $\mathcal{TN}_{[-1.5,\,1.5]}(0,\,0.4)$ \\[2pt]
                       & Intervention~1 & $\mathcal{TN}_{[-1.5,\,1.5]}(0.7,\,0.4)$ \\
                       & Intervention~2 & $\mathcal{TN}_{[-1.5,\,1.5]}(-0.7,\,0.4)$ \\[2pt]
    \midrule
    \multirow{3}{*}{3} & Observational  & Not a root node \\[2pt]
                       & Intervention~1 & $\mathcal{TN}_{[-1.5,\,1.5]}(0.7,\,0.4)$ \\
                       & Intervention~2 & $\mathcal{TN}_{[-1.5,\,1.5]}(-0.7,\,0.4)$ \\[2pt]
    \midrule
    \multirow{3}{*}{4} & Observational  & $\mathcal{TN}_{[-1.5,\,1.5]}(0,\,0.4)$ \\[2pt]
                       & Intervention~1 & $\mathcal{TN}_{[-1.5,\,1.5]}(0.9,\,0.4)$ \\
                       & Intervention~2 & $\mathcal{TN}_{[-1.5,\,1.5]}(-0.9,\,0.4)$ \\[2pt]
    \midrule
    \multirow{3}{*}{5} & Observational  & Not a root node \\[2pt]
                       & Intervention~1 & $\mathcal{TN}_{[-1.5,\,1.5]}(0.9,\,0.4)$ \\
                       & Intervention~2 & $\mathcal{TN}_{[-1.5,\,1.5]}(-0.9,\,0.4)$ \\[2pt]
    \midrule
    \multirow{3}{*}{6} & Observational  & Not a root node  \\[2pt]
                       & Intervention~1 & $\mathcal{TN}_{[0,\,3]}(2.4,\,0.4)$ \\
                       & Intervention~2 & $\mathcal{TN}_{[0,\,3]}(0.7,\,0.4)$ \\
    \bottomrule
    \label{causal_dataset_distributions}
  \end{tabular}
\end{table}

The observational data is generated from a linear structural equation model (SEM) depicted in \Cref{fig:causal model}. Specifically, we first sample the root node joints $J_1,J_2,J_4$ from the relevant distributions in \Cref{tab:noise_models}. Then, values of $\{J_1, \dots, J_6\}$ are determined by linear functions of their ancestors using the weights provided in \Cref{fig:causal model} plus noise $(\epsilon_1,\ldots,\epsilon_6)$ sampled from $\mathcal{TN}_{[0,\,1]}(0.0,\,0.1)$ distribution.
\begin{alignat*}{4}
  & J_1 = J_1 + \epsilon_1 & \qquad & J_3 = 0.88 J_2 + \epsilon_3 \\
  & J_2 = J_2 + \epsilon_2 & \qquad & J_5 = 0.26 J_3 + \epsilon_5 \\
  & J_4 = J_4 + \epsilon_4 & \qquad & J_6 = 0.24 J_1 + 0.31 J_2 + 0.37 J_3 + 0.15 J_5 + \epsilon_6
\end{alignat*}
An intervention on a variable $J_k$ is denoted by setting it to a new value $J'_k$ sampled from the interventional distribution of choice in \Cref{tab:noise_models} --- this works because we consider stochastic hard interventions that sets the interventional values directly equal to the exogenous noise variables. Next, we provide the equations that describe the system's behavior for interventions on each joint.

\paragraph{Intervention on joint 1}
Under $do(J_1 \coloneqq J'_1)$: The equation for $J_6$ depends on the new value $J'_1$ and others being unaffected and $J_1$ takes the value $J'_1$.
\begin{align*}
    J_6 = 0.24 J'_1 + 0.31 J_2 + 0.37 J_3 + 0.15 J_5
\end{align*}

\paragraph{Intervention on joint 2}
 Under $do(J_2 \coloneqq J'_2)$: The equations for the descendants of $J_2$ are updated as given below. 
\begin{align*}
    J_3 &= 0.88 J'_2 + \epsilon_3 \\
    J_5 &= 0.26 J_3 + \epsilon_5 \\
    J_6 &= 0.24 J_1 + 0.31 J'_2 + 0.37 J_3 + 0.15 J_5 + \epsilon_6
\end{align*}

\paragraph{Intervention on joint 3}
Under $do(J_3 \coloneqq J'_3)$: The descendants of $J_3$ are affected.
\begin{align*}
    J_5 &= 0.26 J'_3 + \epsilon_5 \\
    J_6 &= 0.24 J_1 + 0.31 J_2 + 0.37 J'_3 + 0.15 J_5 + \epsilon_6
\end{align*}

\paragraph{Intervention on joint 4}
Under $do(J_4 \coloneqq J'_4)$: No downstream variables are affected by an intervention on $J_4$, suggesting it is a root node with no children in this system.

\paragraph{Intervention on joint 5}
Under $do(J_5 \coloneqq J'_5)$: The equation for $J_6$ is updated.
\begin{align*}
    J_6 = 0.24 J_1 + 0.31 J_2 + 0.37 J_3 + 0.15 J'_5 + \epsilon_6
\end{align*}

\paragraph{Intervention on joint 6}
Under $do(J_6 \coloneqq J'_6)$: No changes occur in any other variables except $J_6$, as $J_6$ is a sink node (it does not cause other variables in the system).

\subsection{Training Details} \label{appendix:training_details}

All models were trained on TPUs. We performed a hyperparameter search for the optimal learning rate, testing values in the range of $10^{-7}$ to $10^{-3}$. Our dataset consists of 10,000 observational images. We generated corresponding interventional images for two experimental setups: In the single-camera setup, each observational image yields 6 interventional images. In the two-camera setup, each observational image yields 24 interventional images. The training process involved three stages. First, Autoencoder-1 was trained on 70k images (single-camera) and 260k images (two-camera) with a batch size of 256. Second, we trained a separate binary classifier network (i.e., the log-ratio density estimator) for each joint, using 20k samples (single-camera) and 40k samples (two-camera) with a batch size of 64. Finally, for each joint, we trained a corresponding Autoencoder-2 alongside its specific LDR. Autoencoder-2 is trained on the same number of samples as Autoencoder-1 for both single camera and two camera setup. 

In the main set of results, we use Adam optimizer with learning rates: AE1 with $1 \times 10^{-4}$, LDR with $1 \times 10^{-3}$, and AE2 with $7 \times 10^{-5}$.

\subsection{Additional Quantitative Results} \label{appendix:regime_a_results}

\begin{table}[t]
\centering
\caption{MCC and MSE in assumption-matched regime. MSE is reported in radians squared.}
\label{results_crl_main_3dp_abbrev}
\begin{threeparttable}
\footnotesize 
\setlength{\tabcolsep}{2.3pt} 

\begin{tabular}{@{} 
    >{\raggedright}m{2.2cm} 
    S[table-format=1.3] S[table-format=1.3] @{\hskip 1.1em} 
    S[table-format=1.3] S[table-format=1.3] @{\hskip 1.1em} 
    S[table-format=1.3] S[table-format=1.3] @{\hskip 1.1em} 
    S[table-format=1.3] S[table-format=1.3] @{\hskip 1.1em} 
    S[table-format=1.3] S[table-format=1.3] @{\hskip 1.1em} 
    S[table-format=1.3] S[table-format=1.3] @{}}
\toprule
& \multicolumn{2}{c}{\textbf{Joint 1}} & \multicolumn{2}{c}{\textbf{Joint 2}} & \multicolumn{2}{c}{\textbf{Joint 3}} & \multicolumn{2}{c}{\textbf{Joint 4}} & \multicolumn{2}{c}{\textbf{Joint 5}} & \multicolumn{2}{c}{\textbf{Joint 6}} \\
\midrule
\textbf{Model Setup} & {\textbf{MCC}} & {\textbf{MSE}} & {\textbf{MCC}} & {\textbf{MSE}} & {\textbf{MCC}} & {\textbf{MSE}} & {\textbf{MCC}} & {\textbf{MSE}} & {\textbf{MCC}} & {\textbf{MSE}} & {\textbf{MCC}} & {\textbf{MSE}} \\
\midrule
1C, indep. & {--} & {--} & 0.949 & 0.053 & {--} & {--} & 0.975 & 0.029 & {--} & {--} & 0.957 & 0.049 \\
\midrule
2C, indep. & 0.874 & 0.083 & 0.979 & 0.015 & 0.634 & 0.217 & 0.950 & 0.035 & 0.679 & 0.198 & 0.884 & 0.080 \\
\midrule
2C, causal & 0.921 & 0.058 & 0.966 & 0.020 & 0.788 & 0.106 & 0.976 & 0.019 & 0.742 & 0.051 & 0.756 & 0.070 \\
\midrule
2C, indep., occl. & 0.844 & 0.101 & 0.964 & 0.025 & 0.568 & 0.245 & 0.884 & 0.082 & 0.617 & 0.225 & 0.768 & 0.145 \\
\bottomrule
\end{tabular}
\begin{tablenotes}
\item 1C = one camera; 2C = two cameras; indep. and causal = joints distribution; occl. = occlusion(32x32).
\end{tablenotes}
\end{threeparttable}
\end{table}

\subsubsection{Single-Camera: Independent Joints}\label{sec:single-camera-independent}

We begin by evaluating the model in a single-camera setting, where joint angles are sampled \emph{independently} according to the distributions specified in \Cref{single_camera_distributions}. As this setup restricts visibility, we evaluate performance only on the three in-plane joints (2, 4, and 6) discussed in~\Cref{Single_camera_text}. \Cref{results_crl_main_3dp_abbrev} (first row) shows that MCC is consistently at least $0.94$. Remarkably, this disentanglement performance is stronger than that of the prior experiments in CRL studies in much simpler image datasets~(e.g., see \citep[Table 14]{varici2025score}). A qualitative analysis of the model's reconstruction performance is provided in \Cref{fig:scatter_plot_singleAngle}, presenting the scatter plots of the learned latent variables against the ground-truth angles. 

\begin{figure}[t]
    \centering
    

    \hfill
    \begin{subfigure}[b]{0.32\textwidth}
        \centering
        \includegraphics[width=\textwidth]{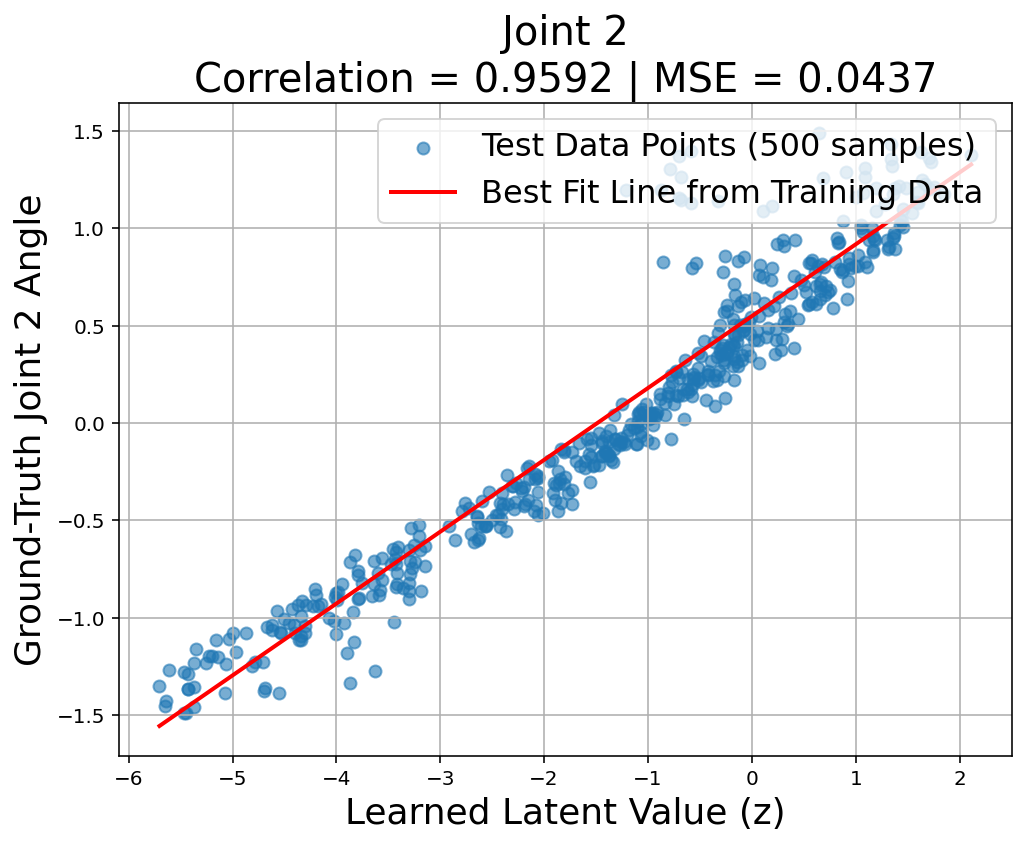}
        \caption{Joint 2}
        \label{fig:single_angle_mcc_plots:mcc_joint2}
    \end{subfigure}%
    \hfill
    \begin{subfigure}[b]{0.32\textwidth}
        \centering
        \includegraphics[width=\textwidth]{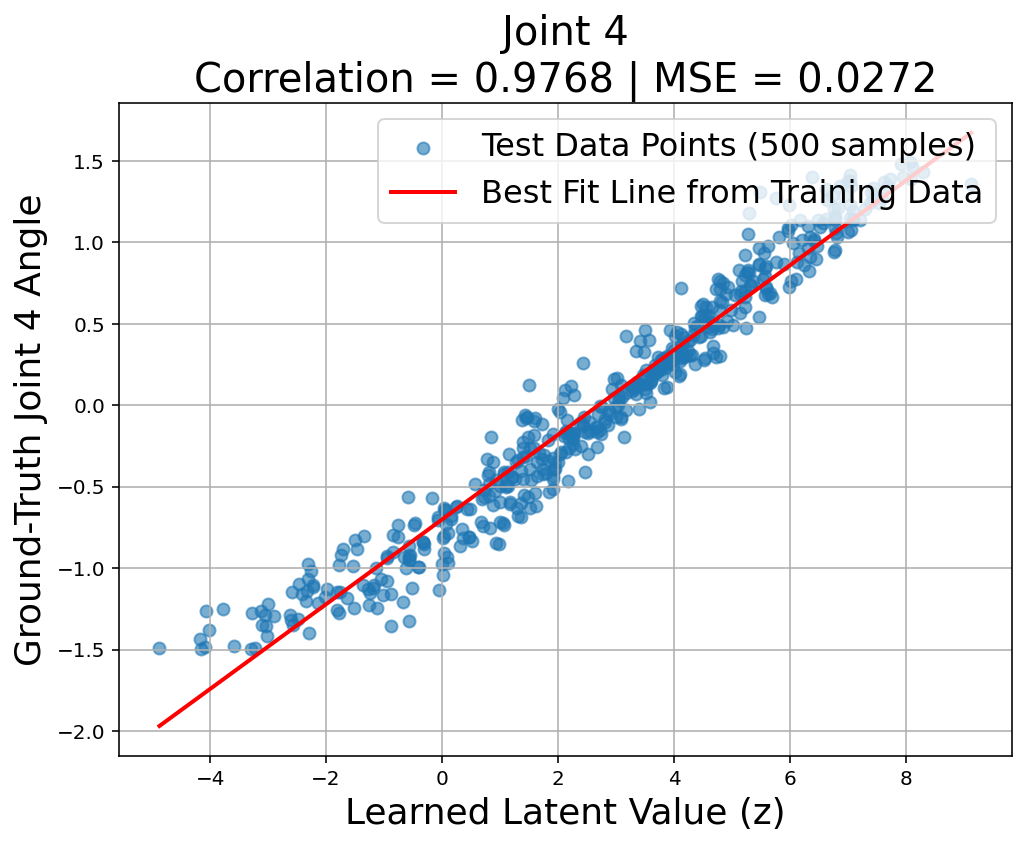}
        \caption{Joint 4}
        \label{fig:single_angle_mcc_plots:mcc_joint4}
    \end{subfigure}%
    \hfill
    \begin{subfigure}[b]{0.32\textwidth}
        \centering
        \includegraphics[width=\textwidth]{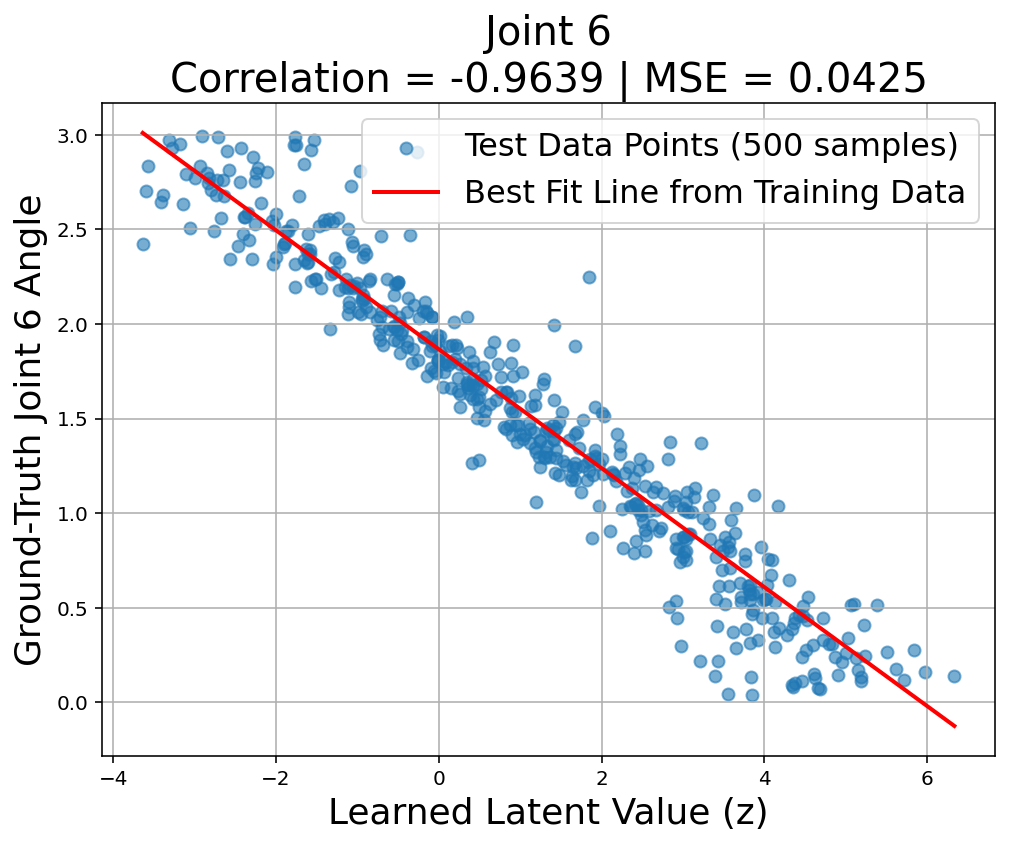}
        \caption{Joint 6}
        \label{fig:single_angle_mcc_plots:mcc_joint6}
    \end{subfigure}%
    \hfill\ %
    
    \caption{Single Camera: Scatter-plots of ground-truth vs. estimated angles for joints 2, 4, and 6}
    \label{fig:scatter_plot_singleAngle}
\end{figure}


\subsubsection{Extending to Multiple Views: Two Cameras with Independent Joints}
We extend our analysis to a more complex two-camera setup, which provides the multi-view information necessary to disentangle all six robot joints. The joint angles are again sampled \emph{independently} (see \Cref{two_camera_distributions}). As shown in \Cref{results_crl_main_3dp_abbrev} (second row) and the scatter plots in \Cref{fig:scatter_plot_twoAngle_independent}, the model successfully disentangles all six joints. However, we observe a performance drop for joints 3 and 5 compared to the others. We attribute this performance discrepancy to less precise score estimates from the LDR network, an interpretation supported by the fact that these two joints exhibited a relatively larger classification loss during LDR training. This suggests that the images of the hard interventions for joints 3 and 5 are less distinct, making them inherently more challenging to classify. Notably, the MCC scores for joints 2, 4, and 6 remain robust when transitioning from the single-camera to the two-camera setup. 
The stability of these scores suggests that the score-based method scales effectively, leveraging the multi-view information. Reconstructed images from the AE1 and AE2 are shown in the \Cref{fig:recon_comparison_2_cameraAngles}.

\subsubsection{Two-Camera with Causal Model over Joints}
To assess whether our method can learn \emph{causal} latents, which is the core premise of CRL, we generate a new dataset where joint angles are sampled according to a linear causal model described in \Cref{appendix causal dataset generation}. The structure and edge weights were chosen randomly, with joints~1, 2, and 4 set as root nodes. We report the results in \Cref{results_crl_main_3dp_abbrev} (third row). We observe that introducing a causal data generation process generally improves performance, with MSE values decreasing across most joints compared to the independent two-camera model. Overall, this experiment demonstrates the flexibility of the score-based approach to accommodate causal interactions among the joints.

\subsubsection{Scatter Plots} 
To provide a qualitative assessment of our models' disentanglement capabilities, we visualize the relationship between the learned latent variables and their corresponding ground-truth joint angles. \Cref{fig:scatter_plot_twoAngle_independent,fig:scatter_plot_twoAngle_causal} present these scatter plots for the two-camera Independent and Causal models, respectively. Each plot is generated from the single best trial (out of 15) for that specific joint, determined by the highest MCC score. This visualization of the best-case performance complements the aggregated statistics reported in our main results tables.

\begin{figure}[ht!]
    \centering 
    
    \begin{subfigure}[b]{0.32\textwidth}
        \centering
        \includegraphics[width=\textwidth]{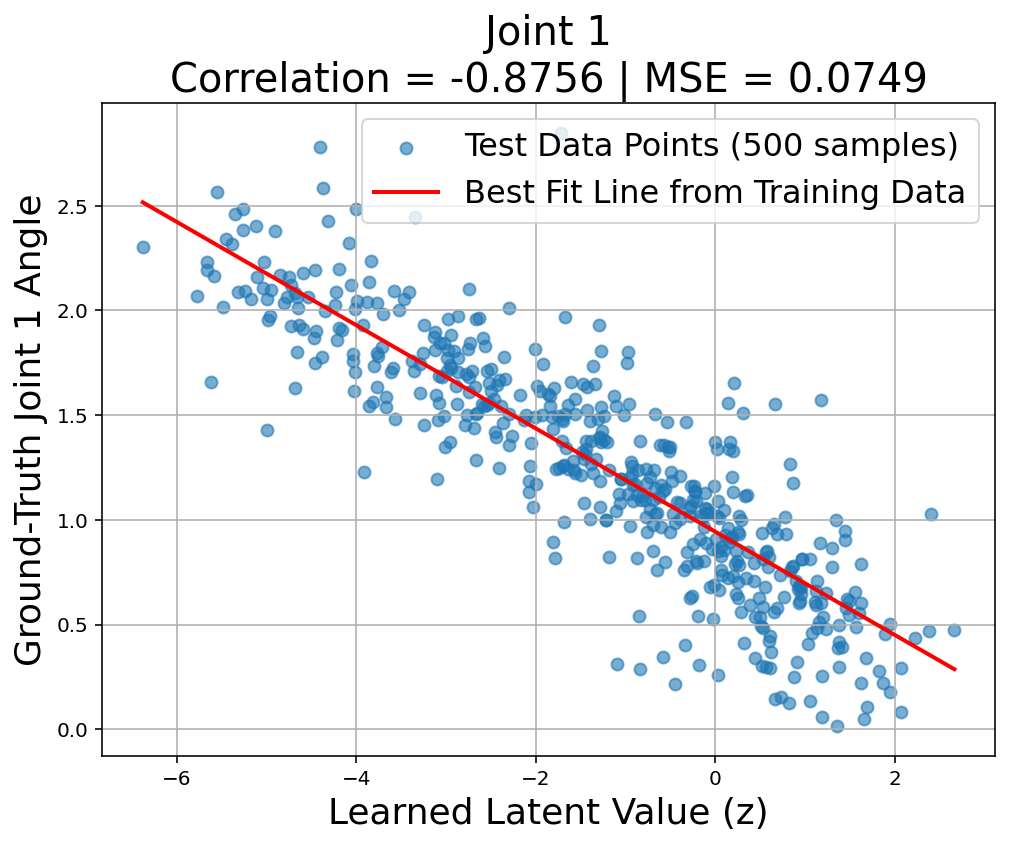}
        \caption{Joint 1}
        \label{fig:mcc_plots_2_cameraAngles:mcc_joint1}
    \end{subfigure}
    \hfill 
    \begin{subfigure}[b]{0.32\textwidth}
        \centering
        \includegraphics[width=\textwidth]{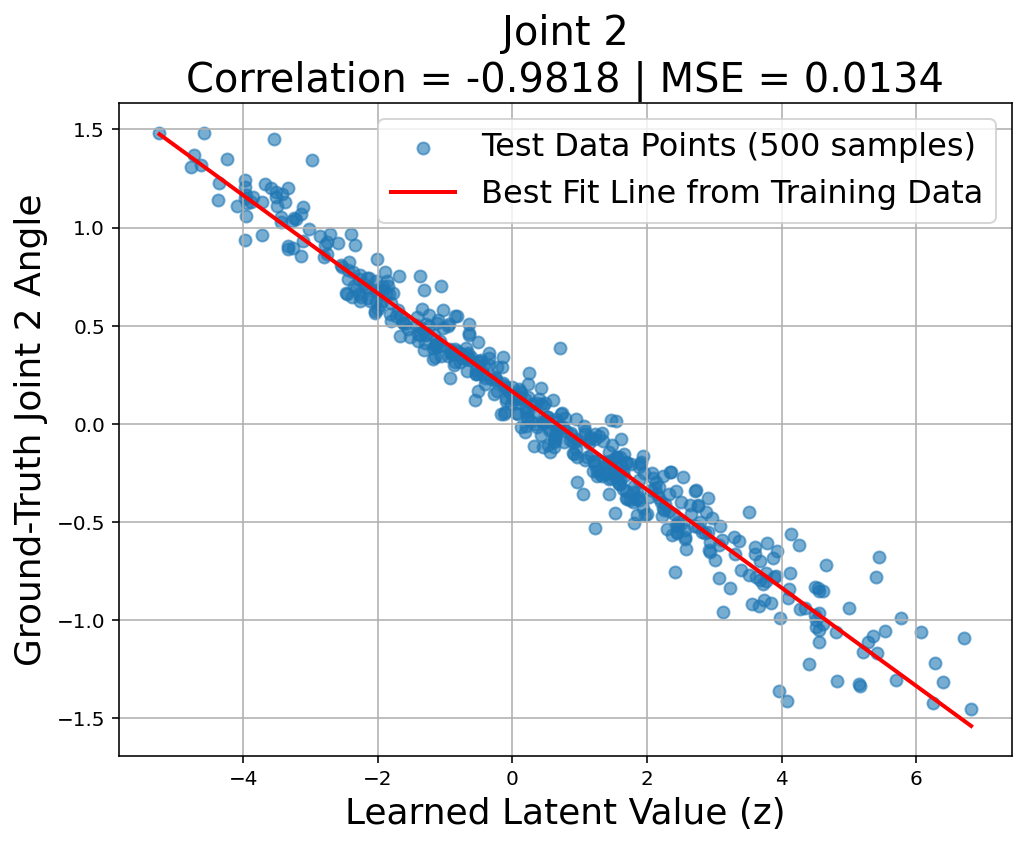}
        \caption{Joint 2}
        \label{fig:mcc_plots_2_cameraAngles:mcc_joint2}
    \end{subfigure}
    \hfill 
    \begin{subfigure}[b]{0.32\textwidth}
        \centering
        \includegraphics[width=\textwidth]{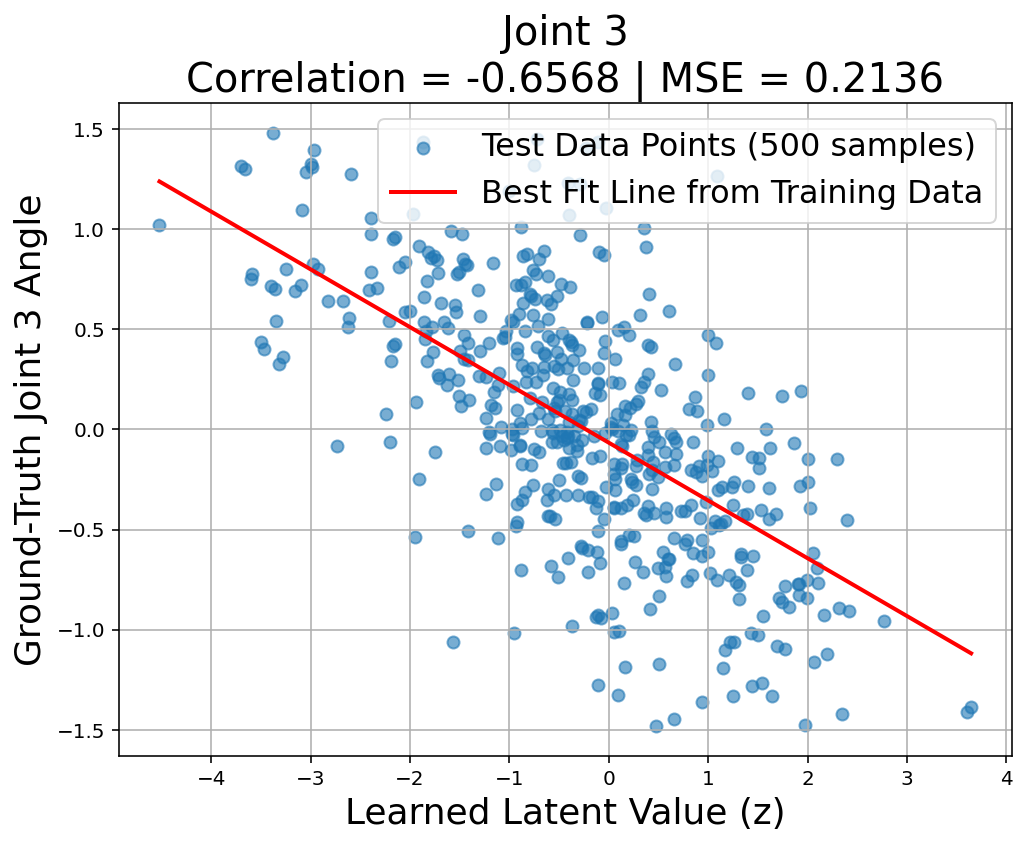}
        \caption{Joint 3}
        \label{fig:mcc_plots_2_cameraAngles:mcc_joint3}
    \end{subfigure}
    

    \begin{subfigure}[b]{0.3\textwidth}
        \centering
        \includegraphics[width=\textwidth]{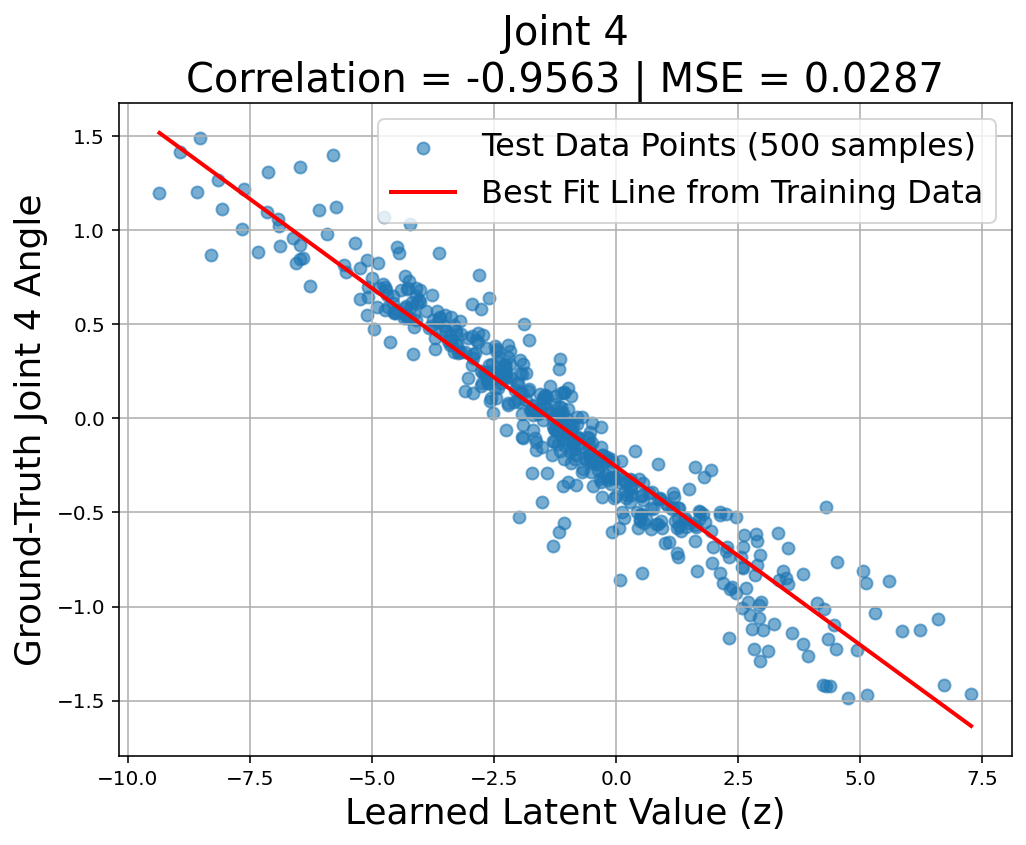}
        \caption{Joint 4}
        \label{fig:mcc_plots_2_cameraAngles:mcc_joint4}
    \end{subfigure}
    \hfill 
    \begin{subfigure}[b]{0.3\textwidth}
        \centering
        \includegraphics[width=\textwidth]{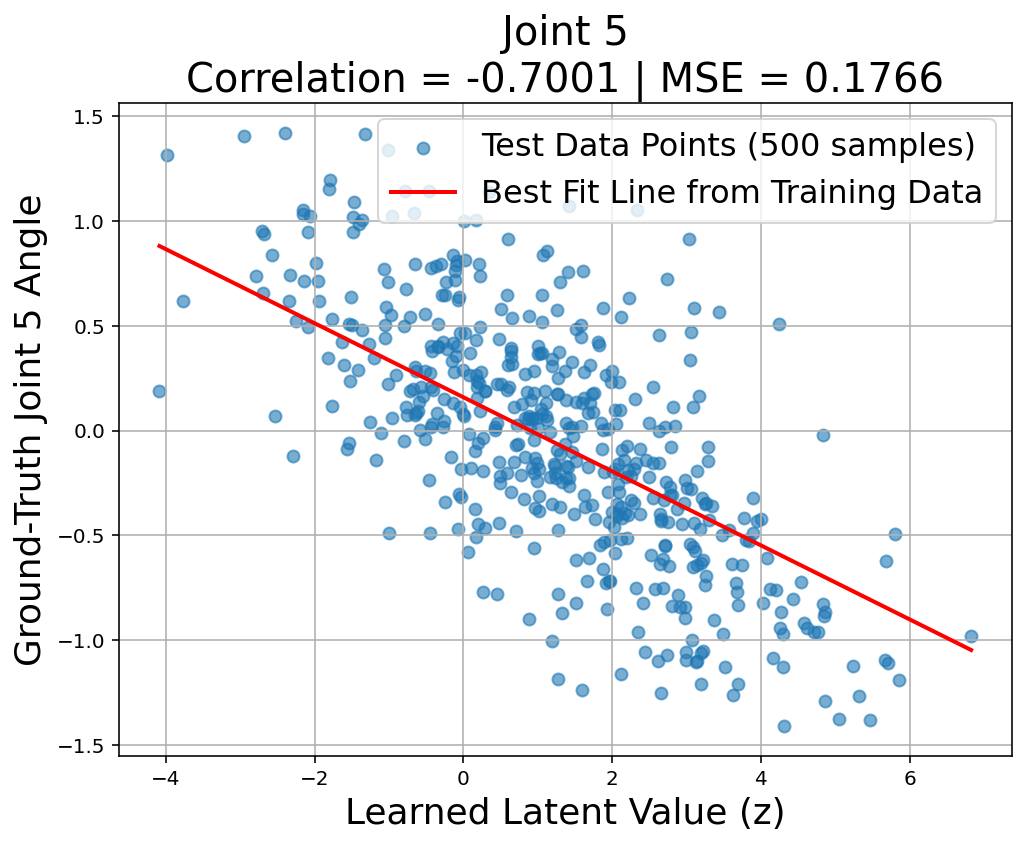}
        \caption{Joint 5}
        \label{fig:mcc_plots_2_cameraAngles:mcc_joint5}
    \end{subfigure}
    \hfill 
    \begin{subfigure}[b]{0.3\textwidth}
        \centering
        \includegraphics[width=\textwidth]{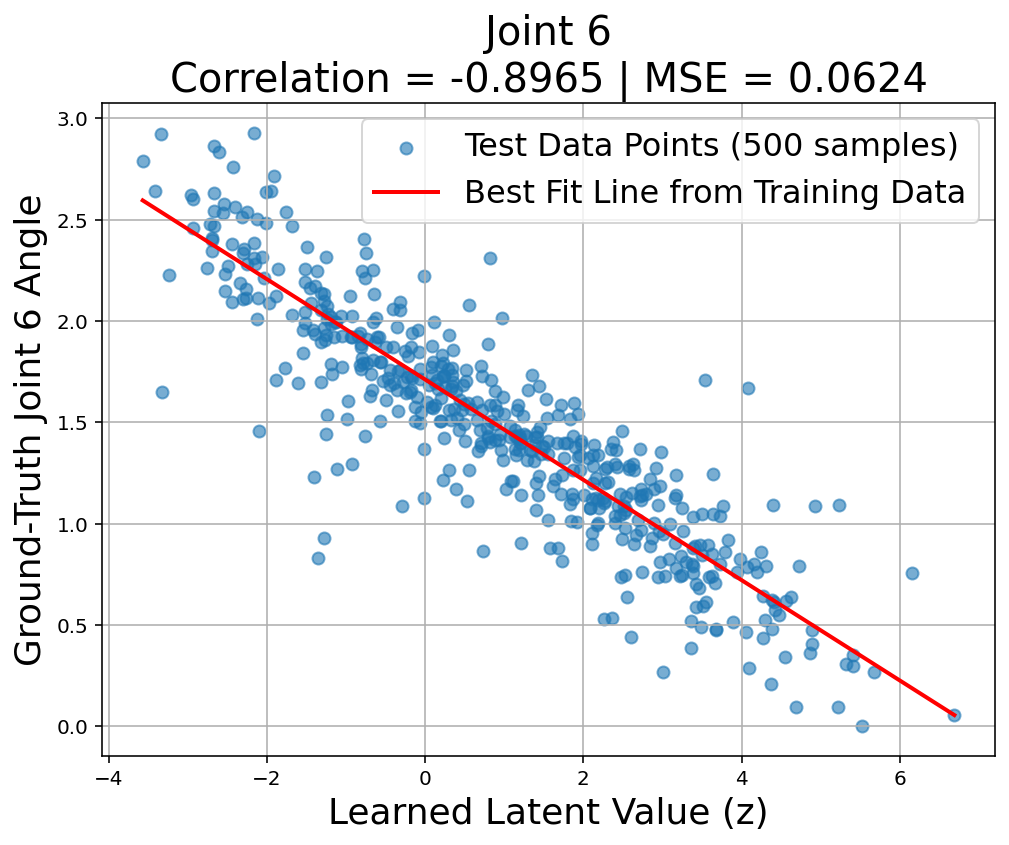}
        \caption{Joint 6}
        \label{fig:mcc_plots_2_cameraAngles:mcc_joint6}
    \end{subfigure}

    \caption{Scatter plots evaluating the two-camera Independent model on the in-distribution (ID) test set (\Cref{two_camera_distributions}) . Each plot visualizes the relationship between a learned latent variable and its corresponding ground-truth joint angle. The displayed Correlation and MSE values correspond to the single best trial out of 15 runs, while the results presented in \Cref{results_crl_main_3dp_abbrev,tab:transposed_results_rounded} correspond to the mean statistics.}
    \label{fig:scatter_plot_twoAngle_independent}
\end{figure}

\begin{figure}[ht!]
    \centering 
    
    \begin{subfigure}[b]{0.32\textwidth}
        \centering
        \includegraphics[width=\textwidth]{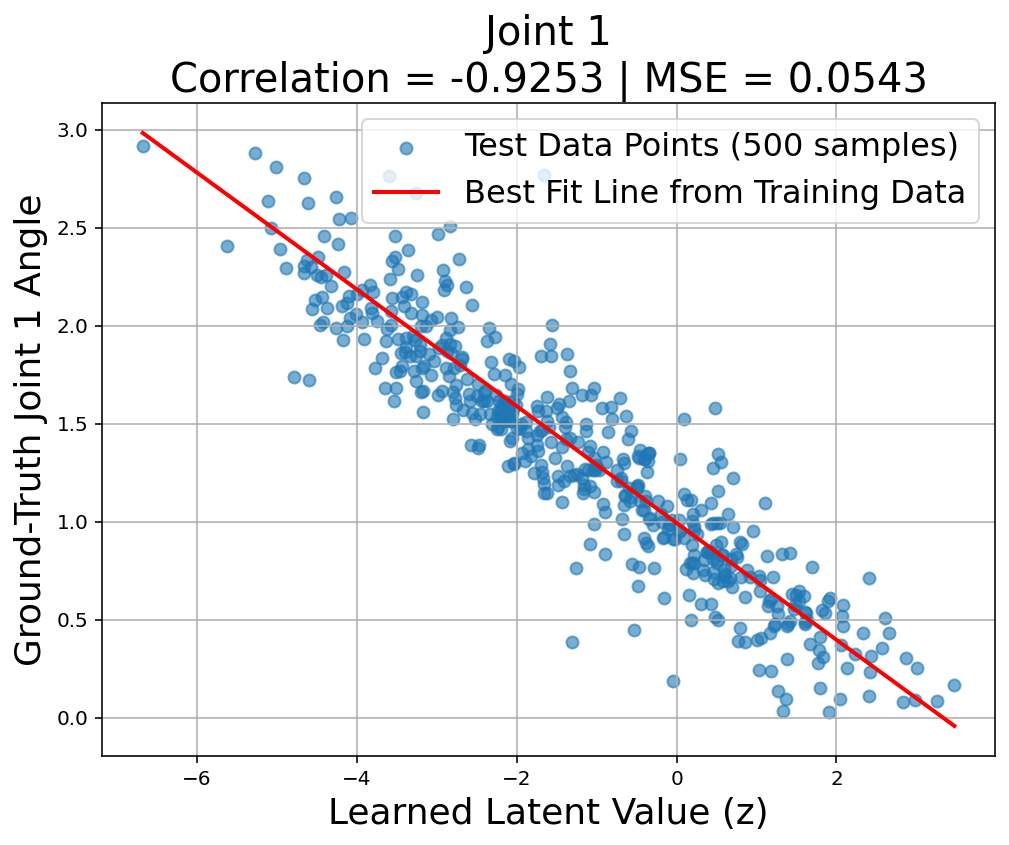}
        \caption{Joint 1}
        \label{fig:mcc_plots_2_cameraAngles:mcc_joint1_causal}
    \end{subfigure}
    \hfill 
    \begin{subfigure}[b]{0.32\textwidth}
        \centering
        \includegraphics[width=\textwidth]{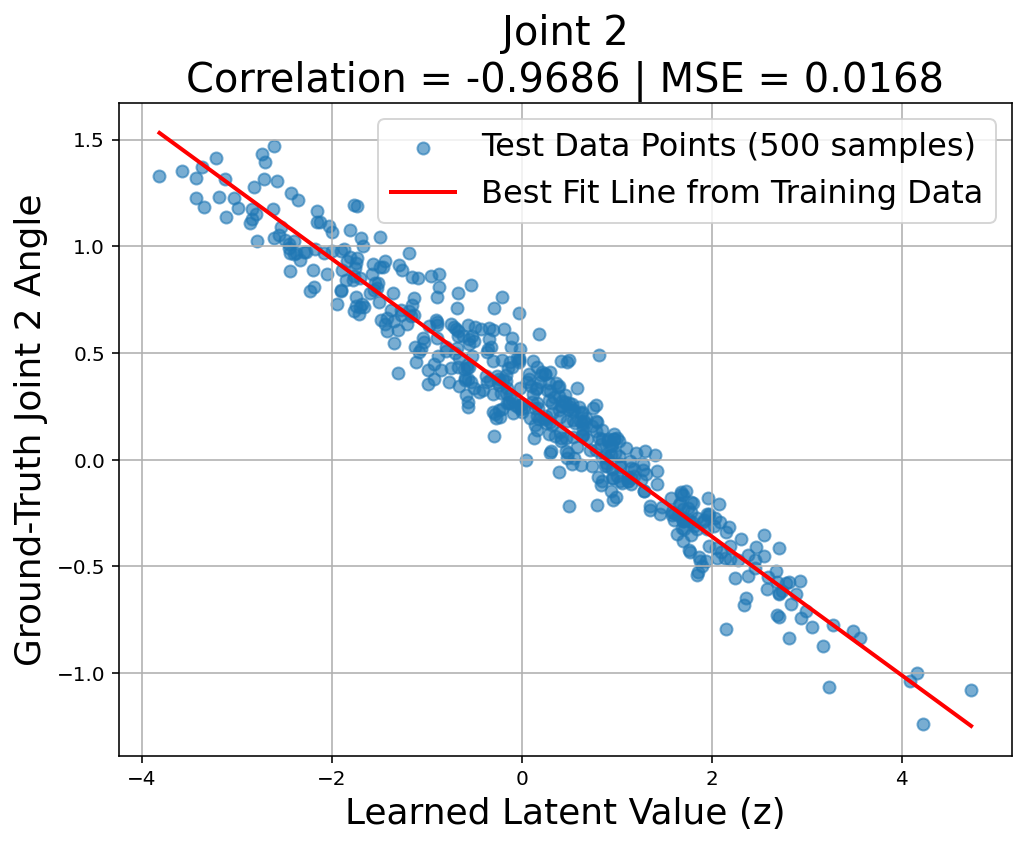}
        \caption{Joint 2}
        \label{fig:mcc_plots_2_cameraAngles:mcc_joint2_causal}
    \end{subfigure}
    \hfill 
    \begin{subfigure}[b]{0.32\textwidth}
        \centering
        \includegraphics[width=\textwidth]{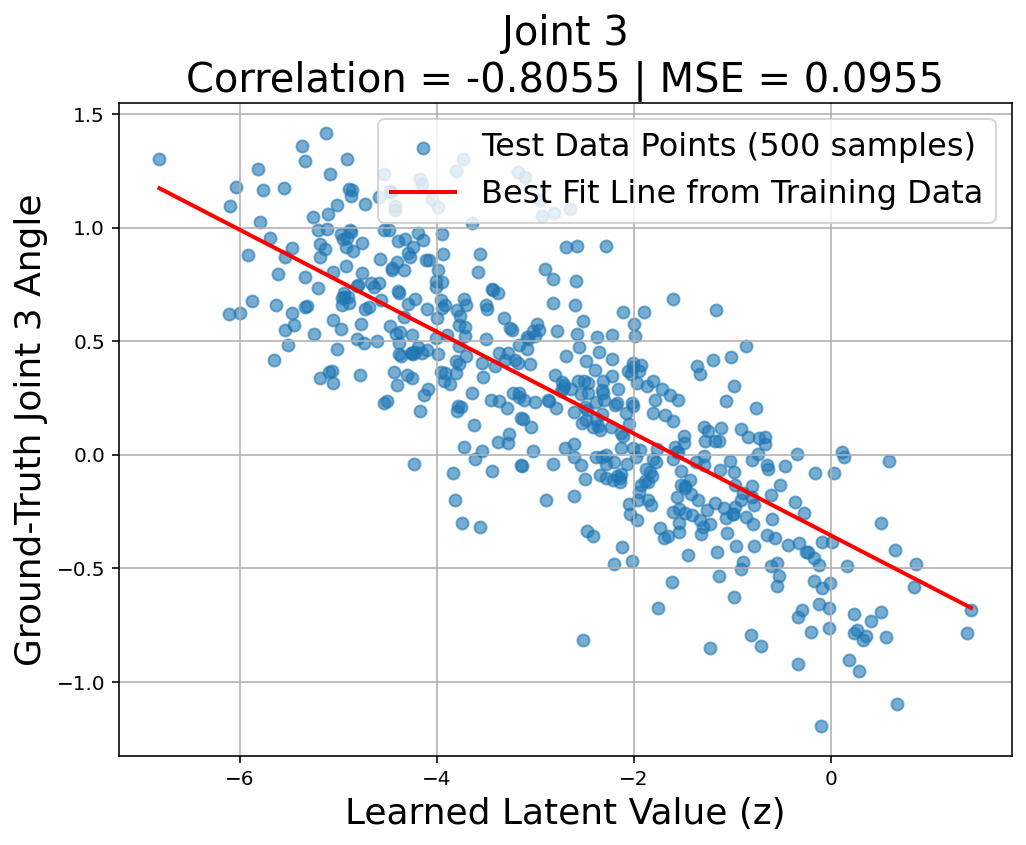}
        \caption{Joint 3}
        \label{fig:mcc_plots_2_cameraAngles:mcc_joint3_causal}
    \end{subfigure}
    

    \begin{subfigure}[b]{0.3\textwidth}
        \centering
        \includegraphics[width=\textwidth]{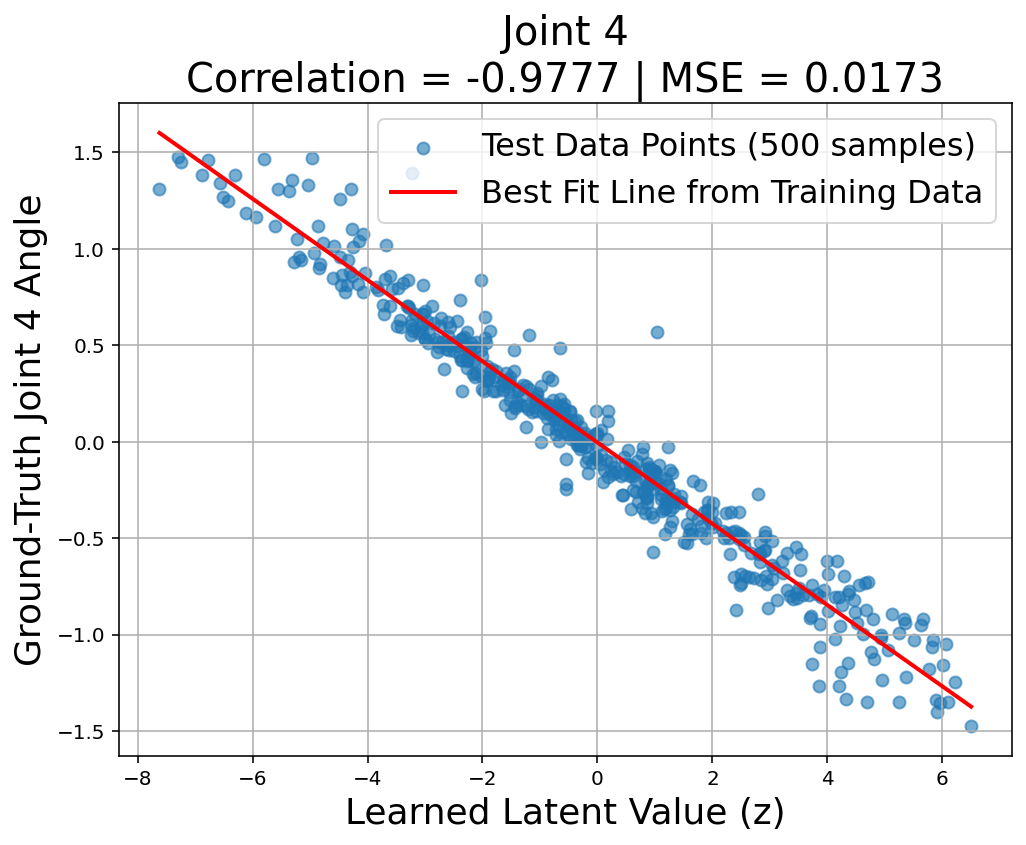}
        \caption{Joint 4}
        \label{fig:mcc_plots_2_cameraAngles:mcc_joint4_causal}
    \end{subfigure}
    \hfill 
    \begin{subfigure}[b]{0.3\textwidth}
        \centering
        \includegraphics[width=\textwidth]{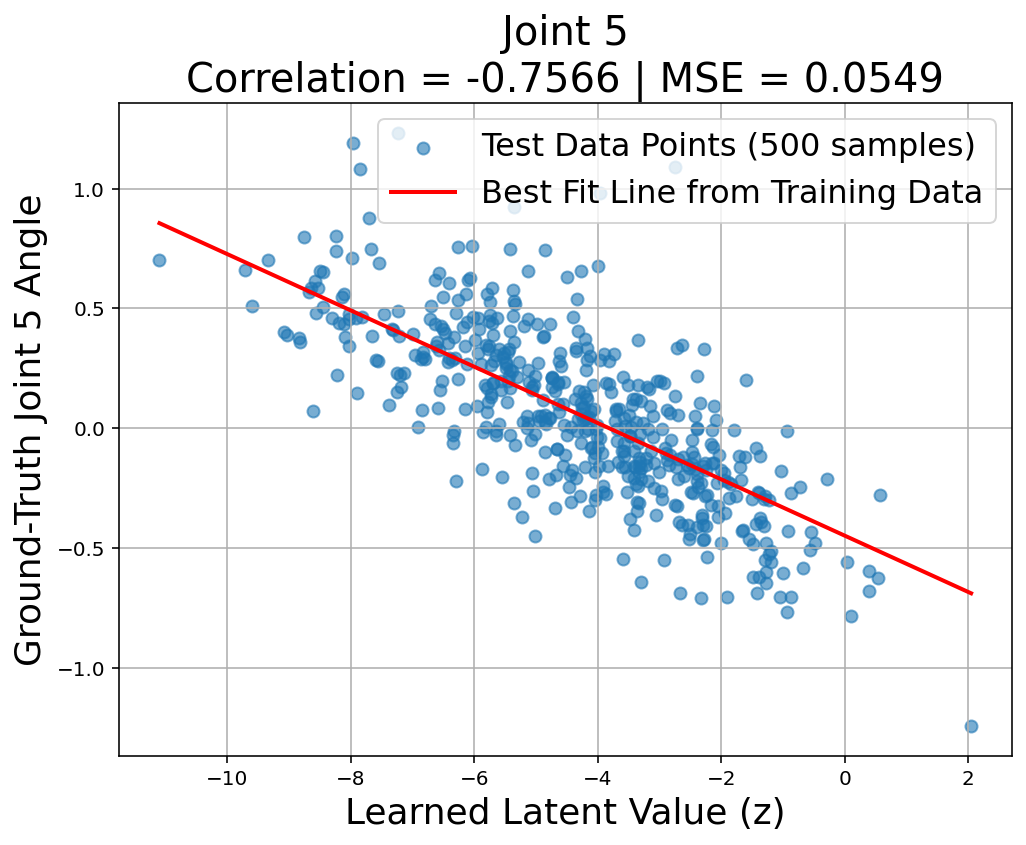}
        \caption{Joint 5}
        \label{fig:mcc_plots_2_cameraAngles:mcc_joint5_causal}
    \end{subfigure}
    \hfill 
    \begin{subfigure}[b]{0.3\textwidth}
        \centering
        \includegraphics[width=\textwidth]{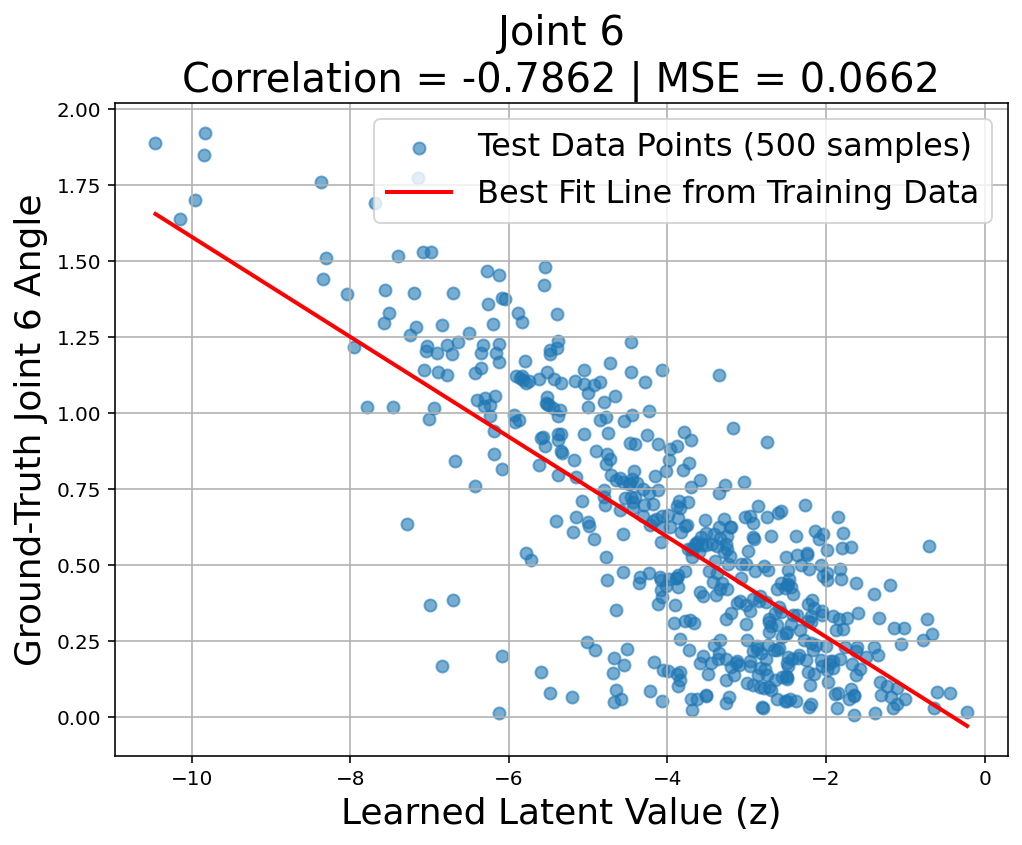}
        \caption{Joint 6}
        \label{fig:mcc_plots_2_cameraAngles:mcc_joint6_causal}
    \end{subfigure}

    \caption{Scatter plots evaluating the two-camera causal model on the causally-generated test set. Each plot visualizes the relationship between a learned latent variable and its corresponding ground-truth joint angle. The displayed Correlation and MSE values correspond to the single best trial out of 15 runs, while the results presented in \Cref{results_crl_main_3dp_abbrev,tab:transposed_results_rounded} correspond to the mean statistics.}
    \label{fig:scatter_plot_twoAngle_causal}
\end{figure}

\begin{table}[htbp]
\centering
\sisetup{table-format=1.3} 
\caption{Mean Squared Error (MSE) in radians squared ($\text{rad}^2$) for each joint under various experimental conditions for two camera angles. The table compares performance on in-distribution (ID) test set (\Cref{two_camera_distributions}) and out-of-distribution (OOD) test set (\Cref{tab:ood_distributions}) for independent, causal, and occluded inference models.}
\label{tab:mse_ablation_study}
\small 

\begin{tabular}{@{} ccccccccc @{}}
\toprule
\textbf{Experiment} & \textbf{Distribution} & \textbf{Patch Size} & {\textbf{Joint 1}} & {\textbf{Joint 2}} & {\textbf{Joint 3}} & {\textbf{Joint 4}} & {\textbf{Joint 5}} & {\textbf{Joint 6}} \\
\midrule
Independent  & ID     & {--} & 0.083 & 0.015 & 0.217 & 0.035 & 0.198 & 0.080 \\
Causal    & ID    & {--} & 0.058 & 0.020 & 0.106 & 0.019 & 0.051 & 0.070 \\
Occlusion & ID    & 16   & 0.089 & 0.018 & 0.225 & 0.044 & 0.200 & 0.082 \\
Occlusion & ID    & 32   & 0.101 & 0.025 & 0.245 & 0.082 & 0.225 & 0.145 \\
Occlusion & ID    & 64   & 0.186 & 0.077 & 0.322 & 0.258 & 0.298 & 0.273 \\
\midrule
Independent  & OOD      & {--} & 0.084 & 0.017 & 0.239 & 0.048 & 0.199 & 0.120 \\
Causal    & OOD    & {--} & 0.108 & 0.044 & 0.241 & 0.085 & 0.219 & 0.116 \\
Occlusion  & OOD   & 16   & 0.092 & 0.019 & 0.249 & 0.054 & 0.205 & 0.109 \\
Occlusion  & OOD   & 32   & 0.103 & 0.024 & 0.288 & 0.079 & 0.226 & 0.140 \\
Occlusion  & OOD   & 64   & 0.201 & 0.087 & 0.356 & 0.227 & 0.331 & 0.216 \\
\bottomrule
\end{tabular}
\end{table}

\begin{table}[htbp]
\centering
\caption{Comparison of MCC and MSE ($\text{rad}^2$) for each joint across two model settings with error bars. Mean and Std Dev are calculated across the 15 runs. Values are reported as Mean ± Std Dev.}
\label{tab:transposed_results_rounded}
\begin{threeparttable}
\footnotesize 
\setlength{\tabcolsep}{5pt} 
\begin{tabular}{@{} c c c @{\hskip 1em} c c @{}}
\toprule
& \multicolumn{2}{c}{\textbf{2C, indep.}} & \multicolumn{2}{c}{\textbf{2C, causal}} \\
\midrule
\textbf{Joint Angle} & {\textbf{MCC}} & {\textbf{MSE}} & {\textbf{MCC}} & {\textbf{MSE}} \\
\midrule
Joint 1 & $0.874 \pm 0.004$ & $0.083 \pm 0.006$ & $0.921 \pm 0.003$ & $0.058 \pm 0.003$ \\
Joint 2 & $0.979 \pm 0.001$ & $0.015 \pm 0.001$ & $0.966 \pm 0.002$ & $0.020 \pm 0.001$ \\
Joint 3 & $0.634 \pm 0.010$ & $0.217 \pm 0.011$ & $0.788 \pm 0.003$ & $0.106 \pm 0.005$ \\
Joint 4 & $0.950 \pm 0.002$ & $0.035 \pm 0.002$ & $0.976 \pm 0.001$ & $0.019 \pm 0.001$ \\
Joint 5 & $0.679 \pm 0.010$ & $0.198 \pm 0.011$ & $0.742 \pm 0.006$ & $0.051 \pm 0.003$ \\
Joint 6 & $0.884 \pm 0.005$ & $0.080 \pm 0.007$ & $0.756 \pm 0.010$ & $0.070 \pm 0.003$ \\
\bottomrule
\end{tabular}
\begin{tablenotes}[flushleft]
\item \footnotesize 2C = two cameras; indep. and causal refer to the joint distribution model.
\end{tablenotes}
\end{threeparttable}
\end{table}

\subsection{Supervised Comparison and Robustness}

\paragraph{RoboPEPP: Label Supervision Baseline.}
For our experiments, we also compare against a recent supervised baseline. Specifically, we use RoboPEPP~\citep{goswami2025robopepp}, which is among the few alternatives that can predict pose with unknown joint angles at test time (along with HPE~\citep{ban2024real} and RoboPOSE~\citep{labbe2021single}), and reports stronger results than the alternatives. RoboPEPP involves a two-stage training process. First, a self-supervised I-JEPA backbone is pre-trained on the entire independent ID training dataset in \Cref{tab:twoCamera_trunc}. In the second stage, a ``JointNet'' is trained to predict six joint angles. Its input is formed by concatenating the I-JEPA embeddings from two camera angles, and it is optimized using an $L_2$ loss with the ground-truth joint angle labels. We observe that the MSE of the RoboPEPP baseline decreases with increasing size of labeled data, with our model achieving a comparable performance on joints 1, 2, 4, and 6 to RoboPEPP when using labels of 5\% (approx. 6.5K labels and 2 camera orientation per pose = 13K samples) of the entire dataset for training.

We pre-train a ViT-Huge model using the I-JEPA objective on 128x128x3 images. The core masking strategy involves one large context block (85-100\% scale) and four smaller, non-overlapping target blocks (15-20\% scale). The model is trained for 10 epochs with a batch size of 32, using a cosine learning rate schedule with a peak learning rate of $1e-3$. The target encoder's weights are updated with an exponential moving average (EMA). Notably, we only use random resized cropping for data augmentation, disabling strong augmentations like color jitter and horizontal flipping. Training is accelerated with bfloat16 mixed-precision.Our training is a two-stage process. In the first stage, we pre-train a ViT-Huge backbone using the I-JEPA self-supervised objective on the full dataset, comprising 260k images, as described previously. For the second stage, we fine-tune our JointNet model for a robotic perception task on the Panda robot. This supervised training uses the smaller RoboPEPP dataset (1\%, 5\%, 10\%, 100\%) with a batch size of 8. The model is trained for (100, 75, 25, 15) epochs respectively using a one-cycle learning rate schedule with a maximum learning rate of $1e-4$ and a weight decay of $1e-5$. \Cref{tab:mse_results_resized,tab:mse_ablation_study,tab:transposed_results_rounded,tab:mse_results_id,tab:mse_results_ood} show a detailed comparison of the our and RoboPEPP performance.

\begin{table}[htbp]
\centering
\label{robopepp_appendix}
\caption{Per-joint Mean Squared Error (MSE) for the RoboPEPP model in radians squared ($\text{rad}^2$). The table presents an ablation study on the effect of varying training data labels, evaluated on both in-distribution (ID) test set (\Cref{two_camera_distributions}) and out-of-distribution (OOD) test set (\Cref{tab:ood_distributions}).}
\label{tab:mse_results_resized}

\resizebox{\textwidth}{!}{%
\begin{tabular}{@{}cccc cccccc@{}} 
\toprule
\textbf{Dataset} & \textbf{Distribution} & \textbf{Epochs} & \textbf{Patch Size} & \textbf{Joint 1} & \textbf{Joint 2} & \textbf{Joint 3} & \textbf{Joint 4} & \textbf{Joint 5} & \textbf{Joint 6} \\
\midrule
2.6k & ID & 100 & - & 0.136 & 0.039 & 0.237 & 0.053 & 0.253 & 0.200 \\
2.6k & ID & 100 & 16 & 0.166 & 0.030 & 0.287 & 0.072 & 0.309 & 0.218 \\
2.6k & ID & 100 & 32 & 0.186 & 0.060 & 0.305 & 0.125 & 0.331 & 0.263 \\
2.6k & ID & 100 & 64 & 0.412 & 0.358 & 0.462 & 1.328 & 0.515 & 0.573 \\
\midrule
2.6k & OOD & 100 & - & 0.202 & 0.054 & 0.374 & 0.101 & 0.378 & 0.454 \\
2.6k & OOD & 100 & 16 & 0.220 & 0.058 & 0.413 & 0.117 & 0.412 & 0.432 \\
2.6k & OOD & 100 & 32 & 0.235 & 0.085 & 0.480 & 0.164 & 0.346 & 0.461 \\
2.6k & OOD & 100 & 64 & 0.496 & 0.334 & 0.538 & 1.156 & 0.529 & 0.570 \\
\midrule \midrule
13k & ID & 75 & - & 0.075 & 0.017 & 0.134 & 0.024 & 0.153 & 0.081 \\
13k & ID & 75 & 16 & 0.081 & 0.021 & 0.155 & 0.040 & 0.189 & 0.111 \\
13k & ID & 75 & 32 & 0.140 & 0.050 & 0.304 & 0.198 & 0.299 & 0.232 \\
13k & ID & 75 & 64 & 0.446 & 0.474 & 0.648 & 2.088 & 0.976 & 0.501 \\
\midrule
13k & OOD & 75 & - & 0.083 & 0.023 & 0.202 & 0.063 & 0.275 & 0.225 \\
13k & OOD & 75 & 16 & 0.127 & 0.040 & 0.304 & 0.088 & 0.321 & 0.388 \\
13k & OOD & 75 & 32 & 0.162 & 0.058 & 0.445 & 0.299 & 0.582 & 0.718 \\
13k & OOD & 75 & 64 & 0.489 & 0.412 & 0.489 & 1.926 & 1.348 & 1.255 \\
\midrule \midrule
26k & ID & 25 & - & 0.030 & 0.010 & 0.072 & 0.022 & 0.091 & 0.063 \\
26k & ID & 25 & 16 & 0.034 & 0.011 & 0.071 & 0.037 & 0.093 & 0.073 \\
26k & ID & 25 & 32 & 0.077 & 0.036 & 0.194 & 0.180 & 0.181 & 0.136 \\
26k & ID & 25 & 64 & 0.375 & 0.380 & 0.529 & 2.808 & 0.524 & 1.474 \\
\midrule
26k & OOD & 25 & - & 0.064 & 0.020 & 0.125 & 0.037 & 0.173 & 0.152 \\
26k & OOD & 25 & 16 & 0.050 & 0.019 & 0.126 & 0.055 & 0.197 & 0.146 \\
26k & OOD & 25 & 32 & 0.101 & 0.066 & 0.230 & 0.262 & 0.308 & 0.202 \\
26k & OOD & 25 & 64 & 0.394 & 0.297 & 0.497 & 2.475 & 0.582 & 1.028 \\
\midrule \midrule
260k & ID & 15 & - & 0.003 & 0.001 & 0.007 & 0.003 & 0.010 & 0.011 \\
260k & ID & 15 & 16 & 0.016 & 0.004 & 0.026 & 0.007 & 0.030 & 0.011 \\
260k & ID & 15 & 32 & 0.066 & 0.036 & 0.097 & 0.053 & 0.090 & 0.045 \\
260k & ID & 15 & 64 & 0.268 & 0.517 & 0.593 & 1.891 & 0.478 & 0.430 \\
\midrule
260k & OOD & 15 & - & 0.007 & 0.003 & 0.013 & 0.004 & 0.018 & 0.019 \\
260k & OOD & 15 & 16 & 0.029 & 0.008 & 0.059 & 0.014 & 0.081 & 0.028 \\
260k & OOD & 15 & 32 & 0.130 & 0.060 & 0.210 & 0.088 & 0.201 & 0.125 \\
260k & OOD & 15 & 64 & 0.262 & 0.447 & 0.649 & 1.437 & 0.465 & 0.616 \\
\bottomrule
\end{tabular}%
} 
\end{table}

\begin{table}[htbp]
\centering
\caption{Comparison of In-Distribution (ID) Mean Squared Error (MSE) in radians squared ($\text{rad}^2$) for the our score-based method and RoboPEPP models. Results are shown per joint under various experimental conditions.}
\label{tab:mse_results_id}
\small 
\sisetup{table-format=1.3} 
\begin{tabular}{@{} ll c *{6}{S} @{}}
\toprule
\textbf{Model} & \textbf{Experiment} & \textbf{Patch Size} & {\textbf{Joint 1}} & {\textbf{Joint 2}} & {\textbf{Joint 3}} & {\textbf{Joint 4}} & {\textbf{Joint 5}} & {\textbf{Joint 6}} \\
\midrule
\multirow{5}{*}{Our method} & Independent & {--} & 0.083 & 0.015 & 0.217 & 0.035 & 0.198 & 0.080 \\
& Causal      & {--} & 0.058 & 0.020 & 0.106 & 0.019 & 0.051 & 0.070 \\
& Occlusion   & 16   & 0.089 & 0.018 & 0.225 & 0.044 & 0.200 & 0.082 \\
& Occlusion   & 32   & 0.101 & 0.025 & 0.245 & 0.082 & 0.225 & 0.145 \\
& Occlusion   & 64   & 0.186 & 0.077 & 0.322 & 0.258 & 0.298 & 0.273 \\
\midrule
\multirow{16}{*}{RoboPEPP} & 2.6k Dataset & {--} & 0.136 & 0.039 & 0.237 & 0.053 & 0.253 & 0.200 \\
& 2.6k Dataset & 16   & 0.166 & 0.030 & 0.287 & 0.072 & 0.309 & 0.218 \\
& 2.6k Dataset & 32   & 0.186 & 0.060 & 0.305 & 0.125 & 0.331 & 0.263 \\
& 2.6k Dataset & 64   & 0.412 & 0.358 & 0.462 & 1.328 & 0.515 & 0.573 \\
\cmidrule{2-9}
& 13k Dataset  & {--} & 0.075 & 0.017 & 0.134 & 0.024 & 0.153 & 0.081 \\
& 13k Dataset  & 16   & 0.081 & 0.021 & 0.155 & 0.040 & 0.189 & 0.111 \\
& 13k Dataset  & 32   & 0.140 & 0.050 & 0.304 & 0.198 & 0.299 & 0.232 \\
& 13k Dataset  & 64   & 0.446 & 0.474 & 0.648 & 2.088 & 0.976 & 0.501 \\
\cmidrule{2-9}
& 26k Dataset  & {--} & 0.030 & 0.010 & 0.072 & 0.022 & 0.091 & 0.063 \\
& 26k Dataset  & 16   & 0.034 & 0.011 & 0.071 & 0.037 & 0.093 & 0.073 \\
& 26k Dataset  & 32   & 0.077 & 0.036 & 0.194 & 0.180 & 0.181 & 0.136 \\
& 26k Dataset  & 64   & 0.375 & 0.380 & 0.529 & 2.808 & 0.524 & 1.474 \\
\cmidrule{2-9}
& 260k Dataset & {--} & 0.003 & 0.001 & 0.007 & 0.003 & 0.010 & 0.011 \\
& 260k Dataset & 16   & 0.016 & 0.004 & 0.026 & 0.007 & 0.030 & 0.011 \\
& 260k Dataset & 32   & 0.066 & 0.036 & 0.097 & 0.053 & 0.090 & 0.045 \\
& 260k Dataset & 64   & 0.268 & 0.517 & 0.593 & 1.891 & 0.478 & 0.430 \\
\bottomrule
\end{tabular}
\end{table}

\begin{table}[htbp]
\centering
\caption{Comparison of Out-of-Distribution (OOD) Mean Squared Error (MSE) in radians squared ($\text{rad}^2$) for our method and RoboPEPP. Results are shown per joint under various experimental conditions.}
\label{tab:mse_results_ood}
\small 
\sisetup{table-format=1.3} 
\begin{tabular}{@{} ll c *{6}{S} @{}}
\toprule
\textbf{Model} & \textbf{Experiment} & \textbf{Patch Size} & {\textbf{Joint 1}} & {\textbf{Joint 2}} & {\textbf{Joint 3}} & {\textbf{Joint 4}} & {\textbf{Joint 5}} & {\textbf{Joint 6}} \\
\midrule
\multirow{5}{*}{Our method} & Independent & {--} & 0.084 & 0.017 & 0.239 & 0.048 & 0.199 & 0.120 \\
& Causal      & {--} & 0.108 & 0.044 & 0.241 & 0.085 & 0.219 & 0.116 \\
& Occlusion   & 16   & 0.092 & 0.019 & 0.249 & 0.054 & 0.205 & 0.109 \\
& Occlusion   & 32   & 0.103 & 0.024 & 0.288 & 0.079 & 0.226 & 0.140 \\
& Occlusion   & 64   & 0.201 & 0.087 & 0.356 & 0.227 & 0.331 & 0.216 \\
\midrule
\multirow{16}{*}{RoboPEPP} & 2.6k Dataset & {--} & 0.202 & 0.054 & 0.374 & 0.101 & 0.378 & 0.454 \\
& 2.6k Dataset & 16   & 0.220 & 0.058 & 0.413 & 0.117 & 0.412 & 0.432 \\
& 2.6k Dataset & 32   & 0.235 & 0.085 & 0.480 & 0.164 & 0.346 & 0.461 \\
& 2.6k Dataset & 64   & 0.496 & 0.334 & 0.538 & 1.156 & 0.529 & 0.570 \\
\cmidrule{2-9}
& 13k Dataset  & {--} & 0.083 & 0.023 & 0.202 & 0.063 & 0.275 & 0.225 \\
& 13k Dataset  & 16   & 0.127 & 0.040 & 0.304 & 0.088 & 0.321 & 0.388 \\
& 13k Dataset  & 32   & 0.162 & 0.058 & 0.445 & 0.299 & 0.582 & 0.718 \\
& 13k Dataset  & 64   & 0.489 & 0.412 & 0.489 & 1.926 & 1.348 & 1.255 \\
\cmidrule{2-9}
& 26k Dataset  & {--} & 0.064 & 0.020 & 0.125 & 0.037 & 0.173 & 0.152 \\
& 26k Dataset  & 16   & 0.050 & 0.019 & 0.126 & 0.055 & 0.197 & 0.146 \\
& 26k Dataset  & 32   & 0.101 & 0.066 & 0.230 & 0.262 & 0.308 & 0.202 \\
& 26k Dataset  & 64   & 0.394 & 0.297 & 0.497 & 2.475 & 0.582 & 1.028 \\
\cmidrule{2-9}
& 260k Dataset & {--} & 0.007 & 0.003 & 0.013 & 0.004 & 0.018 & 0.019 \\
& 260k Dataset & 16   & 0.029 & 0.008 & 0.059 & 0.014 & 0.081 & 0.028 \\
& 260k Dataset & 32   & 0.130 & 0.060 & 0.210 & 0.088 & 0.201 & 0.125 \\
& 260k Dataset & 64   & 0.262 & 0.447 & 0.649 & 1.437 & 0.465 & 0.616 \\
\bottomrule
\end{tabular}
\end{table}

\paragraph{Robustness to Occlusion.}

Finally, we evaluate the score-based method's robustness to corruptions using the trained two-camera model with independent joints. At test time, we introduce artificial occlusions in the form of 32x32 white pixel squares into the input images, and report results for our method and RoboPEPP in \Cref{tab:mse_results_resized}. Our method was not trained on data containing occlusions or any infilling task. Despite this, our method demonstrates superior robustness to occlusions at test time. This is evident from the MSE of joints 2 and 4, where the performance degradation for our method is significantly less than for RoboPEPP when occlusions are introduced. Furthermore, our method maintains a lower overall MSE on these joints at both 10\% and 100\% training data labels, highlighting its robustness to this unseen perturbation. 
\Cref{fig:reconstruction_occlusion} visualizes the process. While the initial reconstruction from AE1 clearly shows the occluded patch, the final reconstruction from AE2 successfully inpaints the missing region by leveraging the learned disentangled representation. \Cref{results_crl_main_3dp_abbrev} (fourth row) presents the quantitative results for this condition, demonstrating strong performance despite the corruption. A detailed analysis of model performance across varying occlusion sizes is provided in \Cref{tab:mse_ablation_study}.

We conclude that AE1 is tasked with high-quality reconstruction, while AE2's loss function sparsifies the score difference. To capture out-of-plane rotations missed by a single camera, an additional viewpoint was essential. This lowered the LDR loss cross-entropy, allowing the model to successfully differentiate all six joint movements.

\begin{figure}[t]

    \centering
    
    \hfill
    \begin{subfigure}[b]{0.2\textwidth}
        \centering
        \includegraphics[width=\textwidth]{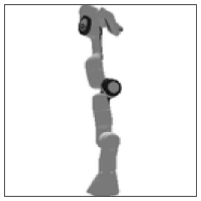}
        \caption{The original input image to the system.}
        \label{fig:four_figs_1}
    \end{subfigure}%
    \hfill
    \begin{subfigure}[b]{0.2\textwidth}
        \centering
        \includegraphics[width=\textwidth]{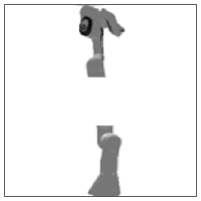}
        \caption{Input with $32\times32$ pixel white occlusion.}
        \label{fig:four_figs_2}
    \end{subfigure}%
    \hfill
    \begin{subfigure}[b]{0.2\textwidth}
        \centering
        \includegraphics[width=\textwidth]{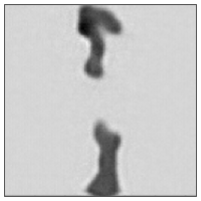}
        \caption{Reconstruction from autoencoder-1.}
        \label{fig:four_figs_3}
    \end{subfigure}%
    \hfill
    \begin{subfigure}[b]{0.2\textwidth}
        \centering
        \includegraphics[width=\textwidth]{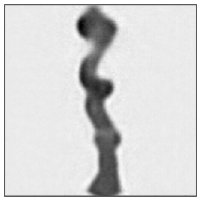}
        \caption{Reconstruction from autoencoder-2}
        \label{fig:four_figs_4}
    \end{subfigure}%
    \hfill\ %
    
    \caption{A step-by-step visualization of the reconstruction process for an occluded input. The final reconstruction from autoencoder-2, generated by passing its output through the decoder of autoencoder-1, successfully inpaints the occluded region.}
    \label{fig:reconstruction_occlusion}
\end{figure}

\subsection{Ablations}

We conducted a series of experiments to find the proper training setting for our score-based method. Specifically, we investigated the effect of the learning rate, batch size, and weights of the loss components ($\lambda=3$ in \Cref{eq:loss_function}) on the model's ability to disentangle the latent features, as measured by the MCC score. All experiments in this section were performed using Autoencoder-2 with a 2-camera setup on the independent dataset for joint 4.

\paragraph{Learning Rate.} First, we analyzed the impact of the learning rate. \Cref{fig:ablation_lr} shows the training progress for different learning rates while keeping the batch size fixed at 128 and the weight of the sparsity loss constant ($\lambda=3$ in \Cref{eq:loss_function}). We observe that smaller learning rates, such as $5e-5$ and $8e-5$, lead to better performance.

\begin{figure}[htbp]
  \centering
    \includegraphics[width=.8\textwidth]{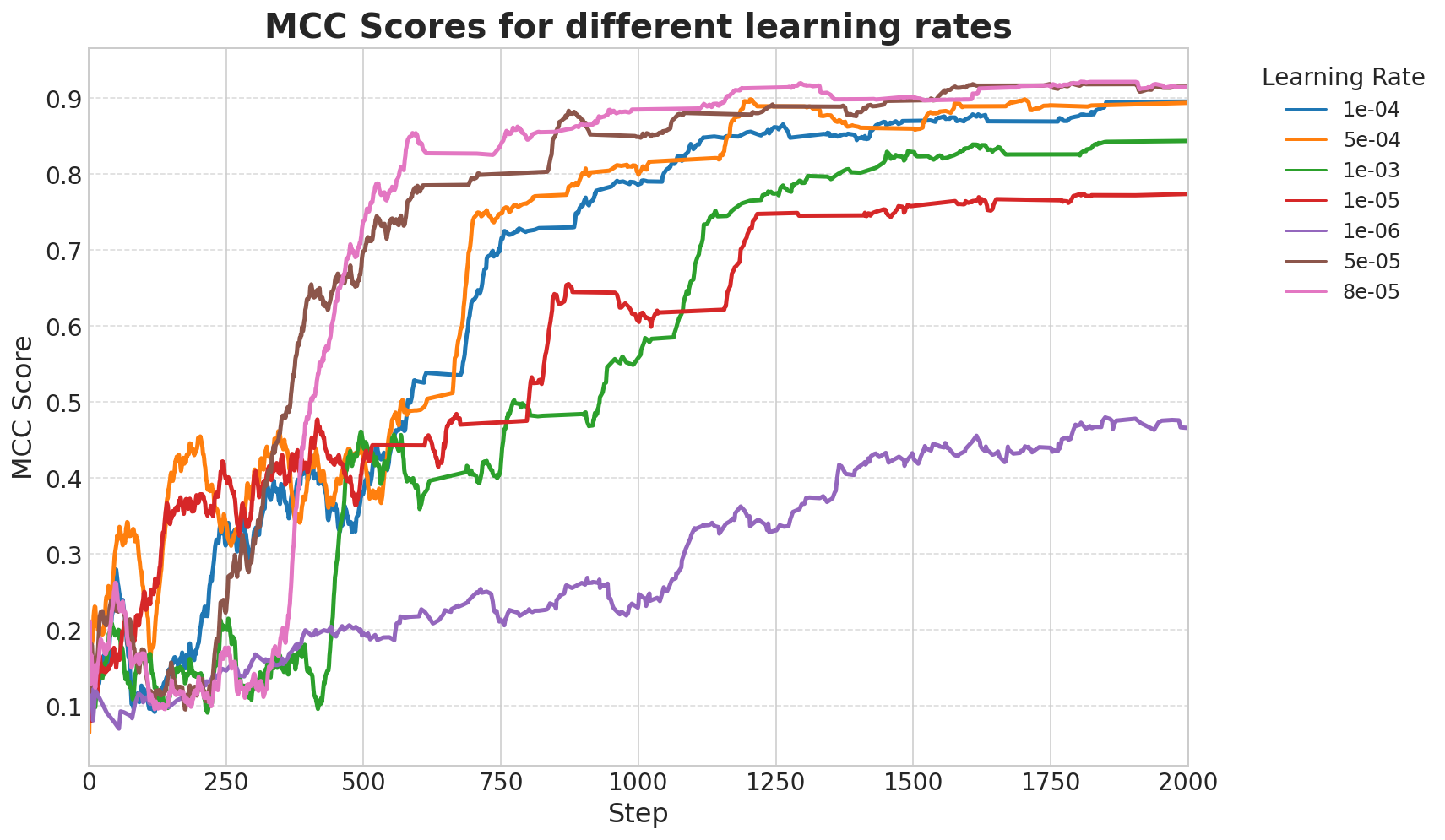} 
  \caption{Evolution of MCC scores during training for AE2 using a 2-camera setup on the independent dataset. The plot compares performance across various learning rates. Training is performed with a batch size of 128, and sparsity loss weight of $\lambda=3$.}
  \label{fig:ablation_lr}
\end{figure}

\paragraph{Batch Size.} Next, we performed an ablation study on the batch size (\Cref{fig:ablation_bsz}), using the optimal learning rate $5e-5$ found previously. We observed that models with larger batch sizes converged faster; however, their final asymptotic performances were nearly identical, irrespective of the batch size.

\begin{figure}[htbp]
  \centering
    \includegraphics[width=.8\textwidth]{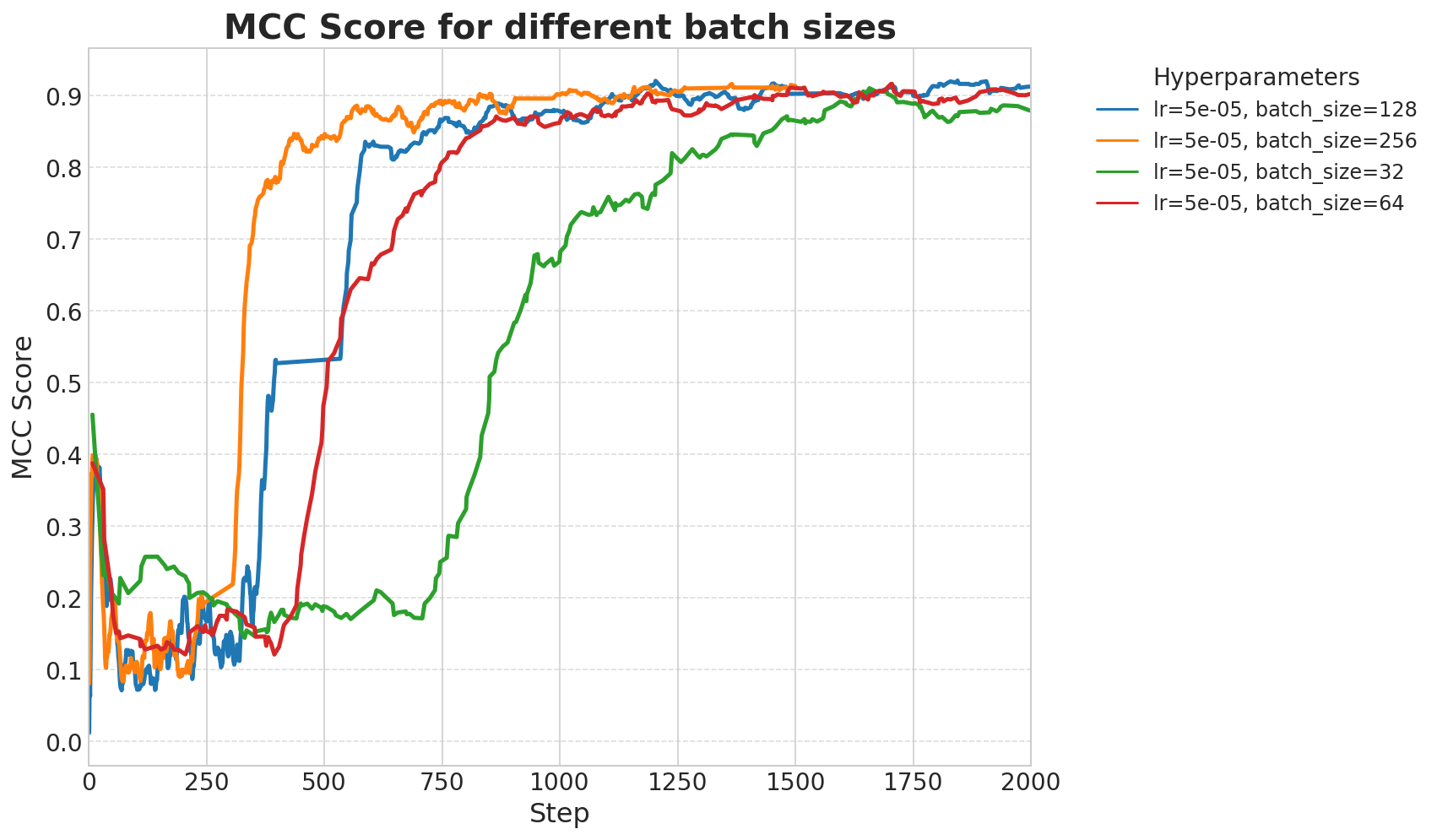}  
  \caption{Evolution of MCC scores during training for AE2 using a 2-camera setup on the independent dataset. The plot compares performance across various batch sizes. Training is performed with a learning rate of $5e-5$ and sparsity loss weight of $\lambda=3$.}
  \label{fig:ablation_bsz}
\end{figure}

\paragraph{Loss Reweighting.} We also looked at how to balance the two training objectives: reconstructing the image and the sparsity loss. \Cref{fig:ablation_wts} compares the MCC scores when we vary the weights given to the reconstruction loss versus the sparsity loss. For these experiments, we fixed the learning rate at $5e{-5}$ and batch size at 128. 
The figure shows that giving a significant weight to the score-based loss (i.e., sparsity loss) is crucial for disentanglement. This aligns with our theory that a reconstruction loss alone cannot perform disentanglement (e.g., MCC of less than 0.15 in the figure), but the score-based loss is the main driver of the disentanglement (e.g., over 0.9 MCC in the figure). 

\begin{figure}[htbp]
  \centering
    \includegraphics[width=.8\textwidth]{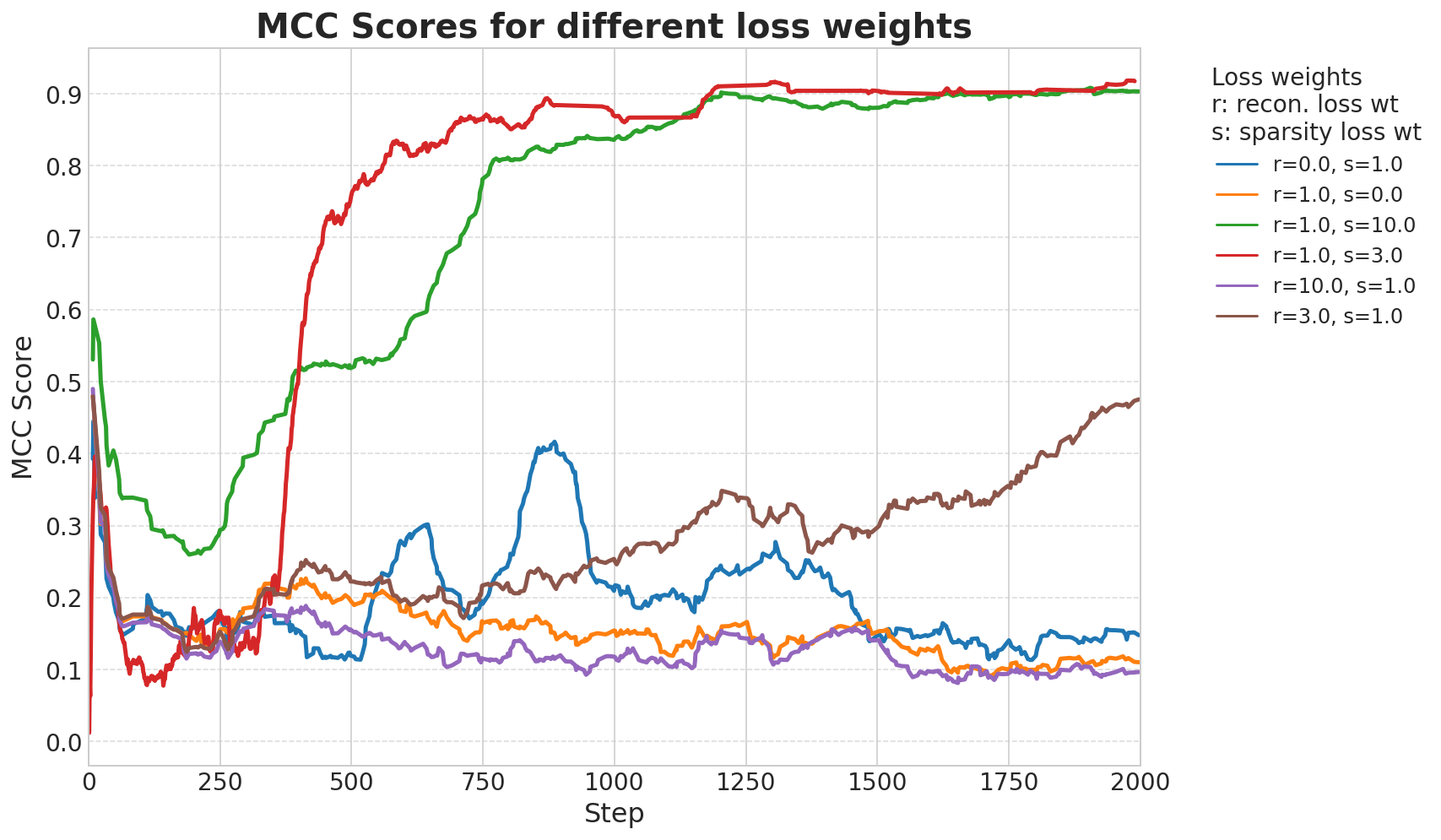} 
  \caption{Evolution of MCC scores during training for AE2 using a 2-camera setup on the independent dataset. The plot compares performance across various combinations of reconstruction loss weights and sparsity loss weights. Training is performed with a batch size of 128 and a fixed learning rate of $5e-5$.}
  \label{fig:ablation_wts}
\end{figure}
\paragraph{Architectural Changes - Residual Connection.} The inclusion of intra-module residual connections—applied independently within the encoder and decoder architectures—is a key factor for optimal performance. This architectural choice as seen in the \Cref{fig:ablation_residual} improves the reconstruction loss and hence the final MCC score.
\begin{figure}[htbp]
  \centering
    \includegraphics[width=.95\textwidth]{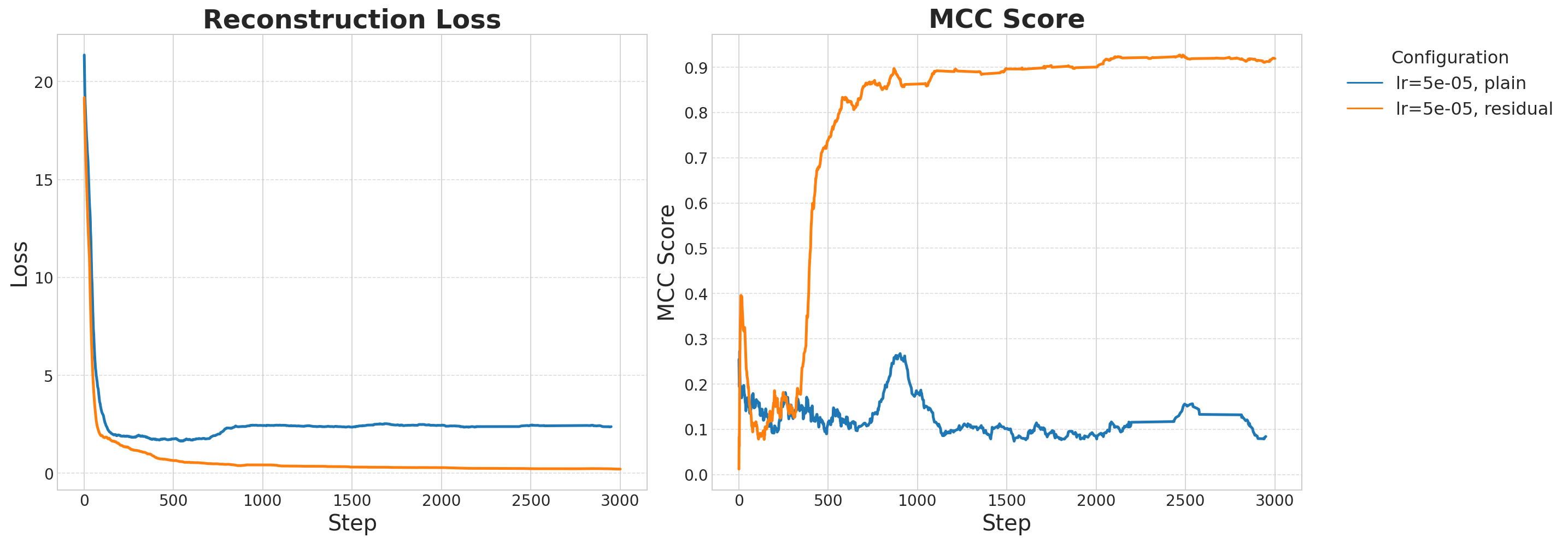}  
  \caption{Ablation showing the advantage of having residual connections for better reconstruction loss and MCC score}
  \label{fig:ablation_residual}
\end{figure}

\paragraph{Calibration Samples.} As demonstrated in \Cref{table:calibration}, the model's Mean Squared Error (MSE) remains stable even when the number of calibration samples is reduced from 1000 to 100.
\begin{table}[H]
\centering
\caption{MCC and MSE of our method across different Calibration settings. MSE is reported in radians squared.}
\label{table:calibration}
\begin{threeparttable}
\footnotesize 
\setlength{\tabcolsep}{2.3pt} 

\begin{tabular}{@{} 
    >{\raggedright}m{2.2cm} 
    S[table-format=1.3] S[table-format=1.3] @{\hskip 1.1em} 
    S[table-format=1.3] S[table-format=1.3] @{\hskip 1.1em} 
    S[table-format=1.3] S[table-format=1.3] @{\hskip 1.1em} 
    S[table-format=1.3] S[table-format=1.3] @{\hskip 1.1em} 
    S[table-format=1.3] S[table-format=1.3] @{\hskip 1.1em} 
    S[table-format=1.3] S[table-format=1.3] @{}}
\toprule
& \multicolumn{2}{c}{\textbf{Joint 1}} & \multicolumn{2}{c}{\textbf{Joint 2}} & \multicolumn{2}{c}{\textbf{Joint 3}} & \multicolumn{2}{c}{\textbf{Joint 4}} & \multicolumn{2}{c}{\textbf{Joint 5}} & \multicolumn{2}{c}{\textbf{Joint 6}} \\
\midrule
\textbf{Calibration Samples} & {\textbf{MCC}} & {\textbf{MSE}} & {\textbf{MCC}} & {\textbf{MSE}} & {\textbf{MCC}} & {\textbf{MSE}} & {\textbf{MCC}} & {\textbf{MSE}} & {\textbf{MCC}} & {\textbf{MSE}} & {\textbf{MCC}} & {\textbf{MSE}} \\
\midrule
1000 samples & 0.874 & 0.083 & 0.979 & 0.015 & 0.634 & 0.217 & 0.950 & 0.035 & 0.679 & 0.198 & 0.884 & 0.080 \\
\midrule
500 samples & 0.863 & 0.094 & 0.979 & 0.015 & 0.618 & 0.234 & 0.945 & 0.037 & 0.677 & 0.191 & 0.895 & 0.075 \\
\midrule
100 samples & 0.881 & 0.079 & 0.975 & 0.019 & 0.612 & 0.227 & 0.942 & 0.039 & 0.661 & 0.199 & 0.879 & 0.084 \\
\bottomrule
\end{tabular}
\end{threeparttable}
\end{table}

\paragraph{Lighting Condition.} \Cref{table:dark_light} evaluates the model's performance under distinct lighting scenarios, simulated by scaling pixel brightness to 0.8x (dark) and 1.25x (light) of the original. Performance degradation is most pronounced in the dark, particularly for Joint 5 and Joint 6, which exhibit a significant drop in MCC. In contrast, the remaining joints maintain strong performance with only minimal degradation under both conditions.
\begin{table}[H]
\centering
\caption{MCC and MSE of our method across different lighting condition. MSE is reported in radians squared.}
\label{table:dark_light}
\begin{threeparttable}
\footnotesize 
\setlength{\tabcolsep}{2.3pt} 
\begin{tabular}{@{} 
    >{\raggedright}m{2.2cm} 
    S[table-format=1.3] S[table-format=1.3] @{\hskip 1.1em} 
    S[table-format=1.3] S[table-format=1.3] @{\hskip 1.1em} 
    S[table-format=1.3] S[table-format=1.3] @{\hskip 1.1em} 
    S[table-format=1.3] S[table-format=1.3] @{\hskip 1.1em} 
    S[table-format=1.3] S[table-format=1.3] @{\hskip 1.1em} 
    S[table-format=1.3] S[table-format=1.3] @{}}
\toprule
& \multicolumn{2}{c}{\textbf{Joint 1}} & \multicolumn{2}{c}{\textbf{Joint 2}} & \multicolumn{2}{c}{\textbf{Joint 3}} & \multicolumn{2}{c}{\textbf{Joint 4}} & \multicolumn{2}{c}{\textbf{Joint 5}} & \multicolumn{2}{c}{\textbf{Joint 6}} \\
\midrule
\textbf{Lighting Condition} & {\textbf{MCC}} & {\textbf{MSE}} & {\textbf{MCC}} & {\textbf{MSE}} & {\textbf{MCC}} & {\textbf{MSE}} & {\textbf{MCC}} & {\textbf{MSE}} & {\textbf{MCC}} & {\textbf{MSE}} & {\textbf{MCC}} & {\textbf{MSE}} \\
\midrule
original & 0.874 & 0.083 & 0.979 & 0.015 & 0.634 & 0.217 & 0.950 & 0.035 & 0.679 & 0.198 & 0.884 & 0.080 \\
\midrule
bright & 0.874 & 0.072 & 0.978 & 0.015 & 0.640 & 0.207 & 0.950 & 0.033 & 0.677 & 0.180 & 0.874 & 0.094 \\
\midrule
dark & 0.811 & 0.107 & 0.940 & 0.046 & 0.520 & 0.211 & 0.860 & 0.093 & 0.199 & 0.303 & 0.526 & 0.265 \\
\bottomrule
\end{tabular}
\end{threeparttable}
\end{table}

\subsection{Qualitative Examples}

To provide a qualitative understanding of our dataset, \Cref{fig:joint4_interventions_group,fig:joint3_interventions_group} visualize the hard intervention images that form the basis of our training data. \Cref{fig:joint4_interventions_group} shows intervention images for Joint 4 from the single-camera dataset. \Cref{fig:joint3_interventions_group} shows intervention images for Joint 3 from the two-camera dataset. Since Joint 3 moves out of the plane, these images highlight the necessity of our two-camera setup.

\begin{figure}[ht] 
    \centering
    
    \begin{minipage}[b]{0.29\textwidth}
        \centering
        
        \begin{subfigure}{\linewidth}
            \centering
            \includegraphics[width=\textwidth]{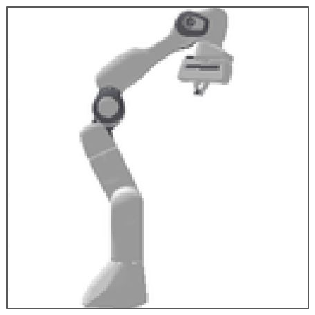}
            \caption{Intervention 1 (label 0)}
            \label{fig:joint4_hard1}
        \end{subfigure}
        
        
        \begin{subfigure}{\linewidth}
            \centering
            \includegraphics[width=\textwidth]{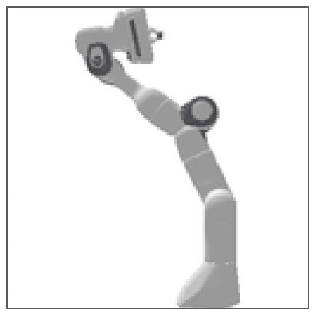}
            \caption{Intervention 2 (label 1)}
            \label{fig:joint4_hard2}
        \end{subfigure}
        
        \caption{Two different hard interventions on Joint 4.}
        \label{fig:joint4_interventions_group}
    \end{minipage}%
    \hspace{0.1\textwidth}%
    \begin{minipage}[b]{0.59\textwidth}
        \centering
        
        \begin{subfigure}[b]{0.49\linewidth}
            \centering
            \includegraphics[width=\textwidth]{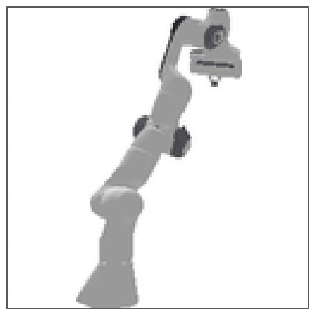}
            \caption{Intervention 1 (45°)}
            \label{fig:j3_int1_cam45}
        \end{subfigure}%
        \hfill
        \begin{subfigure}[b]{0.49\linewidth}
            \centering
            \includegraphics[width=\textwidth]{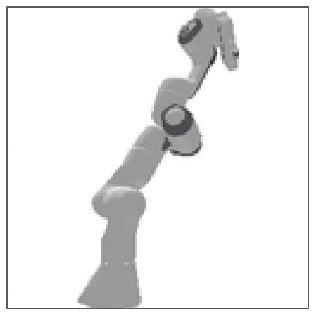}
            \caption{Intervention 2 (45°)}
            \label{fig:j3_int2_cam45}
        \end{subfigure}

        
        \begin{subfigure}[b]{0.49\linewidth}
            \centering
            \includegraphics[width=\textwidth]{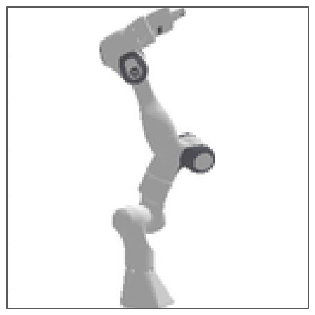}
            \caption{Intervention 1 (135°)}
            \label{fig:j3_int1_cam135}
        \end{subfigure}%
        \hfill
        \begin{subfigure}[b]{0.49\linewidth}
            \centering
            \includegraphics[width=\textwidth]{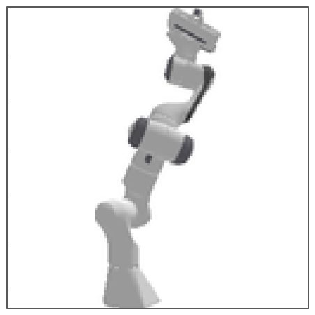}
            \caption{Intervention 2 (135°)}
            \label{fig:j3_int2_cam135}
        \end{subfigure}
        
        \caption{Two different hard interventions on Joint 3, each shown from two camera angles.}
        \label{fig:joint3_interventions_group}
    \end{minipage}

\end{figure}

\Cref{fig:single_camera_reconstruction,fig:recon_comparison_2_cameraAngles} provide a qualitative analysis of our pipeline's reconstruction performance. Specifically, \Cref{fig:autoencoder_stages:ae1_recon,fig:autoencoder_stages:ae2_recon} compare an original image (\Cref{fig:autoencoder_stages:obs_image}) to its reconstructions from AE1 and AE2 respectively for the single-camera setup, while \Cref{fig:recon_comparison_2_cameraAngles} shows the equivalent comparison for the two-camera setup. We observe a slight degradation in the reconstruction quality of AE1 when trained on the two-camera dataset compared to the single-camera setup. As the final autoencoder, AE2, is trained on the feature maps from AE1, this reduction in quality of AE1 consequently affects the quality of the final AE2 reconstructions as well. We hypothesize that this performance difference is attributable to the increased data complexity of the multi-view dataset. The inclusion of multiple viewpoints introduces greater visual variance, presenting a more challenging reconstruction task for the autoencoder compared to the more constrained single-view data.

\begin{figure}[ht] 
    \centering
    
    \begin{subfigure}[b]{0.32\textwidth}
        \centering
        \includegraphics[width=\textwidth]{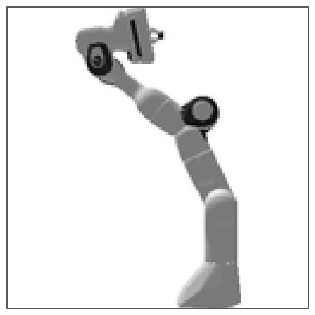}
        \caption{Observational Image}
        \label{fig:autoencoder_stages:obs_image}
    \end{subfigure}
    \hfill 
    \begin{subfigure}[b]{0.32\textwidth}
        \centering
        \includegraphics[width=\textwidth]{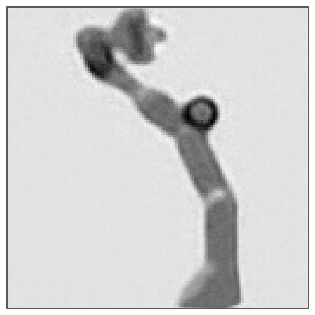}
        \caption{AE1 Reconstruction}
        \label{fig:autoencoder_stages:ae1_recon}
    \end{subfigure}
    \hfill 
    \begin{subfigure}[b]{0.32\textwidth}
        \centering
        \includegraphics[width=\textwidth]{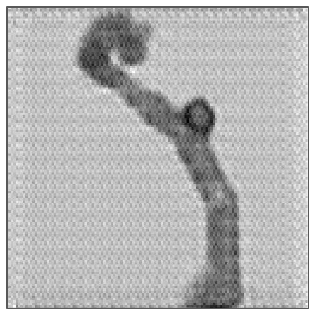}
        \caption{AE2 Reconstruction}
        \label{fig:autoencoder_stages:ae2_recon}
    \end{subfigure} 
    
    \caption{Visual comparison of the reconstruction quality at each stage of our pipeline for the single camera setup. (a) The original input image. (b) The reconstruction from the first autoencoder (AE1). (c) The final reconstruction from the second autoencoder (AE2)}
    \label{fig:single_camera_reconstruction}
\end{figure}
\begin{figure}[ht] 
    \centering
    
    \begin{subfigure}[b]{0.28\textwidth}
        \centering
        \includegraphics[width=\textwidth]{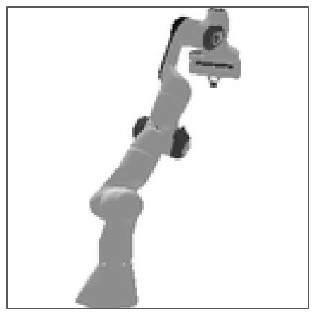}
        \caption{Observational Image}
        \label{fig:recon_comparison_2_cameraAngles:obs_image}
    \end{subfigure}
    \hfill 
    \begin{subfigure}[b]{0.28\textwidth}
        \centering
        \includegraphics[width=\textwidth]{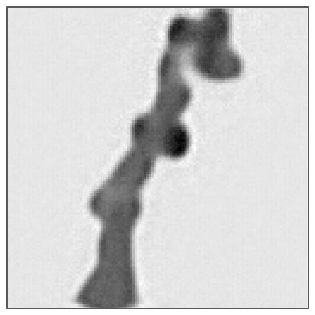}
        \caption{AE1 Reconstruction}
        \label{fig:recon_comparison_2_cameraAngles:ae1_recon}
    \end{subfigure}
    \hfill 
    \begin{subfigure}[b]{0.28\textwidth}
        \centering
        \includegraphics[width=\textwidth]{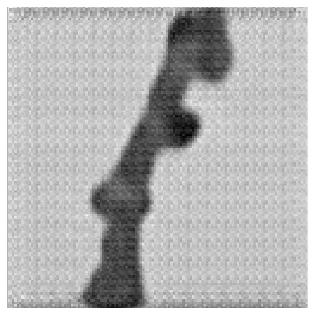}
        \caption{AE2 Reconstruction}
        \label{fig:recon_comparison_2_cameraAngles:ae2_recon}
    \end{subfigure} 
    
    \caption{Visual comparison of the reconstruction quality at each stage of our pipeline using the two camera independent model. (a) The original input image. (b) The reconstruction from the first autoencoder (AE1). (c) The final reconstruction from the second autoencoder (AE2) }
    \label{fig:recon_comparison_2_cameraAngles}
\end{figure}

\section{Architecture Details}\label{app:architecture}

The architectural details of the first autoencoder (AE1), the log-density ratio network (LDR) and the second autoencoder (AE2) are presented in \Cref{tab:AE1_arch,tab:ldr_arch,tab:AE2_arch}, respectively.

\begin{table}[htbp] 
\centering
\small
\caption{Autoencoder 1 architecture ResNet-style with GroupNorm}
\label{tab:AE1_arch}
\begin{tabular}{@{}ll@{}}
\toprule
\textbf{Component} & \textbf{Layer-wise Details} \\
\midrule
\textbf{Block Def.} & \textbf{ResBlockGN(f)}: \\
        & \quad GroupNorm $\to$ ReLU $\to$ Conv(features=f, ks=3, pad='SAME') \\
        & \quad $\to$ GroupNorm $\to$ ReLU $\to$ Conv(features=f, ks=3, pad='SAME') \\
        & \quad $\to$ Add residual input \\
        & \textit{(Note: `ks`=kernel size, `s`=stride, `pad`='SAME')} \\
\midrule
\textbf{Encoder} & \textbf{Input: Image $X \in \mathbb{R}^{128 \times 128 \times 1}$} \\
        & Conv(features=64, ks=3, pad='SAME') \\
        & ResBlockGN(64) x 2 \\
        & Conv(features=64, ks=3, s=2, pad='SAME'), ReLU \quad \textit{// Downsample 128 $\to$ 64} \\
        & Conv(features=128, ks=3, pad='SAME') \\
        & ResBlockGN(128) x 2 \\
        & Conv(features=128, ks=3, s=2, pad='SAME'), ReLU \quad \textit{// Downsample 64 $\to$ 32} \\
        & Conv(features=256, ks=3, pad='SAME') \\
        & ResBlockGN(256) x 2 \\
        & Conv(features=256, ks=3, s=2, pad='SAME'), ReLU \quad \textit{// Downsample 32 $\to$ 16} \\
        & Conv(features=512, ks=3, pad='SAME') \\
        & ResBlockGN(512) x 2 \\
        & Conv(features=1, ks=3, s=2, pad='SAME'), ReLU    \quad \textit{// Downsample 16 $\to$ 8} \\
        & \textbf{Output: Feature map $Y \in \mathbb{R}^{8 \times 8 \times 1}$} \\
\midrule
\textbf{Decoder} & \textbf{Input: Feature map $Y \in \mathbb{R}^{8 \times 8 \times 1}$} \\
        & Conv(features=512, ks=3, pad='SAME') \\
        & ResBlockGN(512) x 2 \\
        & ConvTranspose(features=512, ks=4, s=2, pad='SAME'), ReLU \quad \textit{// Upsample 8 $\to$ 16} \\
        & Conv(features=256, ks=3, pad='SAME') \\
        & ResBlockGN(256) x 2 \\
        & ConvTranspose(features=256, ks=4, s=2, pad='SAME'), ReLU \quad \textit{// Upsample 16 $\to$ 32} \\
        & Conv(features=128, ks=3, pad='SAME') \\
        & ResBlockGN(128) x 2 \\
        & ConvTranspose(features=128, ks=4, s=2, pad='SAME'), ReLU \quad \textit{// Upsample 32 $\to$ 64} \\
        & Conv(features=64, ks=3, pad='SAME') \\
        & ResBlockGN(64) x 2 \\
        & ConvTranspose(features=64, ks=4, s=2, pad='SAME'), ReLU  \quad \textit{// Upsample 64 $\to$ 128} \\
        & \textbf{Conv(features=1, ks=3, pad='SAME'), ReLU} \quad \textit{// Final convolution to 1 channel} \\
        & Reshape to (batch, $128 \times 128 \times 1$) \\
        & \textbf{Output: Reconstructed Image $\hat{X} \in \mathbb{R}^{128 \times 128 \times 1}$} \\
\bottomrule
\end{tabular}
\end{table}

\begin{table}[htbp] 
\centering
\small
\caption{LDR Network Architectures for Single- and Two-Camera Setups.}
\label{tab:ldr_arch}
\begin{tabular}{@{}ll@{}}
\toprule
\textbf{Component} & \textbf{Layer-wise Details} \\
\midrule
\textbf{Input Processing} & The input shape depends on the camera setup: \\
        & \quad \textbf{Single-Camera:} Input $y \in \mathbb{R}^{8 \times 8 \times 1}$ is used directly. \\
        & \quad \textbf{Two-Camera:} Two feature maps $y$ are concatenated to form an input $\in \mathbb{R}^{8 \times 8 \times 2}$. \\
\midrule
\textbf{Core Architecture} & The following layers are applied to the processed input: \\
        & Conv(features=32, ks=3), ReLU \quad \textit{// Spatial dim: 8x8 $\to$ 6x6} \\
        & Conv(features=64, ks=3), ReLU \quad \textit{// Spatial dim: 6x6 $\to$ 4x4} \\
        & Conv(features=128, ks=3), ReLU \quad \textit{// Spatial dim: 4x4 $\to$ 2x2} \\
        & Flatten \\
        & Dense(features=128), ReLU \\
        & Dense(features=1) \\
        & \textbf{Output: Logit $\in \mathbb{R}^{1}$} \\
\midrule
\textbf{Training} & Minimize binary cross-entropy with logits loss on the output. \\
\bottomrule
\end{tabular}
\end{table}

\begin{table}[htbp] 
\centering
\small
\caption{Autoencoder 2 (AE2) Architecture.}
\label{tab:AE2_arch}
\begin{tabular}{@{}ll@{}}
\toprule
\textbf{Component} & \textbf{Layer-wise Details} \\
\midrule
\textbf{Block Def.} & \textbf{ResBlockGN(f)}: \\
        & \quad GroupNorm $\to$ ReLU $\to$ Conv(features=f, ks=3, pad='SAME') \\
        & \quad $\to$ GroupNorm $\to$ ReLU $\to$ Conv(features=f, ks=3, pad='SAME') \\
        & \quad $\to$ Add residual input \\
        & \textit{(Note: `ks`=kernel size, `s`=stride, `pad`='SAME')} \\
\midrule
\textbf{Encoder} & \textbf{Input:} $y \in \mathbb{R}^{8 \times 8 \times C_{in}}$, where $C_{in}$ is 1 (single-cam) or 2 (two-cam). \\
        & Conv(features=64, ks=3, pad='SAME') \\
        & ResBlockGN(64) x 2 \\
        & Conv(features=64, ks=3, s=2, pad='SAME'), ReLU \quad \textit{// Downsample 8 $\to$ 4} \\
        & Conv(features=128, ks=3, pad='SAME') \\
        & ResBlockGN(128) x 2 \\
        & Conv(features=128, ks=3, s=2, pad='SAME'), ReLU \quad \textit{// Downsample 4 $\to$ 2} \\
        & Conv(features=256, ks=3, pad='SAME') \\
        & ResBlockGN(256) x 2 \\
        & Conv(features=256, ks=3, s=2, pad='SAME'), ReLU \quad \textit{// Downsample 2 $\to$ 1} \\
        & Flatten to (batch, 256) \\
        & Dense(features=$d_b$). \\
        & \textbf{Output: Latent $\hatz \in \mathbb{R}^{d_b}$} \\
\midrule
\textbf{Decoder} & \textbf{Input: Latent $\hatz \in \mathbb{R}^{d_b}$} \\
        & Dense(features=256), ReLU \\
        & Reshape to (batch, $1 \times 1 \times 256$) \\
        & Conv(features=512, ks=3, pad='SAME') \\
        & ResBlockGN(512) x 2 \\
        & ConvTranspose(features=512, ks=4, s=2, pad='SAME'), ReLU \quad \textit{// Upsample 1 $\to$ 2} \\
        & Conv(features=256, ks=3, pad='SAME') \\
        & ResBlockGN(256) x 2 \\
        & ConvTranspose(features=256, ks=4, s=2, pad='SAME'), ReLU \quad \textit{// Upsample 2 $\to$ 4} \\
        & Conv(features=128, ks=3, pad='SAME') \\
        & ResBlockGN(128) x 2 \\
        & ConvTranspose(features=128, ks=4, s=2, pad='SAME'), ReLU \quad \textit// Upsample 4 $\to$ 8 \\
        & Conv(features=$C_{in}$, ks=3, pad='SAME'), ReLU \\
        & Reshape to (batch, $8 \times 8 \times C_{in}$) \\
        & \textbf{Output: Reconstructed $\hat y \in \mathbb{R}^{8 \times 8 \times C_{in}}$} \\
\bottomrule
\end{tabular}
\end{table}

\end{document}